\documentclass[]{academic_template}
\usepackage[toc,page,header]{appendix}

\usepackage{marvosym}

\usepackage{bbding}

\usepackage{amssymb}
\usepackage{amsmath}

\usepackage{makecell}
\usepackage{algorithm}      
\usepackage{algpseudocode}  

\usepackage{graphicx}

\usepackage{listings}
\usepackage{hyperref}
\usepackage{url}
\usepackage{booktabs}
\usepackage{wrapfig}

\usepackage{colortbl}
\usepackage{caption}
\usepackage{subcaption}

\usepackage{amsthm}
\newtheorem{theorem}{Theorem}
\newtheorem{assumption}{Assumption}
\newtheorem{remark}{Remark}

\setlogoheight{8mm}   

\setlogospacing{0mm}
\setsjtublue 

\settitlerulethickness{3pt}

\abstractboxon 
\setabstractframecolor{gray} 
\setabstractbgcolor{gray!10} 
\setlogotolineshift{0mm}

\settoprulethickness{2.5pt}  
\setbottomrulethickness{1.5pt} 

\usepackage{amsmath,amsfonts,bm}

\def\eqref#1{equation~\ref{#1}}

\def\1{\bm{1}}

\DeclareMathAlphabet{\mathsfit}{\encodingdefault}{\sfdefault}{m}{sl}
\SetMathAlphabet{\mathsfit}{bold}{\encodingdefault}{\sfdefault}{bx}{n}

\DeclareMathOperator*{\argmin}{arg\,min}

\usepackage{dblfloatfix}

\usepackage{tabularx}

\usepackage{multicol}
\usepackage{multirow}
\usepackage[table]{xcolor}

\usepackage{natbib}
\setcitestyle{numbers}
\setcitestyle{square}
\usepackage{graphicx}

\newcommand{\OurMethod}{ACE-Brain-0}

\title{{\OurMethod}: Spatial Intelligence as a Shared Scaffold for Universal Embodiments}

\author[1,2*]{Ziyang Gong}
\author[1*]{Zehang Luo}
\author[1*]{Anke Tang}
\author[1,5*]{Zhe Liu}
\author[3]{Shi Fu}
\author[1\dagger,\ddagger]{Zhi Hou}
\author[6]{Ganlin Yang}
\author[7]{Weiyun Wang}
\author[1]{Xiaofeng Wang}
\author[1]{Jianbo Liu}
\author[8]{Gen Luo}
\author[5]{Haolan Kang}
\author[3]{Shuang Luo}
\author[9]{Yue Zhou}
\author[10]{Yong Luo}
\author[11]{Li Shen}
\author[7]{Xiaosong Jia}
\author[2]{Yao Mu}
\author[2\ddagger]{Xue Yang}
\author[1]{Chunxiao Liu}
\author[2]{Junchi Yan}
\author[5]{Hengshuang Zhao}
\author[3\ddagger]{Dacheng Tao}
\author[1\ddagger]{Xiaogang Wang}

\contribution[*]{Equal contribution}
\contribution[\dagger]{Project Leader}
\contribution[\ddagger]{Corresponding author}

\affiliation[1]{ACE Robotics}
\affiliation[2]{Shanghai Jiao Tong University}
\affiliation[3]{Nanyang Technological University}
\affiliation[4]{The Chinese University of Hong Kong}
\affiliation[5]{The University of Hong Kong}

\affiliation[6]{University of Science and Technology of China}
\affiliation[7]{Fudan University}
\affiliation[8]{Xiamen University}

\affiliation[9]{East China Normal University}
\affiliation[10]{Wuhan University}
\affiliation[11]{Sun Yat-sen University}

\abstract{

Universal embodied intelligence demands robust generalization across heterogeneous embodiments, such as autonomous driving, robotics, and unmanned aerial vehicles (UAVs).  However, existing embodied brain in training a unified model over diverse embodiments frequently triggers long-tail data, gradient interference, and catastrophic forgetting, making it notoriously difficult to balance universal generalization with domain-specific proficiency. In this report, we introduce \textbf{\OurMethod}, a generalist foundation brain that unifies spatial reasoning, autonomous driving, and embodied manipulation within a single multimodal large language model~(MLLM). Our key insight is that spatial intelligence serves as a universal scaffold across diverse physical embodiments: although vehicles, robots, and UAVs differ drastically in morphology, they share a common need for modeling 3D mental space, making spatial cognition a natural, domain-agnostic foundation for cross-embodiment transfer. Building on this insight, we propose the \textbf{Scaffold-Specialize-Reconcile}~(SSR) paradigm, which first establishes a shared spatial foundation, then cultivates domain-specialized experts, and finally harmonizes them through \textbf{data-free model merging}. Furthermore, we adopt Group Relative Policy Optimization~(GRPO) to strengthen the model's comprehensive capability. Extensive experiments demonstrate that \OurMethod{} achieves competitive and even state-of-the-art performance across 24 spatial and embodiment-related benchmarks. 
}

\date{\today}

\checkdata[Project Page]{\url{https://ace-brain-team.github.io/ACE-Brain-0/}}
\checkdata[Code]{\url{https://github.com/ACE-BRAIN-Team/ACE-Brain-0}}
\checkdata[Hugging Face]{\url{https://huggingface.co/ACE-Brain/ACE-Brain-0-8B}}

\begin{document}

\maketitle

\newpage
\tableofcontents

\newpage

\section{Introduction}

Building embodied agents capable of perceiving, reasoning, and acting in the physical world demands intelligence far beyond isolated tasks or single modalities. In practice, such physical intelligence requires integrating temporal-spatial understanding, decision-making, planning, and so on. These capabilities are indispensable across diverse domains, including autonomous driving~\cite{robotron-drive,drivelm,drivemm,drivelmm-o1,drivegpt4,drivepi}, low-altitude sensing~\cite{mllm_uav_swarm,aircopbench,uav-vl-r1,airspatialbot,zhao2025urbanvideo,dvgbench,geonav,sautenkov2025uav,uav_survey,airvista,airvista2}, and embodied interaction~\cite{robotron-manip,roboafford,embodiedr1,llarva,rovi,gpt4robo,robospatial,eo1}.
Recent advances in multimodal large language models~(MLLMs)~\cite{gemini,glm-4.1v,kimi-k2.5,blip2,qwen2.5-vl,qwen2vl,qwen3,internvl,internvl2,internvl3,internvl3.5} have demonstrated impressive generalization across vision-language tasks, inspiring a surge of research into spatial understanding \cite{cambrian-s,sensenova,sensenova_1.5,internspatial,spatialtuning,spacevista,3dr1,spatial_ssrl,multi_spatialmllm,leovl,tiger} and embodied foundation models~\cite{robobrain,robobrain2,robobrain2.5,mimo-embodied,vlaser,vebrain,pelican-vl,nvidia2025cosmos,fang2025robix}.
Despite rapid advances in MLLMs, developing a \textit{generalist embodied foundation brain} that unifies these heterogeneous capabilities within a single model remains a central challenge.

\begin{figure*}[t!]
    \centering
    \includegraphics[width=\textwidth]{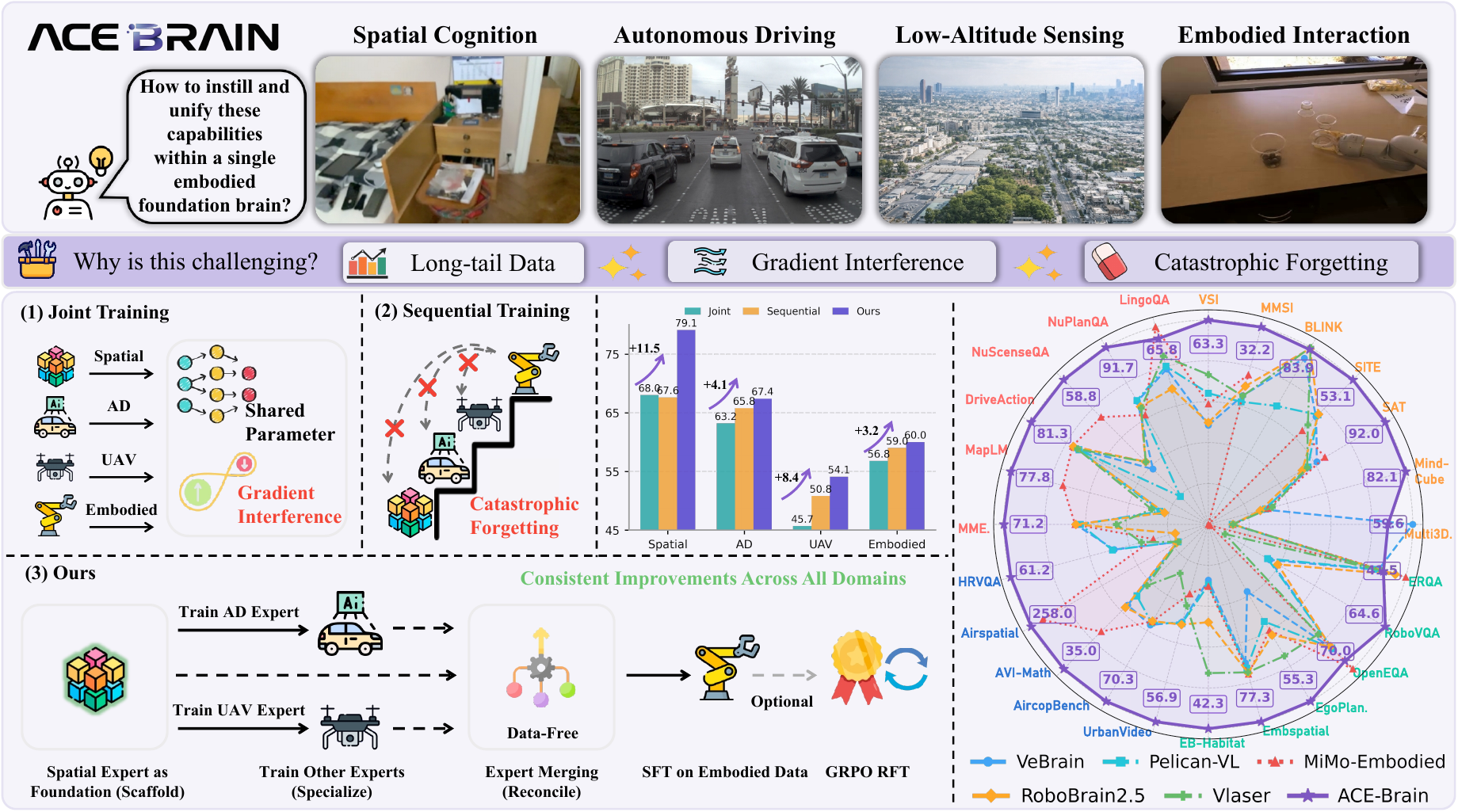}
    \caption{\textbf{Cross-Embodiment Learning Paradigm of \OurMethod{} and Performance Comparison with other Embodied Brains.}   \OurMethod{} unifies tasks from four domains, Spatial Cognition, Autonomous Driving, Low-Altitude Sensing, and Embodied Manipulation. We hope to answer: ``How can we instill and unify these capabilities within a single embodied foundation brain?'' Conventional joint training mixes multi-domain data with shared parameters, which often causes gradient interference across tasks; sequential training accumulates skills via stage-wise fine-tuning, but tends to overwrite previously learned capabilities and leads to catastrophic forgetting. In contrast, we propose our Scaffold-Specialize-Reconcile paradigm: We first construct a Spatial Expert as a universal foundational model, then train the AD and UAV experts separately to acquire domain-specific skills while enabling coarse-grained spatial reasoning, and subsequently combine their expertise into a unified model via data-free expert merging. We further perform Embodied SFT, optionally followed by GRPO-based RFT for reward-guided post-training alignment. This pipeline delivers consistent and stable improvements across all four domains. The radar chart on the right further compares \OurMethod{} against representative embodied brains across multiple benchmarks, showing stronger overall performance on a broader set of tasks, and validating the unified cross-embodiment capability of \OurMethod{}.}
    \label{teaser}
\end{figure*}

Existing approaches toward this goal fall short in two complementary ways. Joint training over mixed embodiment data frequently suffers from long-tail distributions, severe task interference, and diluted domain specialization, as conflicting gradients from heterogeneous domains compromise each other's optimization landscape. Alternatively, sequential domain-specific fine-tuning can sharpen performance on a target domain but inevitably incurs catastrophic forgetting of previously acquired abilities. These failure modes reveal that the core bottleneck is not merely data diversity or model capacity, but rather the absence of a principled mechanism to \textbf{organize, integrate, and preserve} cross-embodiment physical knowledge.

In this work, we identify a key structural insight that opens a path forward: \textbf{spatial intelligence serves as a shared scaffold across diverse embodiments}. Although autonomous driving, robotic interaction, and low-altitude sensing differ drastically in morphology and action space, they share a fundamental reliance on 3D spatial understanding, perceiving object layouts, thinking about geometric relations, and predicting spatial consequences of actions. This common denominator makes spatial modeling a natural, domain-agnostic foundation that can scaffold and catalyze learning across otherwise disparate physical domains.

Furthermore, this spatial foundation naturally anchors a \textbf{coarse-to-fine cognitive progression: from spatial perception, to high-level planning, and ultimately to fine-grained action}. Specifically, autonomous driving and low-altitude sensing primarily demand spatial-aware planning capabilities, functioning as Vision-and-Language Navigation~(VLN) tasks that focus on trajectory planning or behavior decision-making. Conversely, embodied interaction requires fine-grained execution, aligning with Vision-Language-Action~(VLA) paradigms that govern low-level kinematic control and precise object manipulation. 

Building on this insight, we introduce \textbf{\OurMethod}, a generalist foundation brain that unifies spatial cognition, autonomous driving, low-altitude  sensing~\cite{cao2021visdrone,zhou2021egoplanner,zhou2025multimodal,cao2025proximal,wang2025jtd,chu2024towards,zhang2025your,guo2025bedi,wang2025uav,yao2024aeroverse}, and embodied interaction within a single MLLM. To effectively harness spatial intelligence as a universal scaffold while preserving domain-specific proficiency, we propose the \textbf{Scaffold-Specialize-Reconcile~(SSR)} training paradigm. As shown in Fig.~\ref{teaser}, SSR operates in three phases:1)~\textbf{Scaffold:}\quad Establish a shared spatial foundation that encodes domain-agnostic 3D understanding as a universal structural prior.
2)~\textbf{Specialize:}\quad Cultivate domain-specialized experts that build upon the spatial scaffold to acquire embodiment-specific capabilities. 3)~\textbf{Reconcile:}\quad Harmonize heterogeneous experts into a unified model through \textbf{data-free model merging}, avoiding both gradient interference from joint training and catastrophic forgetting from sequential training. 
The SSR-trained model thus serves as a new foundation for subsequent capability expansion. On top of this, we further integrate embodied interaction data to enable finer-grained embodied knowledge acquisition. Finally, a Reinforcement Fine-Tuning~(RFT) strategy can be optionally employed to amplify targeted competencies.

Extensive evaluation across \textbf{24} embodiment-related benchmarks demonstrates that \OurMethod{} achieves competitive and even state-of-the-art performance across all targeted domains. In visual spatial intelligence, it attains top results on SAT~\cite{sat}~(92.0\%) and Mindcube-Tiny~\cite{yin2025spatial}~(82.1\%), significantly outperforming both open-source and closed-source models. In autonomous driving, \OurMethod{} reaches 71.2\% on MME-RealWorld~\cite{zhang2024mme} and 91.7\%
    on NuPlanQA~\cite{park2025nuplanqa}, while setting new records on low-altitude benchmarks, including UrbanVideo-Bench~\cite{zhao2025urbanvideo}~(56.9\%) and AircopBench~\cite{aircopbench}~(70.3\%). Crucially, ablation studies confirm that the SSR paradigm consistently avoids the catastrophic forgetting observed in sequential training and surpasses the limited gains of joint training, validating our core hypothesis that principled expert synthesis is fundamental to organizing heterogeneous physical knowledge.

In summary, our main contributions are as follows: 1)~We identify \textbf{spatial intelligence as a shared scaffold} for cross-embodiment transfer, empirically demonstrating that a shared spatial foundation substantially boosts learning across
    diverse physical domains; 2)~We propose the \textbf{Scaffold-Specialize-Reconcile} training paradigm, which decouples shared spatial structure from domain-specific specialization and reconciles heterogeneous experts via data-free model merging, effectively resolving the stability-plasticity dilemma. 3)~We build \textbf{\OurMethod}, a generalist foundation brain that achieves competitive and even state-of-the-art performance across \textbf{24} benchmarks spanning spatial cognition, autonomous driving, low-altitude sensing, and embodied interaction (shown in Fig.~\ref{four_domain}), which provides a principled blueprint for generalist embodied AI.

The remainder of this report is organized as follows. Section~\ref{sec:arch} presents the multimodal auto-regressive architecture of \OurMethod{}. The proposed Scaffold-Specialize-Reconcile training strategy, covering five stages from base knowledge acquisition to GRPO reinforcement learning, is detailed in Section~\ref{sec:sr}. Section~\ref{sec:experiments} reports extensive experimental results and Section~\ref{sec:ab} presents the ablation analyses that validate our core findings on spatial scaffolding and expert reconciliation. Section~\ref{sec:data} describes the multi-domain training corpus, spanning general multimodal instructions, embodied manipulation, autonomous driving, and low-altitude aerial datasets. Section~\ref{sec:bench} introduces the evaluation benchmarks. Finally, Section~\ref{sec:con} concludes with perspectives on future generalist embodied agents.

\begin{figure*}[t!]
    \centering
    \includegraphics[width=\textwidth]{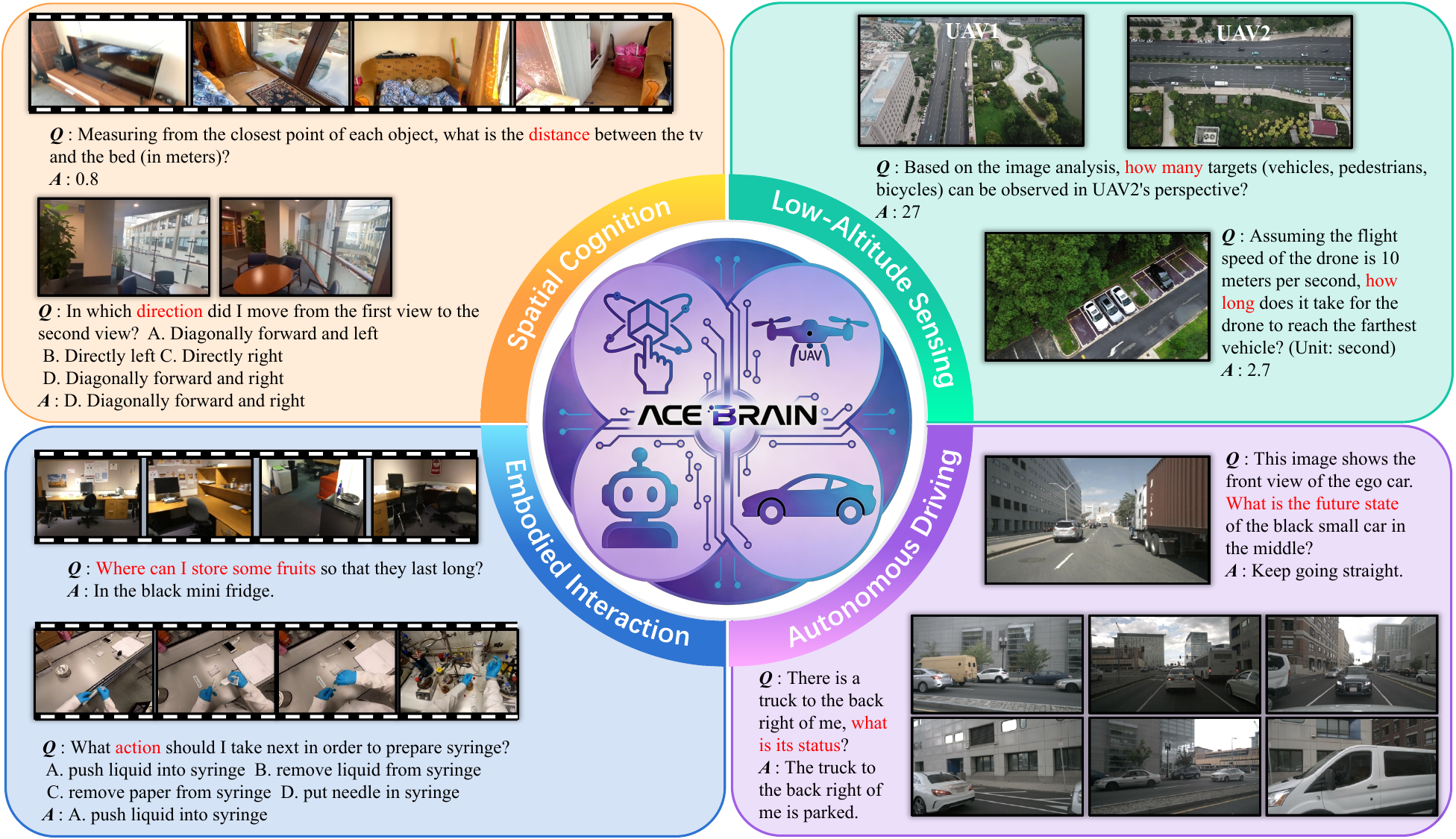}
    \caption{\textbf{Overview of \OurMethod{} Capabilities.} \OurMethod{} is a spatial-centric foundation
    brain that supports Spatial Intelligence, Embodied Manipulation, Low-Altitude Sensing, and Autonomous Driving.
    Specifically, \OurMethod{} is evaluated on 7 benchmarks for Spatial Cognition, 6 benchmarks for Autonomous Driving, 5 benchmarks for Low-Altitude Sensing, and 6 benchmarks for Embodied Interaction. \OurMethod{}'s ability to integrate perception, decision, and planning across diverse real-world embodied scenarios, highlighting its generalization capability as a universal embodied intelligence model.
}
    \label{four_domain}
\end{figure*}

\section{\OurMethod{} Architecture}
\label{sec:arch}

\begin{figure}
    \includegraphics[width=\textwidth]{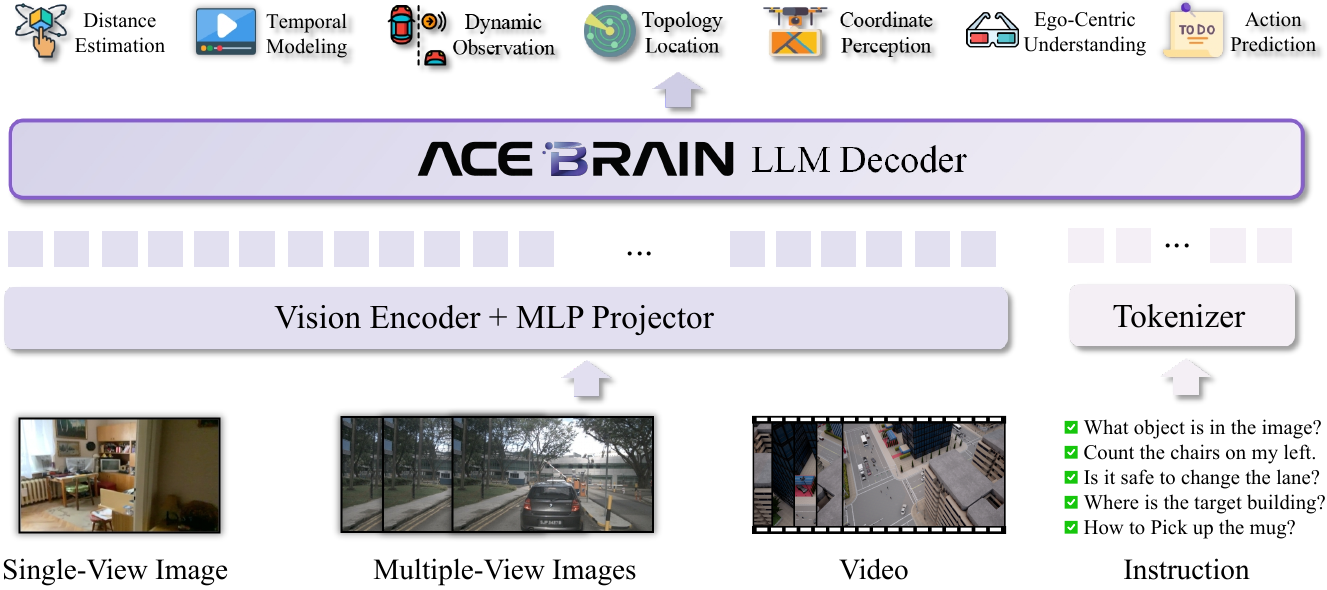}
    \caption{\textbf{\OurMethod{}’s unified multimodal architecture and cross-domain capability coverage.}
        \OurMethod{} supports inputs including single-view images, multi-view images, and videos; the instruction examples illustrate that the model can perform Q\&A-style tasks across domains (General/Spatial/Driving/Aerial/Embodied).
        The top row summarizes \OurMethod’s core capability spectrum for cross-embodiment scenarios, such as Spatial Perception and Temporal Modeling, enabling unified representation and compositional generalization across domains.}
    \label{fig:arch}
\end{figure}

\subsection{Task Formulation}

We consider a cross-domain embodiment learning setting, where a single MLLM is trained to perform tasks arising from distinct embodied agent forms. We denote the set of domains as $\mathcal{M} = \{ m_1, m_2, \dots, m_K \}$. Specifically, we consider the following domains in this work:
\begin{equation}
    \mathcal{M} = \{m_{\text{general}}, m_\text{embodied}, m_{\text{spatial}}, m_{\text{driving}}, m_{\text{aerial}}\},
\end{equation}
encompassing general, spatial, autonomous driving, and low-altitude domains.
Each $m_k \in \mathcal{M}$ induces a specific task distribution $\mathcal{D}_{m_k}$ over training samples $(o, c, y)$, where (1) $o \in \mathcal{O}_{m_k}$ denotes multimodal observations (images or video sequences);
(2) $c \in \mathcal{C}$ denotes task conditioning, including natural language instructions, queries, or high-level goals;
and (3) $y \in \mathcal{Y}_{m_k}$ denotes the target output, which may be textual responses, reasoning traces, spatial descriptions, action sequences, or planning trajectories depending on the task.
Despite substantial heterogeneity across observation spaces $\mathcal{O}_{m_k}$ and output formats $\mathcal{Y}_{m_k}$, we model all tasks using a unified conditional autoregressive formulation:
\begin{equation}
    p_\theta(y \mid o, c),
\end{equation}
where $\theta$ denotes the parameters of a single shared MLLM. This design choice enforces a common representational backbone and thinking substrate across all embodiments, enabling knowledge transfer and compositional generalization.

\subsection{Multimodal Architecture}

As shown in Figure~\ref{fig:arch}, \OurMethod{} adopts a multimodal autoregressive architecture, following recent designs of MLLM~\citep{bai2025qwen3vltechnicalreport}.
The model consists of three core components: a Vision Encoder paired with an MLP Projector, a Tokenizer, and the \OurMethod{} LLM Decoder.

The model accepts diverse visual inputs, including single-view images, multiple-view images, and video, together with natural language instructions. This flexibility enables \OurMethod{} to handle a broad spectrum of embodied tasks, from static scene understanding to temporal reasoning over video sequences.
All visual inputs are processed by a Vision Encoder, which extracts rich visual features regardless of the input modality or domain. The extracted features are then projected into the language model's embedding space through an MLP Projector. Notably, the resulting visual tokens are conceptually organized by domain into five categories-\textit{General}, \textit{Spatial}, \textit{Driving}, \textit{Aerial}, and \textit{Embodied}.
Natural language instructions and queries are converted into text tokens by the Tokenizer. These text tokens serve as task conditioning, specifying the desired output format, domain context, or action space for each query.

The domain-organized visual tokens and text tokens are concatenated into a unified sequence and fed into the \OurMethod{} LLM Decoder, which autoregressively generates output tokens. This unified decoder enables the model to jointly attend to visual and textual information, supporting a wide range of output capabilities, such as \textit{Spatial Perception}, \textit{Temporal Modeling}, \textit{Trajectory Prediction}, \textit{Safety Control}, \textit{Aerial Location}, \textit{Multi-UAV Cooperation}, \textit{Embodied Interaction}, \textit{Task Planning}, and so on.

Formally, given visual observations $o \in \mathbb{R}^{T \times H \times W \times 3}$ (e.g., single-view images with $T{=}1$, multiple-view images, or video frames) and textual conditioning $c$ (e.g., instructions or queries), the model predicts the next-token distribution as follows:
\begin{equation}
    p = \mathcal{F}_{\text{dec}}\Big(
    t_N \mid
    \mathcal{F}_{\text{proj}}\big(\mathcal{F}_{\text{enc}}(o; \theta_{\text{enc}}); \theta_{\text{proj}}\big),
    \mathcal{F}_{\text{tok}}(c),
    t_{0:N-1}; \theta_{\text{dec}}
    \Big),
\end{equation}
where $\mathcal{F}_{\text{enc}}(\cdot; \theta_{\text{enc}})$ denotes the Vision Encoder, $\mathcal{F}_{\text{proj}}(\cdot; \theta_{\text{proj}})$ denotes the MLP Projector, $\mathcal{F}_{\text{tok}}(\cdot)$ denotes the Tokenizer, and $\mathcal{F}_{\text{dec}}(\cdot; \theta_{\text{dec}})$ denotes the LLM Decoder. The output $p \in \mathbb{R}^{m}$ represents the probability distribution over a vocabulary of size $m$, and $t_i$ denotes the $i$-th generated token.

\subsection{Multimodal Autoregressive Objective}

Given a training sample $(o, c, y)$ where $y$ denotes the target output (e.g., textual answers, thinking traces, or action sequences), the MLLM processes the observation $o$ and conditioning text $c$ into a unified token sequence $\mathbf{s} = (s_1, \dots, s_L)$, where the first tokens correspond to encoded visual features and input text, followed by the target output tokens from $y$. We adopt the standard left-to-right autoregressive objective:
\begin{equation}
    \mathcal{L}_{\text{full}}(\theta)
    = - \sum_{i=1}^{L} w_i \log p_\theta(s_i \mid s_{<i}),
\end{equation}
where $w_i$ denotes the loss weight of token $s_i$, and $s_{<i} = (s_1, \dots, s_{i-1})$ denotes the preceding context.
Following common practice in MLLM training, loss computation is restricted to text tokens only, while visual tokens serve as conditioning context and are not directly predicted. This yields the supervised objective:
\begin{equation}
    \mathcal{L}_{\text{Text}}(\theta)
    = - \sum_{\substack{i=1, s_i \in \text{Text}}}^{L}
    w_i \log p_\theta(s_i \mid s_{<i}).
\end{equation}
Regarding the choice of token weights $w_i$, naive token averaging or sample averaging may introduce biases toward longer or shorter responses. To mitigate this issue, we adopt square averaging, which balances gradient contributions across samples with different sequence lengths.

\section{Training Strategy}
\label{sec:sr}

In this section, we describe the training methodology of \textbf{\OurMethod}.
Our training strategy follows a staged paradigm, which progressively builds shared multimodal representations, morphology-specific expertise, and finally a unified policy capable of generalizing across embodied domains.

\subsection{Stage 1: Spatial Scaffold Training}

In Stage 1, our primary objective is to train an expert model with spatial cognition capabilities. This stage consists of two main steps: first, we train a base model $\theta_{\text{base}}$ based on Qwen3-VL $\theta$ using general data~\cite{cambrian} to perform early activation through instruction tuning; second, we train a spatial expert model using large-scale spatial data from $\theta_{\text{base}}$. The spatial expert $\theta_{\text{spatial}}$ will serve as the central node to provide a shared universal scaffold for the following experts' training.

\subsection{Stage 2: Supervised Specialized Expert Fine-Tuning}
\label{sec:expert_finetune}

Starting from the spatial expert model $ \{ \theta_{\text{spatial}}$,
$ \theta_{\text{uav}}, \theta_{\text{ad}} \}$,
where each expert is initialized independently on data sampled from its corresponding distribution.
This isolation prevents gradient interference between domains with conflicting optimization objectives.

Specifically, we have: 
(1) $\theta_{\text{spatial}}$: Expert model specialized in spatial cognition modeling, trained on spatial intelligence datasets.
(2) $\theta_{\text{uav}}$: Expert model trained on $\theta_{\text{spatial}}$ and specialized in low-altitude sensing, location, and navigation for Unmanned Aerial Vehicles (UAV).
(3) $\theta_{\text{ad}}$: Expert model trained on $\theta_{\text{spatial}}$ and specialized in autonomous driving perception, planning, and control, trained on driving-specific datasets.

\subsection{Stage 3: Across-Embodiment Reconcile Model Merging}
\label{sec:stage_merging}

At this stage, we merge these expert models into a single unified model through a cross-embodiment merging procedure in a data-free manner. 
The goal of this stage is to synthesize complementary capabilities learned by each expert while mitigating interference.
The first systematic attempt to introduce adaptivity into multi-task model merging was made by AdaMerging~\citep{yangadamerging}, which learns task-specific merging coefficients and demonstrates clear advantages over fixed-weight baselines.
Building on this line of work, Shen et al.~\citep{shen2025efficient} propose what is arguably the most effective and efficient weight-ensembling mixture-of-experts framework to date, providing both theoretical justification and strong empirical evidence on large-scale multi-task benchmarks.
We adopt the optimization-based Merging algorithm~\citep{chengWhoeverStartedInterference2025,weiUnifyingMultimodalLarge2025}, which approximates the linear subspace of fine-tuning data for each expert by the task vector $\tau$, which is computed as the difference between the expert model parameters and the base model parameters.
The merging process begins with an initialization of the merged model as the average of all expert models $\theta_{\text{merge}}^{(0)} = \frac{1}{K} \sum_{i=1}^{K} \theta_i$.
Then, we iteratively optimize the merged model parameters to minimize the task interference across all experts as follows:
\begin{align}
    \theta_{\text{merge},l}^*
     & = \argmin_{\theta_{\text{merge},l}} \sum_{i=1}^K \mathbb{E}_{x_{i,l} \sim \mathcal{D}_{m_i,l}}\left\|\theta_{i,l} x_{i,l} - \theta_{\text{merge},l} x_{i,l}\right\|_2^2            \\
     & = \theta_{l} + \argmin_{\tau_{\text{merge},l}} \sum_{i=1}^K \mathbb{E}_{x_{i,l} \sim \mathcal{D}_{m_i,l}}\left\|\tau_{i,l} x_{i,l} - \tau_{\text{merge},l} x_{i,l}\right\|_2^2, \label{eq_wudi_1}
\end{align}
where $\theta$ denotes the parameters of the base model, $l$ denotes the $l$-th layer of the model, and $\mathcal{D}_{m_i,l}$ denotes the approximated data distribution for the $i$-th morphology expert at layer $l$.
By analyzing the update dynamics of the task vector, we derive the following upper bound on task interference:
\begin{equation}
\mathbb{E}_{\boldsymbol{x}_{i,l}\sim \mathcal{D}_{m_i,l}}\|(\theta_{i,l} - \tau_{\text{merge},l}) {x}_{i,l}\|_2^2\leq\omega_{i,l}^1\cdot\left\|(\theta_{i,l} - \tau_{\text{merge},l})({\tau}_{i,l})^\top\right\|_F^2 +\omega_{i,l}^2\cdot\left\|\theta_{i,l} - \tau_{\text{merge},l}\right\|_F^2,
\end{equation}
where $\omega_{i,l}^1$ and $\omega_{i,l}^2$ are constants, and the proof is provided in~\citep{chengWhoeverStartedInterference2025}.
Therefore, Eq.(\ref{eq_wudi_1}) can be rewrite as follows
\begin{equation}
     \theta_{\text{merge},l}^* \approx \theta_{pre,l} + \argmin_{\tau_{\text{merge,l}}} \sum_{i=1}^K \frac{1}{\|\tau_{i,l}\|_F^2} \left\| (\tau_{\text{merge},l} - \tau_{i,l})\tau_i^\top \right\|_F^2,
\end{equation}
The optimization of the merged task vector $\tau_{\text{merge}}$ is performed using the Adam optimizer with a learning rate of 1e-5, weight decay set to 0, and 1,000 iterations, implemented via the FusionBench framework~\citep{tangFusionBenchUnifiedLibrary2025}. The merging methods also can be changed to any other data-free methods, and here we also explored alternative merging techniques, including SVD-based Task Singular Vector Merging (TSVM)~\citep{gargiuloTaskSingularVectors2025} and vanilla parameter averaging (also known as ModelSoups)~\citep{wortsmanModelSoupsAveraging2022a,cheginiModelSoupBetter2024}.
We provide detailed comparisons between merging-based approaches and data mixing strategies in Section~\ref{sec:experiments}.

\subsection{Stage 4: Supervised Fine-Tuning on Embodied Data}

After cross-embodiment merging, the unified model $\theta_{\text{merged}}$ undergoes embodied-enhanced supervised fine-tuning to obtain $\theta_{\text{embodied}}$.
This stage focuses on further strengthening models' embodied capability.
The training data comprises large-scale embodied and ego-centric multimodal data pairs that emphasize consistent embodied interaction, task planning, and action prediction.
This refinement step ensures that the model can reliably interpret embodied instructions and generate appropriate responses while maintaining the spatial and cross-embodied capabilities acquired through merging.

\subsection{Stage 5: Reinforcement Learning with GRPO}

Finally, the $\theta_{\text{embodied}}$ is further refined through preference-based reinforcement learning to obtain $\theta_{\text{GRPO}}$ with 100k mixed data from spatial, ad, uav, and embodied corpus.
We adopt Group Relative Policy Optimization (GRPO)~\citep{shao2024deepseekmath}, which optimizes the model using relative rewards computed from multiple sampled responses to the same query.
Specifically, for each question $q$, GRPO samples a group of $G$ outputs $\{o_1, o_2, \ldots, o_G\}$ from the old policy $\pi_{\theta_{\text{old}}}$ and optimizes the policy by maximizing:
\begin{small}
    \begin{align}
        \mathcal{J}_{\text{GRPO}}(\theta) = \mathbb{E}_{\substack{q \sim P(Q) \\ \{o_i\}_{i=1}^G \sim \pi_{\theta_{\text{old}}}(\cdot|q)}}
        \frac{1}{G}\sum_{i=1}^{G} \frac{1}{|o_i|}\sum_{t=1}^{|o_i|}
        \left[
            \min\!\left(
            \frac{\pi_\theta(o_{i,t}|q, o_{i,<t})}{\pi_{\theta_{\text{old}}}(o_{i,t}|q, o_{i,<t})} \hat{A}_{i,t},\;
            \text{clip}\!\left(
            \frac{\pi_\theta(o_{i,t}|q, o_{i,<t})}{\pi_{\theta_{\text{old}}}(o_{i,t}|q, o_{i,<t})},\, 1{-}\varepsilon,\, 1{+}\varepsilon
            \right) \hat{A}_{i,t}
            \right)
            \right],
    \end{align}
\end{small}
where $\varepsilon$ is the clipping hyper-parameter for stabilizing training and $\hat{A}_{i,t}$ is the group-relative advantage.
Note that, refer to \cite{internvl3.5}, we omit the KL divergence penalty term against a reference policy in our implementation, as we empirically find that the clipped surrogate objective alone provides sufficient regularization for stable training in our embodied setting.
The advantage $\hat{A}_{i,t}$ is computed from group-normalized rewards without requiring a learned value function.
A reward model scores each output $o_i$, yielding rewards $\{r_1, r_2, \ldots, r_G\}$.
Under outcome supervision, the rewards are normalized within the group and assigned uniformly to all tokens:
\begin{equation}
    \hat{A}_{i,t} = \tilde{r}_i = \frac{r_i - \text{mean}(\mathbf{r})}{\text{std}(\mathbf{r})},
\end{equation}
where $\mathbf{r} = \{r_1, \ldots, r_G\}$.
By leveraging GRPO, the model learns to optimize for decision quality under uncertainty and the ability to handle multi-step task planning required in complex spatial and embodied scenarios.

\begin{table}[!t]
    \vspace{-0.2cm}
    \setlength{\tabcolsep}{6pt}
    \scriptsize
    \centering
    \caption{\textbf{Details of Scaffold-Specialize-Reconcile training strategy of ACE-Brain-0.}}
    \label{tab:training-recipe}
    \vspace{2mm}
    \renewcommand{\arraystretch}{1.3}
    \resizebox{\linewidth}{!}{
        \begin{tabular}{l|c|c|c|c|c}
            \toprule
                                           & \textbf{Stage-1}  & \textbf{Stage-2}                              & \textbf{Stage-3}          & \textbf{Stage-4}        & \textbf{Stage-5}   \\ \midrule
            \textbf{Target Objective}      & \textbf{Scaffold SFT} & \textbf{Specialize SFT}                         & \textbf{Expert Reconcile} & \textbf{Embodied SFT}   & \textbf{RTF}       \\
            \textbf{Data Domain}           & \tiny{Spatial}    & \tiny{AD,UAV} & - (Data-Free)             & \tiny{Embodied} & \tiny{Mixed}       \\
            \textbf{Base Model}            & \tiny{$\theta_{\text{base}}$}   & \tiny{$\theta_{\text{spatial}}$}            & \tiny{$\theta$, $\theta_{\text{spatial}}$, $\theta_{\text{ad}}$, $\theta_{\text{uav}}$}           & \tiny{$\theta_{\text{merge}}$}  & \tiny{$\theta_{\text{embodied}}$} \\ \midrule
            \textbf{Trainable Part}        & MLP, LLM          & MLP, LLM                                      & MLP, LLM                  & MLP, LLM                & MLP, LLM           \\
            \textbf{Per-device Batch Size} & 8                 & 8                                             & -                         & 8                       & 8 / 256            \\
            \textbf{Gradient Accumulation} & 4                 & 4                                             & 0                         & 4                       & 8                  \\
            \textbf{Epoch/Step}                 & 1                 & 1                                             & 1,000 steps               & 1                       & 1                  \\
            \textbf{Optimizer}             & AdamW             & AdamW                                         & Adam                      & AdamW                   & AdamW              \\
            \textbf{LR}                    & $5\times10^{-6}$  & $5\times10^{-6}$                              & $1\times10^{-5}$          & $5\times10^{-6}$        & $1\times10^{-6}$   \\
            \textbf{Deepspeed}             & Zero2             & Zero2                                         & -                         & Zero2                   & -                  \\
            \textbf{Weight Decay}          & 0.0               & 0.0                                           & 0.0                       & 0.0                     & 0.01               \\
            \textbf{Warmup Ratio}          & 0.03              & 0.03                                          & -                         & 0.03                    & 0.0                \\
            \textbf{LR Schedule}           & cosine            & cosine                                        & -                         & cosine                  & constant           \\
            \textbf{Min Pixels}            & 50176             & 50176                                         & -                         & 50176                   & 50176              \\
            \textbf{Max Pixels}            & 50176             & 50176                                         & -                         & 50176                   & 50176              \\
            \textbf{Video Min Pixels}      & 200,704           & 200,704                                       & -                         & 200,704                 & -                  \\
            \textbf{Video Max Pixels}      & 802,816           & 802,816                                       & -                         & 802,816                 & -                  \\
            \textbf{Model Max Length}      & 16384              & 16384                                         & -                         & 16384                   & 8192               \\ \midrule
        \end{tabular}}
    \vspace{-0.3cm}
    \label{table:training-recipe}
\end{table}

\section{Experiments}
\label{sec:experiments}

\subsection{Spatial Intelligence}

As shown in Table~\ref{tab:spatial_comprehension_comparison}, we evaluate \OurMethod{} on seven visual spatial benchmarks~(\textit{i.e.}, VSI~\cite{vsi}, MMSI~\cite{mmsi}, BLINK~\cite{blink}, SITE~\cite{site}, SAT~\cite{sat}, MindCube~\cite{yin2025spatial}, and Multi3DRef~\cite{multi3drefer}), collectively covering spatial relationship comprehension, 3D referring grounding, temporal memorization, viewpoint transformation, and mental spatial modeling. \OurMethod{}-8B achieves consistently strong performance across all benchmarks, matching or surpassing both closed-source MLLMs and state-of-the-art embodied brains.
    In spatial relationship and distance understanding, \OurMethod{} reaches 83.9\% on BLINK, outperforming Gemini2.5-Pro~(81.8\%), and 63.3\% on VSI, surpassing both Gemini2.5-Pro~(47.8\%) and the strongest embodied brain Vlaser~(60.3\%). For viewpoint transformation and mental modeling, it obtains 92.0\% on SAT, largely exceeding Gemini2.5-Pro~(79.3\%) and MiMo-Embodied-7B~(78.7\%), while on MindCube it achieves 82.1\%, far ahead of Gemini2.5-Pro~(57.6\%), GPT-4o~(46.1\%), and Vlaser-8B~(34.6\%).
    In scene understanding and 3D grounding, \OurMethod{} reaches 53.1\% on SITE, exceeding InternVL3.5-8B~(50.1\%) and RoboBrain2.5-8B~(52.6\%) while remaining competitive with Gemini2.5-Pro~(57.0\%); on Multi3DRef, our result of 59.6\% surpasses most open-source MLLMs and embodied brains, though VeBrain-7B~(67.8\%) retains the lead.

\begin{table*}[t]
\centering
\caption{\textbf{Performance Comparison on Spatial Benchmarks.} \textbf{Bold} numbers indicate the best results, \underline{underlined} numbers indicate the second-best results, and results marked with $^*$ are obtained using our evaluation framework.}
\label{tab:spatial_comprehension_comparison}
\setlength{\tabcolsep}{10pt}
\renewcommand{\arraystretch}{1.15}
\resizebox{\textwidth}{!}{
\begin{tabular}{l|ccccccc}
\toprule
\textbf{Model} &
\textbf{VSI} &
\textbf{MMSI} &
\textbf{BLINK} &
\textbf{SITE} &
\textbf{SAT} &
\textbf{Mindcube} &
\textbf{Multi3DRef}
\\
\midrule
\multicolumn{6}{l}{\textit{Closed-source MLLMs}} \\
\midrule
\textbf{GPT-4o~\cite{chatgpt4o}}                  & 43.6 & \underline{30.3} & 77.9 & 37.8 & 66.7 &46.1 & 8.1 \\
\textbf{Gemini-2.5-Pro~\cite{gemini}}          & \textbf{47.8} &\textbf{ 38.0} & \textbf{81.8 }& 57.0 & \textbf{79.3} &57.6 & -- \\
\textbf{Claude-4-Sonnet~\cite{claude4}}         & \underline{47.0} & --   & \underline{78.1} & --   & \underline{75.3} &36.6 & -- \\
\textbf{Qwen-VL-Max~\cite{bai2025qwen3vltechnicalreport}}             & 41.8 & --   & --   & --   & 56.7 &--   & -- \\

\midrule
\multicolumn{6}{l}{\textit{Open-source general-purpose MLLMs}} \\
\midrule
\textbf{MiMo-VL-7B~\cite{mimovl}}              & 36.4 & --   & --   & 37.6 & 59.3          &--       & 8.1 \\
\textbf{Qwen2.5-VL-7B-Inst.~\cite{qwen2.5-vl}}     & 32.3 & 26.8 & 82.5$^*$ & 31.4 & 52.0      &\underline{36.0}       & 21.1$^*$ \\
\textbf{InternVL3-8B~\cite{internvl3}}            & 42.1 & 25.7 & 82.9$^*$   & 41.1 & \textbf{72.7}$^*$      &\textbf{41.5}      & 8.1$^*$ \\
\textbf{InternVL3.5-8B~\cite{internvl3.5}}          & \underline{56.3} & \underline{30.2}$^*$ & \underline{84.1}$^*$   & \textbf{50.1}$^*$   & 59.3      &35.1$^*$   & 8.1$^*$ \\
\textbf{Qwen3-VL-2B-Inst.~\cite{bai2025qwen3vltechnicalreport}}       & 53.9 & 29.7$^*$ & 74.9 & 35.6 & \underline{67.3}$^*$  &34.5      & -- \\
\textbf{Qwen3-VL-8B-Inst.~\cite{bai2025qwen3vltechnicalreport}}       & \textbf{59.4} & \textbf{31.0}$^*$ & \textbf{85.2} & \underline{45.8} & 66.0$^*$  &29.4      & \textbf{54.6}$^*$ \\

\midrule
\multicolumn{6}{l}{\textit{Embodied Brain MLLMs}} \\
\midrule
\textbf{RoboBrain2.0-7B~\cite{robobrain2}}         & 36.1     & 27.9     & 81.4$^*$   & 49.2$^*$ & 75.3         &31.2$^*$     & 8.1$^*$ \\
\textbf{RoboBrain2.5-8B~\cite{robobrain2.5}}         & 41.0$^*$ & 29.3$^*$ & \underline{84.3}$^*$   & \underline{52.6}$^*$ & 63.3$^*$     &28.1$^*$     & 8.1$^*$ \\
\textbf{VeBrain-7B~\cite{vebrain}}              & 39.9     & 27.3$^*$ & 79.7       & 51.4$^*$ & 73.3$^*$     &30.1$^*$     & \textbf{67.8}$^*$ \\
\textbf{Pelican-VL-7B~\cite{pelican-vl}}           & 52.8     & 26.0$^*$ & 56.8       & 52.3$^*$ & 67.3$^*$     &31.0$^*$     & 7.9$^*$ \\
\textbf{MiMo-Embodied-7B~\cite{mimo-embodied}}        & 48.5     & \underline{31.7}$^*$ & 0.0$^*$    & 44.8     & \underline{78.7}         &32.3$^*$     & 8.1$^*$ \\
\textbf{Vlaser-8B~\cite{vlaser}}               & \underline{60.3}	    & 27.2	   &\textbf{84.9}$^*$	&47.5$^*$  &66.7$^*$	  &\underline{34.6}$^*$     & 8.1$^*$\\
\textbf{ACE-Brain-0-8B}            & \textbf{63.3} & \textbf{32.2} & 83.9 & \textbf{53.1} & \textbf{92.0} & \textbf{82.1}   & \underline{59.6} \\
\bottomrule
\end{tabular}}
\vspace{-0.25cm}
\end{table*}

\subsection{Autonomous Driving Intelligence}
\begin{table*}[t]
    \centering
    \caption{\textbf{Performance Comparison on Autonomous Driving Benchmarks.}}
    \label{tab:performance_comparison_driving}
    \setlength{\tabcolsep}{3pt}
    \renewcommand{\arraystretch}{1.15}
    \resizebox{\textwidth}{!}{
        \begin{tabular}{l|cccccc}
            \toprule
            \textbf{Model}               &
            \textbf{MME-RealWorld}       &
            \textbf{MAPLM}               &
            \textbf{DriveAction}         &
            \textbf{NuscenesQA}          &
            \textbf{NuPlanQA}            &
            \textbf{LingoQA}\\
            \midrule
            \multicolumn{6}{l}{\textit{Closed-source MLLMs}} \\
            \midrule
            \textbf{GPT-4o~\cite{chatgpt4o}}              & 58.0    & \textbf{26.6}  & 72.5  & \textbf{34.3}  & \textbf{81.5}  & 56.0   \\
            \textbf{Gemini2.5-pro~\cite{gemini}}       & \textbf{67.0}    & \underline{26.1}  & \textbf{73.5}  & \underline{16.1}  & --    & \textbf{64.1} \\
            \textbf{Claude-4-Sonnet~\cite{claude4}}     & --    & --    & --    & --    & --    & --   \\
            \textbf{Qwen-VL-Max~\cite{bai2025qwen3vltechnicalreport}}         & \underline{61.7}  & 24.8  & \underline{72.6}  & 6.7   & --    & \underline{58.8} \\

            \midrule
            \multicolumn{6}{l}{\textit{Open-source general-purpose MLLMs}} \\
            \midrule
            \textbf{MiMo-VL-7B~\cite{mimovl}}          & 54.1     & 31.0      & \underline{78.9}      & \textbf{33.9}       & --  & 54.8 \\
            \textbf{Qwen2.5-VL-7B-Inst.~\cite{qwen2.5-vl}} & \underline{58.6}     & 24.8      & 73.4      & \underline{25.8}       & 41.8      & \underline{55.6} \\
            \textbf{InternVL3-8B~\cite{internvl3}}        & 52.1$^*$ & \underline{31.0}$^*$  & 77.6$^*$  & 26.6$^*$    & 82.7$^*$  & 50.8$^*$ \\
            \textbf{InternVL3.5-8B~\cite{internvl3.5}}      & 49.2     & 14.2    & 78.1      & 17.2       & 83.9$^*$  & 46.7 \\
            \textbf{Qwen3-VL-2B-Inst.~\cite{bai2025qwen3vltechnicalreport}}   & --       & 20.5$^*$  & 77.1$^*$  & --         & 66.7$^*$        & 34.8$^*$   \\
            \textbf{Qwen3-VL-8B-Inst.~\cite{bai2025qwen3vltechnicalreport}}   & \textbf{63.3}$^*$ & \textbf{31.5}$^*$  & \textbf{79.0}$^*$  & 22.8$^*$   & \textbf{82.2}$^*$  & \textbf{57.0}$^*$   \\

            \midrule
            \multicolumn{6}{l}{\textit{Embodied Brain MLLMs}} \\
            \midrule
            \textbf{RoboBrain2.0-7B~\cite{robobrain2}}     & 59.6$^*$ & 31.7$^*$ & 80.9$^*$  & 32.3$^*$     & 82.8$^*$ & 39.2$^*$    \\
            \textbf{RoboBrain2.5-8B~\cite{robobrain2.5}}     & 60.0$^*$  & 22.5$^*$ & 80.5$^*$ & 33.2$^*$     & 79.3$^*$   & 48.0$^*$    \\
            \textbf{VeBrain-7B~\cite{vebrain}}          & 60.1$^*$ & 22.9$^*$ & 78.3$^*$ & 29.3$^*$     & 82.9$^*$ & 55.0$^*$    \\
            \textbf{Pelican-VL-7B~\cite{pelican-vl}}       & 57.9$^*$ & 24.4$^*$ & 77.2$^*$ & 14.8        & 83.4$^*$ & 56.0$^*$    \\
            \textbf{MiMo-Embodied-7B~\cite{mimo-embodied}}    & \underline{60.3}     & \underline{74.5}     & \underline{81.0}     & \underline{56.7}        & 73.7$^*$ & \textbf{69.9 }\\
            \textbf{Vlaser-8B~\cite{vlaser}}           & 41.6$^*$	& 29.1$^*$ & 78.1$^*$ & 33.1$^*$	    & 78.3$^*$ & 59.6* \\
            \textbf{ACE-Brain-0-8B}                & \textbf{71.2}     & \textbf{77.8}       &\textbf{81.3 }     &\textbf{58.8}    &\textbf{91.7}       & \underline{65.8}   \\

            \bottomrule
        \end{tabular}
    }
    
\end{table*}

As shown in Table~\ref{tab:performance_comparison_driving}, we evaluate \OurMethod{} on six autonomous driving benchmarks (MME-RealWorld~\cite{zhang2024mme}, MAPLM~\cite{maplm}, DriveAction~\cite{driveaction}, NuscenesQA~\cite{nuscenesqa}, NuPlanQA~\cite{park2025nuplanqa}, and LingoQA~\cite{lingoqa}), collectively covering multi-view traffic perception, planning-aware language modeling, action recognition, ego-centric scene understanding, and language-grounded driving semantics.

\OurMethod{}-8B achieves consistently strong results across all
benchmarks, outperforming closed-source MLLMs, open-source general MLLMs, and state-of-the-art embodied brains in most cases. In object understanding under the driving scenarios, \OurMethod{} reaches 71.2\% on MME-RealWorld, surpassing Gemini2.5-Pro~(67.0\%), Qwen3-VL-8B-Inst~(63.3\%), and MiMo-Embodied-7B (60.3\%), and obtains 77.8\% on MAPLM, substantially exceeding the strongest embodied brain baseline MiMo-Embodied-7B~(74.5\%). In action understanding and scene QA, it achieves 81.3\% on DriveAction, outperforming Gemini2.5-Pro~(73.5\%) and MiMo-Embodied-7B~(81.0\%), while on NuscenesQA it reaches 58.8\%, largely exceeding GPT-4o~(34.3\%) and MiMo-Embodied-7B~(56.7\%). In physical kinematics comprehension and decision making, \OurMethod{} attains 91.7\% on NuPlanQA, where models need integrate surround-view inputs with kinematic cues to produce driving justifications, outperforming Pelican-VL-7B (83.4\%) and RoboBrain2.0-7B~(82.8\%); on LingoQA, it achieves 65.8\%, exceeding Gemini2.5-Pro~(64.1\%) and Vlaser-8B~(59.6\%), demonstrating the ability to generate interpretable, causally grounded behavior descriptions.

Overall, these results suggest that \OurMethod{} does not merely memorize driving templates but learns an ego-centric, multi-view consistent driving representation that bridges perception, kinematics, and language into decision-relevant reasoning, supporting reliable next-step prediction and interpretable interaction within complex traffic environments.

\subsection{Low-Altitude Intelligence}

\begin{table*}[t]
\centering
\caption{\textbf{Performance Comparison on Low-Altitude Benchmarks.}}
\label{tab:uav_comprehension_comparison}
\setlength{\tabcolsep}{6pt}
\renewcommand{\arraystretch}{1.15}
\resizebox{\textwidth}{!}{
\begin{tabular}{l|ccccc}
\toprule
\textbf{Model} &
\textbf{UrbanVideo-Bench} &
\textbf{AircopBench} &
\textbf{Avi-Math} &
\textbf{Airspatial-VQA} &
\textbf{HRVQA}\\
\midrule
\multicolumn{6}{l}{\textit{Closed-source MLLMs}} \\
\midrule
\textbf{GPT-4o~\cite{chatgpt4o}}               & \underline{43.6} & \textbf{51.8} &\textbf{33.5} & \textbf{192.4} & \textbf{36.9} \\
\textbf{Gemini2.5-pro~\cite{gemini}}        &  --  & 49.1 &  --  & -- & -- \\
\textbf{Claude-4-Sonnet~\cite{claude4}}      &  --  & \underline{50.7} &  --  & -- & -- \\
\textbf{Qwen-VL-Max~\cite{bai2025qwen3vltechnicalreport}}          & \textbf{45.5} & 50.5 &  --  & -- & -- \\

\midrule
\multicolumn{6}{l}{\textit{Open-source general-purpose MLLMs}} \\
\midrule
\textbf{MiMo-VL-7B~\cite{mimovl}}                   & --         & 48.6$^*$ & --       & --           & 27.3$^*$ \\
\textbf{Qwen2.5-VL-7B-Inst.~\cite{qwen2.5-vl}}          & \underline{34.6}$^*$   & 47.3     & \textbf{27.9}     & --           & 27.1$^*$ \\
\textbf{InternVL3-8B~\cite{internvl3}}                 & 30.8$^*$   & 52.2     & \underline{26.8}     & \textbf{337.0}        & \textbf{37.6}$^*$ \\
\textbf{InternVL3.5-8B~\cite{internvl3.5}}               & \underline{34.6}$^*$   & 23.6$^*$ & 23.3$^*$ & 1507.0       & 14.6$^*$ \\
\textbf{Qwen3-VL-2B-Inst.~\cite{bai2025qwen3vltechnicalreport}}            & 30.5$^*$   & 38.3$^*$ & 11.6$^*$ & --           & \underline{37.5}$^*$ \\
\textbf{Qwen3-VL-8B-Inst.~\cite{bai2025qwen3vltechnicalreport}}            & \textbf{39.}2$^*$   & 48.7$^*$ & 25.6$^*$ & \underline{1254.4}$^*$   & 37.4$^*$ \\

\midrule
\multicolumn{6}{l}{\textit{Embodied Brain MLLMs}} \\
\midrule
\textbf{RoboBrain2.0-7B~\cite{robobrain2}}              & 30.0$^*$   & 47.2$^*$ & 22.1$^*$ & 764.4$^*$    & 19.2$^*$ \\
\textbf{RoboBrain2.5-8B~\cite{robobrain2.5}}              & \underline{37.5}$^*$ & 49.9$^*$ & 26.1$^*$ & 1509.3$^*$   & 13.4$^*$ \\
\textbf{VeBrain-7B~\cite{vebrain}}                   & 36.5$^*$ & \underline{51.9}$^*$ & 25.4$^*$ & 1583.4$^*$   & 37.9$^*$ \\
\textbf{Pelican-VL-7B~\cite{pelican-vl}}                & 37.1$^*$ & 50.8$^*$ & 22.5$^*$ & 1586.6$^*$   & \underline{38.6}$^*$ \\
\textbf{MiMo-Embodied-7B~\cite{mimo-embodied}}             & 26.0$^*$ & 50.2$^*$ & \underline{33.7}$^*$ & \underline{289.4}$^*$    & 22.2$^*$ \\
\textbf{Vlaser-8B~\cite{vlaser}}                    & 30.4$^*$ & 25.3$^*$	& 19.3$^*$ & 1597.7$^*$   & 27.0$^*$   \\
\textbf{ACE-Brain-0-8B}                 & \textbf{56.9}    & \textbf{70.3}     &\textbf{35.0}     & \textbf{258.0} & \textbf{61.2} \\

\bottomrule
\end{tabular}
}
\end{table*}

As shown in Table~\ref{tab:uav_comprehension_comparison}, we evaluate \OurMethod{} on five low-altitude UAV benchmarks~(\textit{i.e.}, UrbanVideo-Bench~\cite{zhao2025urbanvideo}, AircopBench~\cite{aircopbench}, Avi-Math~\cite{zhou2025multimodal}, Airspatial-VQA~\cite{airspatialbot}, and HRVQA~\cite{hrvqa}), collectively covering aerial location, navigation, bird's-eye traffic-scene reasoning, and geometry-aware visual computation under drastic viewpoint changes.

\OurMethod{}-8B delivers consistently strong results across all
benchmarks, with clear advantages in the most decision-relevant UAV
settings. In aerial location and safety-critical scene reasoning,
\OurMethod{} achieves 56.9\% on UrbanVideo-Bench, outperforming
Qwen-VL-Max~(45.5\%), GPT-4o~(43.6\%), and RoboBrain2.5-8B~(37.5\%),
and reaches 70.3\% on AircopBench, which requires resolving
topology-aware spatial relations such as crosswalk occupancy and
lane-aware ordering, largely surpassing GPT-4o~(51.8\%) and InternVL3-8B~(52.2\%). In quantitative aerial reasoning and high-resolution
understanding, \OurMethod{} obtains 35.0\% on Avi-Math, outperforming
Qwen2.5-VL-7B-Inst.~(27.9\%) and MiMo-Embodied-7B~(33.7\%),
demonstrating competence in region-conditioned counting, altitude
estimation, and physics-based computation (Fig.~\ref{avimath57}); on HRVQA,it achieves 61.2\%, significantly exceeding InternVL3-8B~(37.6\%) and Pelican-VL-7B~(38.6\%).

\subsection{Embodied Egocentric Intelligence}

\begin{table*}[t]
    \centering
    \caption{\textbf{Performance Comparison on Embodied Benchmarks.}}
    \label{tab:performance_comparison_embodied}
    \setlength{\tabcolsep}{8.3pt}
    \renewcommand{\arraystretch}{1.15}
    \resizebox{\textwidth}{!}{
        \begin{tabular}{l|cccccc}
            \toprule
            \textbf{Model} &
            \textbf{ERQA} &
            \textbf{RoboVQA} &
            \textbf{OpenEQA} &
            \textbf{EmbSpatial} &
            \textbf{Ego-Plan2} &
            \textbf{EB-Habitat}
            \\
            \midrule
            \multicolumn{6}{l}{\textit{Closed-source MLLMs}} \\
            \midrule
            \textbf{GPT-4o~\cite{chatgpt4o}}          & \underline{47.0}    & 3.3   & 56.4    & \underline{71.9}  & 41.8  & \textbf{59.0}   \\
            \textbf{Gemini2.5-pro~\cite{gemini}}   & \textbf{48.3}  & --   & --    & \textbf{78.7}  & \underline{42.9}  & --   \\
            \textbf{Claude-4-Sonnet~\cite{claude4}} & --    & --   & --    & 64.3  & 41.3  & --   \\
            \textbf{Qwen-VL-Max~\cite{bai2025qwen3vltechnicalreport}}     & --    & --  & --    & --    & \textbf{44.7}  & \underline{45.3} \\

            \midrule
            \multicolumn{6}{l}{\textit{Open-source general-purpose MLLMs}} \\
            \midrule
            \textbf{MiMo-VL-7B~\cite{mimovl}}          & 37.8  & 35.3      & --       & --       & 34.1      & --   \\
            \textbf{Qwen2.5-VL-7B-Inst.~\cite{qwen2.5-vl}} & 38.8  & \textbf{57.2}$^*$  & 59.9$^*$ & --       & 39.7      & 14.3 \\
            \textbf{InternVL3-8B~\cite{internvl3}}        & 35.3  & 29.8$^*$  & \textbf{68.3}$^*$ & \underline{73.9}$^*$ & 37.9$^*$  & 24.3 \\
            \textbf{InternVL3.5-8B~\cite{internvl3.5}}      & \underline{41.5}  & 28.6      & 64.9$^*$ & 70.3     & \underline{42.9}      & \textbf{32.0}$^*$   \\
            \textbf{Qwen3-VL-2B-Inst.~\cite{bai2025qwen3vltechnicalreport}}   & 28.3  & 18.2$^*$  & 56.4$^*$ & 69.2$^*$ & 33.6$^*$  & --   \\
            \textbf{Qwen3-VL-8B-Inst.~\cite{bai2025qwen3vltechnicalreport}}   & \textbf{45.8}  & \underline{47.0}$^*$  & \underline{67.1}$^*$ & \textbf{78.5}$^*$ & \textbf{53.5}$^*$  & \underline{27.7}$^*$ \\

            \midrule
            \multicolumn{6}{l}{\textit{Embodied Brain MLLMs}} \\
            \midrule
            \textbf{RoboBrain2.0-7B~\cite{robobrain2}}  & 42.5$^*$ & 6.6$^*$  & 60.0$^*$   & \underline{76.3}$^*$ & 33.2     & 29.3   \\
            \textbf{RoboBrain2.5-8B~\cite{robobrain2.5}}  & \underline{44.3$^*$} & 18.7$^*$ & 62.6$^*$ & 75.6$^*$ & 44.9$^*$ & 26.3$^*$   \\
            \textbf{VeBrain-7B~\cite{vebrain}}       & 40.3$^*$ & 24.7$^*$  & 63.8$^*$ & 70.5$^*$ & 27.3     & 15.0$^*$   \\
            \textbf{Pelican-VL-7B~\cite{pelican-vl}}    & 39.8     & 23.6$^*$ & 63.3$^*$ & 73.2$^*$ & 39.4$^*$ & 16.3$^*$   \\
            \textbf{MiMo-Embodied-7B~\cite{mimo-embodied}} & \textbf{46.8}     & \underline{32.8}$^*$      & \textbf{74.1}$^*$ & 76.2$^*$ & 43.0 & 16.7$^*$   \\
            \textbf{Vlaser-8B~\cite{vlaser}}        &41.0	     &7.9$^*$	&56.3$^*$   &75.3$^*$   &\underline{53.4}    &\underline{40.0}   \\
            \textbf{ACE-Brain-0-8B}     & 41.5     & \textbf{64.6} & \underline{70.0} & \textbf{77.3} & \textbf{55.3} & \textbf{42.3}  \\

            \bottomrule
        \end{tabular}
    }
\end{table*}

As shown in Table~\ref{tab:performance_comparison_embodied}, we evaluate \OurMethod{} on six embodied benchmarks, including ERQA~\cite{gemini_robotics}, RoboVQA~\cite{robovqa}, OpenEQA~\cite{majumdar2024openeqa}, EmbSpatial-Bench~\cite{embspatial}, EgoPlan-Bench2~\cite{egoplanbench}, and EmbodiedBench(EB)-Habitat~\cite{embodiedbench}, collectively covering embodied interaction, egocentric planning, and {next-step prediction} under temporal observations. 

\OurMethod{}-8B achieves consistently strong results across all
benchmarks, outperforming or matching both general-purpose MLLMs and
state-of-the-art embodied brains. In embodied interaction and egocentric scene understanding, \OurMethod{} reaches 64.6\% on RoboVQA, surpassing GPT-4o~(34.5\%), Gemini2.5-Pro (33.9\%), Qwen2.5-VL-7B-Inst.~(57.2\%), and MiMo-Embodied-7B (32.8\%), and obtains 70.0\% on OpenEQA, exceeding Qwen3-VL-8B-Inst.~(67.1\%), VeBrain-7B~(63.8\%), and RoboBrain2.5-8B~(62.6\%), while remaining competitive with MiMo-Embodied-7B~(74.1\%). In sequential decision-making and temporal reasoning, \OurMethod{} achieves 55.3\% on EgoPlan-Bench2, delivering the best overall result and outperforming Qwen3-VL-8B-Inst.~(53.5\%), Vlaser-8B~(53.4\%), and RoboBrain2.5-8B~(44.9\%); on EB-Habitat, it obtains 41.7\%, exceeding Vlaser-8B~(40.0\%) and RoboBrain2.0-7B~(29.3\%), showing solid generalization to longer-horizon embodied environments. In embodied spatial comprehension,
\OurMethod{} reaches 77.3\% on EmbSpatial-Bench, largely exceeding
RoboBrain2.0-7B~(76.3\%) and InternVL3-8B~(73.9\%), while staying close
to the strongest results of Gemini2.5-Pro~(78.7\%) and Qwen3-VL-8B-Inst.~(78.5\%).

\section{Ablation Study}
\label{sec:ab}
\subsection{Spatial Intelligence as a Shared Scaffold}

It is widely believed that strong visual spatial intelligence benefits physical world comprehension, yet its actual impact is rarely quantified, especially across embodiments. To explicitly measure
how spatial knowledge transfers, we compare multiple training and adaptation routes on three domains: autonomous driving~(AD), low-altitude
aerial intelligence~(UAV), and embodied intelligence~(Embodied). 
As summarized in Table~\ref{tab:spatial_scaffold_transfer}, domain experts trained directly from the base model~(Qwen3-VL-8B-Instruct) already yield
clear gains over the base model on AD and UAV benchmarks, confirming that domain-specific data alone improves in-domain performance. However, we carefully find that there is 1.9\% performance degradation on the Embodied benchmark compared with the base model. Compared with AD and UAV benchmarks, we argue that the Embodied benchmark needs more fine-grained capability in action understanding~(fine-grained manipulation vs. coarse-grained planning), making it difficult to transfer directly from the general domain of the base model to the embodied domain.
Crucially, when initializing these experts from a spatial-centric pretrained checkpoint, our framework yields substantial and consistent improvements: +19.3\% in AD, +16.5\% in UAV, and +5.4\% in Embodied over the base model. These results provide compelling empirical evidence that spatial knowledge functions as a transferable structural scaffold that catalyzes learning across universe embodiments, rather than an isolated capability confined to spatial understanding benchmarks.

\begin{table*}[t]
\centering
\caption{\textbf{Spatial knowledge consistently improves expert performance.}
We compare different pretraining routes for adapting to three domains (AD, UAV, and Embodied).
The average score of AD is computed over NuscenesQA, NuPlanQA, and LingoQA benchmarks; 
the average score of UAV is computed over UrbanVideoBench, AircopBench, and Avi-Math; 
the average score of Embodied is computed over RoboVQA and EgoPlan. 
\textbf{Bold} indicates the best score, and improvements ($\Delta$) are reported over Qwen3-VL-8B-Instruct. 
}
\label{tab:spatial_scaffold_transfer}
\scriptsize
\setlength{\tabcolsep}{8pt}
\renewcommand{\arraystretch}{1.36}
\resizebox{\textwidth}{!}{
\begin{tabular}{l|cc|cc|cc}
\toprule
\multirow{2}{*}{\textbf{Initialization / Route}} 
& \multicolumn{2}{c|}{\textbf{AD}} 
& \multicolumn{2}{c|}{\textbf{UAV}} 
& \multicolumn{2}{c}{\textbf{Embodied}} \\
& \textbf{Avg.} & $\Delta$ 
& \textbf{Avg.} & $\Delta$ 
& \textbf{Avg.} & $\Delta$ \\
\midrule

\textbf{Qwen3-VL-8B-Instruct ($\theta$)} 
& 47.0 & -- 
& 37.8 & -- 
& \underline{52.7} & -- \\

\textbf{AD Experts} ($\theta \rightarrow \theta_A$) 
& \underline{58.1} & \underline{+11.1}
& -- & --
& -- & -- \\
\textbf{UAV Experts} ($\theta \rightarrow \theta_U$) 
& -- & --
& \underline{48.8} & \underline{+11.0} 
& -- & -- \\
\textbf{Embodied Experts} ($\theta \rightarrow \theta_E$) 
& -- & --
& --& --
& 50.8 & -1.9 \\
\midrule
\textbf{Spatial $\rightarrow$ AD Expert} ($\theta_{S}\rightarrow \theta_A$) 
& \textbf{72.6} & \textbf{+25.6} 
& -- & --
& -- & --\\
\textbf{Spatial $\rightarrow$ UAV Expert} ($\theta_{S}\rightarrow \theta_U$) 
& --& --
& \textbf{54.3} & \textbf{+16.5}  
& -- & -- \\
\textbf{Spatial $\rightarrow$ Embodied Expert} ($\theta_{S}\rightarrow \theta_E$) 
& -- & --
& -- & --
& \textbf{58.1} & \textbf{+5.4 }\\

\bottomrule
\end{tabular}
}
\end{table*}

\subsection{Importance of Data-free Model Merging~(Reconcile)}

A fundamental challenge in developing generalist physical intelligence lies in the effective integration of heterogeneous domain knowledge. In this work, we investigate the efficient composition of domain-specific experts into a unified model through data-free parameter merging. Specifically, we evaluate three merging strategies, which include the naive weight averaging, TSVM, and WUDI on three domains: Spatial, AD, and UAV. Notably, all methods operate without requiring additional training data during composition, with results summarized in Tab.~\ref{tab:merge_vs_joint}. Across all domains, parameter merging consistently outperforms the base model, demonstrating its ability to effectively synthesize domain expertise. While simple averaging yields noticeable improvements, TSVM further enhances performance, suggesting that domain-specific knowledge is largely complementary and can be reconciled at the parameter level. Among the evaluated strategies, WUDI achieves the most robust results. It not only leads in all domains but also surpasses the strongest individual specialists~(\textit{e.g.}, 76.7\% in Spatial and 68.1\% in AD). This highlights a super-additive composition effect rather than mere parameter ensembling, positioning optimized merging as a pivotal mechanism for integrating complementary expertise into unified models.

\begin{table*}[t]
    \centering
    \caption{\textbf{Optimized Merging effectively combines each domain.} The average score
        of Spatial is computed over VSI, SAT, and MindCube benchmarks.
    }
    \label{tab:merge_vs_joint}
    \scriptsize
    \setlength{\tabcolsep}{10pt}
    \renewcommand{\arraystretch}{1.5}
    \resizebox{\textwidth}{!}{
        \begin{tabular}{l|cc|cc|cc}
            \toprule
            \multirow{2}{*}{\textbf{Training / Synthesis Strategy}}
                                                           & \multicolumn{2}{c|}{\textbf{Spatial}}
                                                           & \multicolumn{2}{c|}{\textbf{AD}}
                                                           & \multicolumn{2}{c}{\textbf{UAV}}                                                                       \\
                                                           & \textbf{Avg.}                         & $\Delta$
                                                           & \textbf{Avg.}                         & $\Delta$
                                                           & \textbf{Avg.}                         & $\Delta$                                                       \\
            \midrule
            \textbf{Qwen3-VL-8B-Instruct} ($\theta$)       & 51.6                                  & --             & 47.0 & --             & 37.8 & --             \\

            \textbf{Spatial} ($\theta_{S}$)
                                                           & 72.5                                  & +20.9          & --   & --             & --   & --             \\

            \textbf{Spatial$\rightarrow$AD} ($\theta_{S} \rightarrow \theta_A$)
                                                           & --                                    & --             & 72.6 & +25.6          & -- & --         \\
            \textbf{Spatial$\rightarrow$UAV} ($\theta_{S} \rightarrow \theta_U$)
                                                           & --                                    & --             & -- & --          & \textbf{54.3} & \textbf{+16.5}          \\

            \midrule
            \rowcolor{gray!12}
\multicolumn{7}{l}{Merging weights of $\theta_0$, $\theta_{0\rightarrow S}$, $\theta_{S\rightarrow A}$, $\theta_{S\rightarrow U}$} \\
            \textbf{AVG Merging} ($\theta_{\text{avg}}$) & 71.6                                    & +20.0     & 66.6  & +19.6         & 48.0   & +10.2         \\
            \textbf{TSVM Merging} ($\theta_{\text{tsvm}}$)   & \underline{74.8}                                    & \underline{+23.2}          & \underline{72.8}   & \underline{+25.8}         & 51.4  & +13.6          \\
            
            \textbf{WUDI Merging} ($\theta_{\text{wudi}}$)   & \textbf{76.7}                                   & \textbf{+25.1}   & \textbf{72.9} & \textbf{+25.9}   & \underline{52.6}& \underline{+14.8} \\
            \bottomrule
        \end{tabular}
    }
\end{table*}

\subsection{Effectiveness of Scaffold-Specialize-Reconcile Training Paradigm}
\begin{table*}[t]
\centering
\caption{\textbf{Training paradigm comparison.}
We compare four paradigms: 
\textbf{Joint} trains on mixed \{Spatial, AD, UAV, Embodied\} data; 
\textbf{Sequential} performs stage-wise adaptation Spatial$\rightarrow$AD$\rightarrow$UAV$\rightarrow$Embodied; 
\textbf{SR} follows Spatial$\rightarrow$Merging$\rightarrow$Embodied. 
}
\label{tab:training_paradigm}
\scriptsize
\setlength{\tabcolsep}{8.3pt}
\renewcommand{\arraystretch}{1.36}
\resizebox{\textwidth}{!}{
\begin{tabular}{l|cc|cc|cc|cc}
\toprule
\multirow{2}{*}{\textbf{Training Paradigm}} 
& \multicolumn{2}{c|}{\textbf{Spatial}} 
& \multicolumn{2}{c|}{\textbf{AD}} 
& \multicolumn{2}{c|}{\textbf{UAV}} 
& \multicolumn{2}{c}{\textbf{Embodied}} \\
& \textbf{Avg.} & $\Delta$ 
& \textbf{Avg.} & $\Delta$ 
& \textbf{Avg.} & $\Delta$ 
& \textbf{Avg.} & $\Delta$ \\
\midrule
\textbf{Spatial} ($\theta_{S}$)
    & 72.5 & -- & -- & -- & --  & -- & -- & --\\
\textbf{Spatial $\rightarrow$ AD Expert} ($\theta_{S}\rightarrow \theta_A$) &  -- & --
& 72.6 & --
& -- & --
& -- & --\\
\textbf{Spatial $\rightarrow$ UAV Expert} ($\theta_{S}\rightarrow \theta_U$) & -- & --
& --& --
& \textbf{54.3} & --
& -- & -- \\
\textbf{Spatial $\rightarrow$ Embodied Expert} ($\theta_{S}\rightarrow \theta_E$) & --& --
& -- & --
& -- & --
& 58.1 & --\\
\textbf{Joint Training} & 68.0 & -4.5 & 65.3 & -7.3 & 45.7 & -8.6 & 56.8 & -1.3 \\
\textbf{Sequential Training} & 67.6 & -4.9 & 70.1 & -2.5 & 50.8 & -3.5 & 59.0 &  +0.9\\
\midrule
\textbf{SSR Training} & 78.5 & +6.0 &72.0 & -0.6 & 50.9 & -3.2& 59.7 & +1.6\\
\textbf{SSR Training w/ GRPO} & \textbf{79.1} & \textbf{+6.6} &\textbf{72.1} & \textbf{-0.5} & \underline{54.1} & \underline{-0.2}& \textbf{60.0} & \textbf{+1.9}\\
\bottomrule
\end{tabular}
}
\end{table*}

Building on the preceding empirical analyses, we introduce the Scaffold-Specialize-Reconcile~(SSR) training paradigm for cross-embodied physical world comprehension.
Under matched total training budgets, we compare SSR against two prevalent alternatives: joint training on mixed data and sequential multi-stage adaptation without merging.

As shown in Table~\ref{tab:training_paradigm}, joint training, directly mixing spatial, AD, and UAV data prior to embodied finetuning, yields limited and inconsistent gains across domains. Sequential training without synthesis improves the final target domain (Embodied) but incurs catastrophic forgetting on previously learned capabilities, corroborating the degradation patterns observed in earlier sections.
In contrast, SSR achieves consistently strong performance across all domains. The paradigm proceeds in four stages: establishing a domain-agnostic spatial scaffold, training domain-specialized experts, reconciling these experts through optimized merging, and finally performing embodied alignment. This structured decomposition effectively integrates heterogeneous physical knowledge while preserving domain-specific strengths. Notably, SSR improves Embodied performance without sacrificing Spatial, AD, or UAV capabilities, achieving superior generalization-specialization trade-offs compared to baselines.

These findings underscore the importance of expert synthesis in embodied learning. While model merging has been extensively studied in language and vision domains, its application to physical intelligence remains underexplored. Our results indicate that synthesis constitutes more than a parameter-level heuristic-it serves as a fundamental mechanism for organizing and reusing heterogeneous physical knowledge. Collectively, the empirical evidence supports the SSR paradigm as a principled training strategy for multi-robot and cross-embodied foundation models.

Building upon the preceding empirical analyses, we propose the Scaffold-Specialize-Reconcile~(SSR) training paradigm for cross embodiments. To verify the effectiveness of the SSR training paradigm, we provide two prevalent strategies, including the joint training strategy on mixed datasets and the sequential training strategy by using multi-stage adaptation without merging.
As shown in Table~\ref{tab:training_paradigm}, the joint training strategy by mixing Spatial, AD, and UAV data before embodied fine-tuning yields only marginal gains. Although the sequential training strategy enhances performance in the final target domain~(\textit{e.g.}, Embodied), it suffers from catastrophic forgetting of previously acquired capabilities, leading to worse performance in earlier domains~(\textit{e.g.}, AD and UAV). In contrast, SSR delivers robust performance across all domains. 
This structured decomposition effectively integrates heterogeneous physical knowledge while preserving domain-specific proficiencies. Notably, SSR bolsters Embodied performance without compromising Spatial, AD, or UAV capabilities, achieving a superior generalization-specialization trade-off compared to existing training paradigms.

\section{Training Datasets}
\label{sec:data}

A central challenge in training a generalist embodied brain lies in harmonizing different embodiment data. These datasets diverge not only in visual appearance but also in their implicit spatial assumptions, temporal organization, and task-level planning demands.
To this end, we construct a training corpus shown in Fig.~\ref{statistic}. It spans general multimodal instruction, spatial intelligence, autonomous driving, and low-altitude aerial datasets. Next, we will briefly introduce them. 

\subsection{General Datasets}
We begin by establishing a shared multimodal foundation that anchors visual perception, language understanding, and instruction following.
Here, we adopt the Cambrain-737K~\cite{cambrian} visual instruction-tuning dataset. It samples built upon the LLaVA-665K~\cite{llava} foundation, augmented with additional OCR-rich and chart-understanding corpora, to enhance document perception capabilities. This composition ensures balanced coverage across general VQA, text-rich image comprehension, and structured reasoning tasks, providing a robust corpus for aligning visual representation with language instruction.
\begin{figure*}
    \centering
    \includegraphics[width=0.99\textwidth]{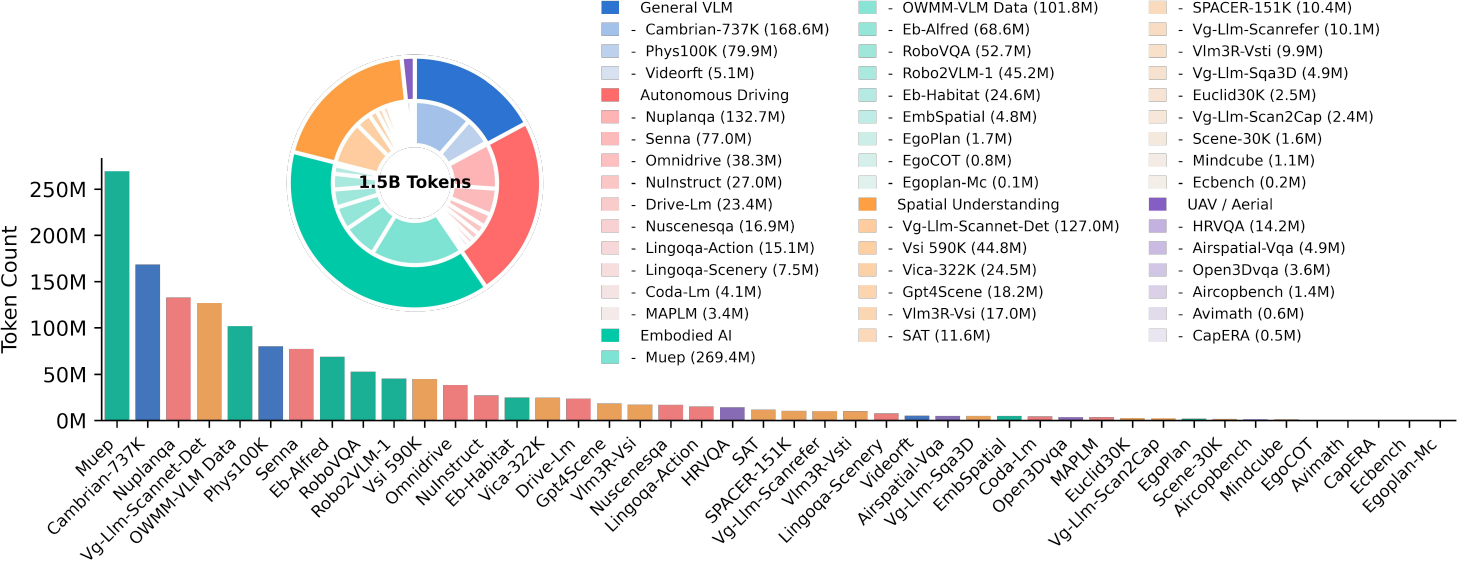}
    \caption{
        \textbf{Domain distribution and Token count.}
        This nested pie chart illustrates the proportion of tokens contributed by each domain.
        The distribution exhibits a long-tailed characteristic, where UAV data constitutes a relatively small proportion of the corpus.
    }
    \label{statistic}
\end{figure*}
\subsection{Spatial Intelligence Datasets}
Agent interaction with the world implicitly relies on consistent spatial-temporal representations.
To support this requirement, we include spatial intelligence datasets that focus on object relations, distance measurement, and spatial layouts in both static and dynamic video inputs.

\textbf{VSI-590K}~\cite{vsi} is a large-scale visual spatial intelligence dataset comprising 590K QA pairs. The dataset spans diverse question types, including size, direction, distance, count, and temporal order across viewpoints and modalities (images and videos).

\textbf{SAT}~\cite{sat} is a generated spatial dataset built on photo-realistic ProcTHOR-10K~\cite{procthor} environments, containing 175K QA pairs across 22K indoor scenes without manual 3D annotation.

\textbf{VICA-322K}~\cite{vica322k} is a large-scale video-based spatial dataset built on real-world indoor videos with high-quality 3D annotations. It combines metadata-supervised spatial tasks (e.g., object count, size, distance, and room scale) that require holistic interpretation of spatial layouts and object relationships over videos.

\textbf{GPT4Scene}~\cite{gpt4scene} is a large-scale indoor 3D scene understanding dataset with 165K annotations paired with processed videos for VLM fine-tuning. It integrates sampled video frames, reconstructed 3D point clouds, BEV images, and spatio-temporal object markers to maintain consistent object identities across views and time. Built via 3D reconstruction and instance segmentation pipelines, it supports vision-only 3D spatial tasks such as 3D QA, dense captioning, and visual grounding.

\textbf{Scene-30K}~\cite{scene_30k} is a synthetic 30K CoT dataset for enhancing 3D VLM reasoning. It covers a wide range of 3D scene understanding tasks, including captioning, QA, grounding, dialogue, reasoning, and planning. Constructed from existing 3D-VL datasets and generated using a Gemini-2.5-Pro-based pipeline, it provides multi-step reasoning annotations filtered for structural correctness and semantic consistency.

\textbf{VLM-3R}~\cite{vlm_3r} consists of two datasets, VSI and VSTI, targeting 3D spatial understanding from static scenes and monocular videos, respectively. VLM-3R-VSI contains over 200K QA pairs. Similar to VSI~\cite{vsi}, it is also built from ScanNet~\cite{scannet}, ScanNet++~\cite{scannet++}, and ARKitScenes~\cite{arkitscenes}.
VLM-3R-VSTI comprises 138.6K video QA pairs for spatio-temporal understanding, spanning camera dynamics, camera-object interactions, and object relative position. It focuses on estimating camera motion, movement direction, and evolving relative distances without relying on explicit 3D reconstructions or depth sensors.

\textbf{EmbSpatial-Bench}~\cite{embspatial} is an instruction-tuning dataset for improving embodied spatial understanding in MLLMs. Constructed from Matterport3D~\cite{matterport3d}, it contains 25K samples covering spatial relation recognition and object localization as an auxiliary grounding task.

\textbf{MindCube~\citep{yin2025spatial}} is a rigorously annotated benchmark with 10k SFT data where all questions require reasoning over unobservable space. Built from multi-view image groups with controlled camera trajectories, it emphasizes cross-view consistency, occlusion handling, and viewpoint-dependent relations, serving as a challenging testbed for evaluating internal spatial mental model construction.

\textbf{SpaceR-151K}~\cite{spacer_151k} is a video spatial reasoning dataset comprising 151K QA pairs, designed to enhance spatial intelligence in MLLMs.

\subsection{Autonomous Driving Datasets}

Autonomous driving datasets provide supervision for perception, prediction, and planning in complex traffic scenarios. These datasets capture diverse driving conditions, weather variations, and interaction dynamics essential for training robust driving intelligence.

\noindent\textbf{MAPLM}~\cite{cao2024maplm} is a large-scale vision-language benchmark specifically engineered for autonomous driving, focusing on map and traffic scene understanding. It moves beyond simple object detection by requiring models to reason about lane-level topology, road geometry, and complex traffic rules. By providing a diverse set of real-world driving scenarios, MAPLM serves as a critical testbed for evaluating the potential of foundation models in enhancing the decision-making and spatial understanding capabilities of autonomous vehicles.

\noindent\textbf{DriveAction}~\cite{hao2025driveaction} is a pioneering dataset designed to explore human-like driving decisions within Vision-Language-Action~(VLA) models. It shifts the focus from pure perception to the closed-loop integration of multimodal understanding and actionable decision-making.

\noindent\textbf{Nuscenes-QA}~\cite{qian2024nuscenes} is a large-scale multimodal VQA dataset specifically designed for autonomous driving scenarios. By leveraging the comprehensive sensor data from the nuScenes dataset, it provides over 460K VQA pairs that require models to perform complex spatial and temporal understanding across 360-degree multi-view camera inputs. 

\noindent\textbf{NuPlanQA}~\cite{park2025nuplanqa} extends the scope of driving scene understanding by providing a large-scale VQA dataset based on the nuPlan dataset. It focuses on multi-view visual understanding and the interpretation of complex driving maneuvers, offering a comprehensive platform to test the limits of MLLMs in understanding real-world, planning-centric autonomous driving scenarios.

\noindent\textbf{LingoQA}~\cite{marcu2024lingoqa} introduces a specialized Visual Question Answering benchmark for autonomous driving that emphasizes the model's ability to explain driving scenes. A key contribution of this work is the introduction of \textit{Lingo-Judge}, a learned evaluation metric that aligns more closely with human judgment than traditional linguistic metrics. LingoQA challenges models to provide temporally consistent and safety-critical explanations for various driving behaviors and environmental conditions.

\subsection{Low-Altitude Datasets}

Low-altitude aerial perception presents unique challenges, including viewpoint variation, scale estimation from altitude, and interpretation of terrain features from oblique perspectives.

\textbf{HRVQA}~\cite{hrvqa} presents a large-scale dataset with over 100k qa pairs. It can be used to evaluate the capabilities of VQA models in performing scene understanding and geospatial understanding for high-resolution aerial images.

\textbf{AirSpatial-VQA}~\cite{airspatialbot} is a pioneering monocular 3D spatial perception dataset centered on vehicles captured by UAV aerial imagery. Based on photogrammetric principles and the ground plane assumption, this dataset establishes a benchmark for evaluating the 3D spatial understanding capabilities of MLLMs in drone-based scenarios. The benchmark tasks encompass vehicle 3D coordinates, 3D dimensions (length, width, height), and depth estimation. It aims to advance the monocular spatial perception of low-altitude intelligent systems, enabling zero-shot vehicle model recognition from aerial views.

\textbf{Open3DVQA}~\cite{zhang2025open3d} is a benchmark for evaluating MLLMs' ability to reason about complex spatial relationships from an aerial perspective. It contains 89k QA pairs across 7 spatial understanding tasks, including multiple-choice, true/false, and short-answer formats, and supports both visual and point cloud data. 

\textbf{AirCopBench}~\cite{zha2025aircopbench} evaluates MLLMs’ ability to answer questions using multi-UAV collaborative visual data under perception-degraded conditions, covering perception, understanding, and decision-making. It comprises over 2.9k multi-view images and 14.6k VQA pairs across four core tasks consisting of scene understanding, object Understanding, perception assessment, and collaborative decision.

\textbf{AVIMath}~\cite{zhou2025multimodal} introduces a multimodal mathematical understanding benchmark based on UAV aerial imagery. This dataset spans six mathematical disciplines-geometry, arithmetic, algebra, statistics, logic, and counting-covering 20 fine-grained topics and comprising 3,773 high-quality problems. The images are sourced from 11 distinct 4K scenes, captured with three pitch angles (45°, 60°, 90°) and three flight altitudes (low, medium, high), thereby closely simulating real-world UAV acquisition conditions.

\textbf{CapERA}~\cite{bashmal2023capera} evaluates whether MLLMs can generate natural-language descriptions for aerial videos captured by UAVs. The dataset comprises 2,864 videos and 14,320 diverse captions, each video paired with five human-aligned textual descriptions generated through manual annotation and automatic augmentation. CapERA emphasizes comprehensive scene understanding, including events, objects, actions, locations, and temporal dynamics.
\subsection{Embodied \& Egocentric Datasets}

To support embodied interaction and egocentric perception, we incorporate datasets that capture first-person visual experiences and action-oriented understanding.

\noindent\textbf{MuEP~\cite{muep}} is a comprehensive multimodal dataset specifically tailored for embodied planning with foundation models. Covering 108 varied household scenes and nearly 15,000 expert demonstration episodes, it facilitates the evaluation of multimodal and multi-turn interactions. By incorporating fine-grained metrics, MuEP assesses the agent's performance throughout task execution, effectively bridging the gap between high-level reasoning and low-level control in complex environments.

\noindent\textbf{OWMM-VLM Data~\cite{owmm}} is a synthesized dataset designed for open-world mobile manipulation. Generated using the Habitat simulator and PDDL-based task sequences, it provides comprehensive supervision for global scene understanding, robot state tracking, and multimodal action generation. This dataset provides supervision for reasoning over multi-view observations and generating action affordances grounded in both the robot’s state and the environment geometry.

\noindent\textbf{Eb-Alfred~\cite{ebalfred}} and \noindent\textbf{Eb-Habitat~\cite{eb_habitat}} serve as the core simulation datasets for long-horizon household task planning. To significantly enhance the model's ability to comprehend complex instructions, decompose tasks, and execute actions based on environmental feedback, we adopted a rigorous re-annotation pipeline. Specifically, we initialized planning tasks following standard annotations (e.g., LLaRP~\cite{eb_habitat}) to specify task goals and permissible actions, and then deployed a GPT-4o~\cite{chatgpt4o} agent to roll out the tasks in the simulator. We recorded the task instructions, action sequences, and real-time environmental observations, retaining only the trajectories that successfully accomplished the task. This LLM-driven generation provides the model with high-quality supervision containing detailed reasoning traces and successful execution paths, improving decision-making robustness in complex embodied scenarios.

\noindent\textbf{RoboVQA~\cite{robovqa}} is a large-scale multimodal dataset designed for long-horizon robotic reasoning. It comprises video-text pairs covering diverse embodiments, including humans and robots. The dataset focuses on temporal grounding tasks, such as describing past events and reasoning about future affordances, thereby enhancing the model's ability to understand the temporal dynamics of physical interactions and answer open-ended queries about the robot's environment.

\noindent\textbf{Robo2VLM~\cite{robo2vlm}} is a large-scale visual question answering dataset generated from in-the-wild robot manipulation trajectories. Leveraging data from various robotic platforms, it utilizes proprioceptive and kinematic states to automatically generate ground-truth QA pairs focused on spatial relations, object states, and interaction reasoning. This dataset allows the model to learn from diverse, scalable real-robot data without expensive manual annotation, bridging the gap between visual perception and physical robot states.

\noindent\textbf{EgoPlan~\cite{egoplanbench} }is designed for egocentric embodied planning derived from real-world videos involving human-object interactions. It focuses on predicting the next feasible action based on the task progress, current visual observation, and language instruction. We also include \textbf{EgoPlan-Mc}, a multiple-choice variant, to further refine the model's decision-making capabilities by aligning high-level task planning with intricate, real-world situations captured from a first-person perspective.

\noindent\textbf{EgoCOT~\cite{egocot}} serves as a large-scale embodied planning dataset featuring chain-of-thought (CoT) supervision. Constructed from carefully selected egocentric videos, it includes high-quality, step-by-step language instructions that are machine-generated and human-verified. This dataset is crucial for training the model to decompose complex tasks into logical intermediate reasoning steps, effectively enabling the model to solve manageable sub-tasks step by step before executing the final action.

\section{Evaluation Details}
\label{sec:bench}

We evaluate \OurMethod{} on 24 benchmarks, using both the off-the-shelf LMMs-Eval~\cite{lmms_eval} framework and the official evaluation code provided by each benchmark.

\subsection{Evaluation with LMMs-Eval}
We first introduce the evaluation with LMMs-Eval. A subset of our adopted benchmarks is already supported by LMMs-Eval, and we follow their default evaluation settings. For these built-in benchmarks, we set the image resolution constraints to $max pixels$ = $1024\times28\times28$ and $min pixels$ = $256\times28\times28$.

For benchmarks that are not natively available in LMMs-Eval, we manually integrate them into the same evaluation pipeline for consistent reporting. For these manually integrated benchmarks, we resize every evaluation image to a fixed resolution of 50,176 pixels, while keeping the number of evaluation images consistent with each benchmark’s original protocol. These benchmarks already supported in LMMs-Eval include VSI~\cite{vsi}, MindCube-Tiny~\cite{yin2025spatial}, Blink~\cite{blink}, SITE~\cite{site}, MME-RealWorld~\cite{zhang2024mme}, ERQA~\cite{gemini_robotics}, and EmbSpatial-Bench~\cite{embspatial}.

\textbf{NuScenes-QA}~\cite{nuscenesqa} contains 83335 VQA, and each question uses six camera images with text. The model should output a short answer (a single word or an Arabic numeral). It evaluates multi-view perception and attribute/object/state recognition in road-traffic scenes. The evaluation metric is exact-match accuracy.

\textbf{NuPlanQA}~\cite{park2025nuplanqa} contains 1,801 multiple-choice (A-E) questions. It evaluates decision-making and traffic-element understanding in autonomous driving. The metric is accuracy.

\textbf{HRVQA}~\cite{hrvqa} contains 13,524 questions. Each sample uses one UAV-view image, and the model generates a free-form answer. It evaluates understanding of high-resolution visual details. The metric is exact-match and accuracy.

\subsection{Evaluation with Official Code}
Next, we delineate the benchmarks evaluated with their official code.
\textbf{SAT}~\cite{sat} contains 150 multiple-choice questions. Each question uses 1-2 images with text, and the model outputs an option letter. The metric is exact-match accuracy.

\textbf{Multi3DRef}~\cite{multi3drefer} contains 11,120 samples. Each sample provides an indoor 3D scene with a set of objects and a natural-language referring expression; the model outputs the matching object instance IDs (one or multiple, e.g., \texttt{<OBJxxx>}). It evaluates multi-target referring understanding and instance-level retrieval in 3D scenes. Metric: F1@0.25, computed by matching predicted vs. ground-truth instance sets using 3D box IoU $\ge 0.25$ as the threshold, then averaging F1 over all samples.

\textbf{MAPLM}~\cite{maplm} contains 6,000 multiple-choice (A-F) questions. Each question provides a set of map/traffic-related observation images, and the model outputs a single option letter. It evaluates multiple-choice reasoning and decision-making in map and driving contexts. Metrics: Accuracy, reported as the average score of question-level accuracy (QNS) and frame-level accuracy (FRM), where a frame is considered correct only if all questions from the same scene/frame are answered correctly.

\textbf{DriveAction}~\cite{driveaction} contains 16,185 questions. Each question uses three images from the same driving scene, and the model outputs either True/False or an option letter (depending on the task). It evaluates action/behavior decision understanding in autonomous driving from multi-frame observations. The evaluation metric adopts exact-match accuracy.

\textbf{LingoQA}~\cite{lingoqa} contains 500 questions. Inputs are consecutive frames from a driving video segment, and the model generates an open-ended answer. It evaluates language understanding in driving scenes and its fusion with multi-frame visual evidence. Evaluation uses the LingoJudge discriminator to judge correctness against the reference, and reports the fraction judged correct as accuracy.

\textbf{UrbanVideo-Bench}~\cite{zhao2025urbanvideo} contains 5,355 five-choice (A-E) questions on urban egocentric videos. Each video is uniformly downsampled to up to 32 frames, and the model outputs an option letter (optionally with a brief rationale). It evaluates temporal understanding and embodied navigation reasoning in complex city scenes, covering action prediction, landmark localization, progress estimation, trajectory description, target detection, high-level planning, cognitive mapping, and counterfactual reasoning. The evaluation metric adopts accuracy.

\textbf{AirCopBench}~\cite{aircopbench} consists of 1,031 four-choice (A-D) questions across four subsets (Real2, Sim3, Sim5, Sim6). Inputs are multi-UAV observations, and the model outputs an option letter. It evaluates multi-view scene and target understanding, perception quality assessment, and cooperative decision-making. The evaluation metric adopts accuracy.

\textbf{AVI-Math}~\cite{zhou2025multimodal} follows the official two-stage protocol: (i) free-form generation that prioritizes reasoning, and (ii) extraction into a prescribed format to reduce errors from formatting. Answers are typed via the \texttt{eva} field and normalized with type-specific rules (e.g., lowercasing strings; stripping units; rounding integers; truncating floats to one decimal).

\textbf{AirSpatial-VQA}~\cite{airspatialbot} follows the official protocol and uses Mean Absolute Error (MAE) for five numerical spatial questions (Depth, Distance, Length, Width, Height). Lower MAE means closer estimates to the label.

\textbf{RoboVQA}~\cite{robovqa} includes 1,335 trajectories, split into 1,921 open-ended questions. For each trajectory, 32 frames are uniformly sampled as input, and the model generates free-form answers. It evaluates embodied visual understanding and task reasoning, including next-step prediction/planning, affordance judgments, and state/event description. The evaluation metric adopts the average score of BLEU-1/2/3/4.

\textbf{OpenEQA}~\cite{majumdar2024openeqa} includes 1,636 open-ended questions with 32 uniformly sampled frames per question. It evaluates indoor scene understanding under multi-frame aggregation (object recognition, attribute/state recognition, and spatial relations). Scoring uses a GPT-based 1-5 match rating averaged over all questions, also mapped to a 0-100 scale.

\textbf{EgoPlan-Bench2}~\cite{egoplanbench} contains 1,321 questions. We uniformly sample 16 frames from the task start to the current observation as input. It evaluates temporal task understanding for embodied interaction, next-step action planning/decision-making, and goal-directed reasoning. The evaluation metric adopts accuracy.

\textbf{EmbodiedBench(EB)-Habitat}~\cite{embodiedbench} is a simulation benchmark built on the Habitat~\cite{Habitat2} simulator, including various vision-language conditioned decision-making tasks with 282 diverse language instructions. In these tasks, the agent needs to understand the language-described goals and use commonsense and spatial reasoning to plan with the provided skills, such as navigation, pick, and open. The evaluation metric adopts the final success rate of the task.

\section{Conclusions and Perspectives}
\label{sec:con}
In this report, we introduce \OurMethod{}, a generalist foundation model that unifies spatial cognition, autonomous driving, low-altitude sensing, and embodied interaction through the Scaffold-Specialize-Reconcile (SSR) paradigm. The SSR paradigm provides a blueprint for cross-embodiment learning: construct a shared spatial scaffold, cultivate domain experts in isolation, and reconcile them via parameter merging. This enables incremental capability expansion with less gradient interference and catastrophic forgetting. Evaluated on 24 benchmarks spanning four physical domains, \OurMethod{} achieves competitive or even state-of-the-art performance, significantly outperforming both general-purpose VLMs and domain-specific embodied brains.

Looking forward, we will advance along three axes: (1) Spatially-grounded visuomotor policies, extending \OurMethod{} to vision-language-action models for closed-loop control across robot embodiments; (2) Physics-aware continuous prediction, iterating beyond discrete scene understanding toward fine-grained physical world modeling; and (3) Cross-Embodiment Continual Learning, advancing the SSR paradigm toward lifelong, interference-free capability accumulation, enabling seamless integration of novel embodiments (\textit{e.g.}, legged locomotion, underwater vehicles).

\clearpage
\bibliographystyle{unsrt}
\bibliography{references}

@article{weiUnifyingMultimodalLarge2025,
  title   = {Unifying Multimodal Large Language Model Capabilities and Modalities via Model Merging},
  author  = {Wei, Yongxian and Cheng, Runxi and Jin, Weike and Yang, Enneng and Shen, Li and Hou, Lu and Du, Sinan and Yuan, Chun and Cao, Xiaochun and Tao, Dacheng},
  journal = {arXiv preprint arXiv:2505.19892},
  year    = {2025}
}

@inproceedings{chengWhoeverStartedInterference2025,
  title      = {Whoever {{Started}} the Interference {{Should End It}}: {{Guiding Data-Free Model Merging}} via {{Task Vectors}}},
  shorttitle = {Whoever {{Started}} the Interference {{Should End It}}},
  booktitle  = {Proceedings of the 42nd {{International Conference}} on {{Machine Learning}}},
  author     = {Cheng, Runxi and Xiong, Feng and Wei, Yongxian and Zhu, Wanyun and Yuan, Chun},
  date       = {2025-10-06},
  pages      = {10121--10143},
  publisher  = {PMLR},
  issn       = {2640-3498},
  url        = {https://proceedings.mlr.press/v267/cheng25h.html},
  eventtitle = {International {{Conference}} on {{Machine Learning}}},
  langid     = {english}
}

@inproceedings{yin2025spatial,
  title     = {Spatial mental modeling from limited views},
  author    = {Yin, Baiqiao and Wang, Qineng and Zhang, Pingyue and Zhang, Jianshu and Wang, Kangrui and Wang, Zihan and Zhang, Jieyu and Chandrasegaran, Keshigeyan and Liu, Han and Krishna, Ranjay and others},
  booktitle = {Structural Priors for Vision Workshop at ICCV'25},
  year      = {2025}
}

@article{tangFusionBenchUnifiedLibrary2025,
  title   = {FusionBench: A Unified Library and Comprehensive Benchmark for Deep Model Fusion},
  author  = {Anke Tang and Li Shen and Yong Luo and Enneng Yang and Han Hu and Lefei Zhang and Bo Du and Dacheng Tao},
  journal = {arXiv preprint arXiv:2406.03280},
  year    = {2025}
}

@inproceedings{gargiuloTaskSingularVectors2025,
  title     = {Task singular vectors: Reducing task interference in model merging},
  author    = {Gargiulo, Antonio Andrea and Crisostomi, Donato and Bucarelli, Maria Sofia and Scardapane, Simone and Silvestri, Fabrizio and Rodola, Emanuele},
  booktitle = {Proceedings of the Computer Vision and Pattern Recognition Conference},
  pages     = {18695--18705},
  year      = {2025}
}

@inproceedings{wortsmanModelSoupsAveraging2022a,
  title        = {Model soups: averaging weights of multiple fine-tuned models improves accuracy without increasing inference time},
  author       = {Wortsman, Mitchell and Ilharco, Gabriel and Gadre, Samir Ya and Roelofs, Rebecca and Gontijo-Lopes, Raphael and Morcos, Ari S and Namkoong, Hongseok and Farhadi, Ali and Carmon, Yair and Kornblith, Simon and others},
  booktitle    = {International conference on machine learning},
  pages        = {23965--23998},
  year         = {2022},
  organization = {PMLR}
}

@inproceedings{cheginiModelSoupBetter2024,
  title     = {Model soup for better rlhf: Weight space averaging to improve alignment in llms},
  author    = {Chegini, Atoosa and Kazemi, Hamid and Mirzadeh, Seyed Iman and Yin, Dong and Horton, Maxwell and Nabi, Moin and Farajtabar, Mehrdad and Alizadeh, Keivan},
  booktitle = {NeurIPS 2024 Workshop on Fine-Tuning in Modern Machine Learning: Principles and Scalability},
  year      = {2024}
}

@inproceedings{majumdar2024openeqa,
  title     = {Openeqa: Embodied question answering in the era of foundation models},
  author    = {Majumdar, Arjun and Ajay, Anurag and Zhang, Xiaohan and Putta, Pranav and Yenamandra, Sriram and Henaff, Mikael and Silwal, Sneha and Mcvay, Paul and Maksymets, Oleksandr and Arnaud, Sergio and others},
  booktitle = {Proceedings of the IEEE/CVF conference on computer vision and pattern recognition},
  pages     = {16488--16498},
  year      = {2024}
}

@article{zhang2024mme,
  title   = {Mme-realworld: Could your multimodal llm challenge high-resolution real-world scenarios that are difficult for humans?},
  author  = {Zhang, Yi-Fan and Zhang, Huanyu and Tian, Haochen and Fu, Chaoyou and Zhang, Shuangqing and Wu, Junfei and Li, Feng and Wang, Kun and Wen, Qingsong and Zhang, Zhang and others},
  journal = {arXiv preprint arXiv:2408.13257},
  year    = {2024}
}

@inproceedings{cao2024maplm,
  title     = {Maplm: A real-world large-scale vision-language benchmark for map and traffic scene understanding},
  author    = {Cao, Xu and Zhou, Tong and Ma, Yunsheng and Ye, Wenqian and Cui, Can and Tang, Kun and Cao, Zhipeng and Liang, Kaizhao and Wang, Ziran and Rehg, James M and others},
  booktitle = {Proceedings of the IEEE/CVF conference on computer vision and pattern recognition},
  pages     = {21819--21830},
  year      = {2024}
}

@inproceedings{yangadamerging,
  title={AdaMerging: Adaptive Model Merging for Multi-Task Learning},
  author={Yang, Enneng and Wang, Zhenyi and Shen, Li and Liu, Shiwei and Guo, Guibing and Wang, Xingwei and Tao, Dacheng},
  booktitle={The Twelfth International Conference on Learning Representations},
  year={2024}
}

@article{hao2025driveaction,
  title   = {Driveaction: A benchmark for exploring human-like driving decisions in vla models},
  author  = {Hao, Yuhan and Li, Zhengning and Sun, Lei and Wang, Weilong and Yi, Naixin and Song, Sheng and Qin, Caihong and Zhou, Mofan and Zhan, Yifei and Lang, Xianpeng},
  journal = {arXiv preprint arXiv:2506.05667},
  year    = {2025}
}

@inproceedings{qian2024nuscenes,
  title     = {Nuscenes-qa: A multi-modal visual question answering benchmark for autonomous driving scenario},
  author    = {Qian, Tianwen and Chen, Jingjing and Zhuo, Linhai and Jiao, Yang and Jiang, Yu-Gang},
  booktitle = {Proceedings of the AAAI Conference on Artificial Intelligence},
  volume    = {38},
  number    = {5},
  pages     = {4542--4550},
  year      = {2024}
}

@article{park2025nuplanqa,
  title   = {Nuplanqa: A large-scale dataset and benchmark for multi-view driving scene understanding in multi-modal large language models},
  author  = {Park, Sung-Yeon and Cui, Can and Ma, Yunsheng and Moradipari, Ahmadreza and Gupta, Rohit and Han, Kyungtae and Wang, Ziran},
  journal = {arXiv preprint arXiv:2503.12772},
  year    = {2025}
}

@inproceedings{marcu2024lingoqa,
  title        = {Lingoqa: Visual question answering for autonomous driving},
  author       = {Marcu, Ana-Maria and Chen, Long and H{\"u}nermann, Jan and Karnsund, Alice and Hanotte, Benoit and Chidananda, Prajwal and Nair, Saurabh and Badrinarayanan, Vijay and Kendall, Alex and Shotton, Jamie and others},
  booktitle    = {European Conference on Computer Vision},
  pages        = {252--269},
  year         = {2024},
  organization = {Springer}
}

@misc{shao2024deepseekmath,
  title         = {DeepSeekMath: Pushing the Limits of Mathematical Reasoning in Open Language Models},
  author        = {Zhihong Shao and Peiyi Wang and Qihao Zhu and Runxin Xu and Junxiao Song and Xiao Bi and Haowei Zhang and Mingchuan Zhang and Y. K. Li and Y. Wu and Daya Guo},
  year          = {2024},
  eprint        = {2402.03300},
  archiveprefix = {arXiv},
  primaryclass  = {cs.CL},
  url           = {https://arxiv.org/abs/2402.03300}
}

@article{bai2025qwen3vltechnicalreport,
  title   = {Qwen3-VL Technical Report},
  author  = {Shuai Bai and Yuxuan Cai and Ruizhe Chen and Keqin Chen and Xionghui Chen and Zesen Cheng and Lianghao Deng and Wei Ding and Chang Gao and Chunjiang Ge and Wenbin Ge and Zhifang Guo and Qidong Huang and Jie Huang and Fei Huang and Binyuan Hui and Shutong Jiang and Zhaohai Li and Mingsheng Li and Mei Li and Kaixin Li and Zicheng Lin and Junyang Lin and Xuejing Liu and Jiawei Liu and Chenglong Liu and Yang Liu and Dayiheng Liu and Shixuan Liu and Dunjie Lu and Ruilin Luo and Chenxu Lv and Rui Men and Lingchen Meng and Xuancheng Ren and Xingzhang Ren and Sibo Song and Yuchong Sun and Jun Tang and Jianhong Tu and Jianqiang Wan and Peng Wang and Pengfei Wang and Qiuyue Wang and Yuxuan Wang and Tianbao Xie and Yiheng Xu and Haiyang Xu and Jin Xu and Zhibo Yang and Mingkun Yang and Jianxin Yang and An Yang and Bowen Yu and Fei Zhang and Hang Zhang and Xi Zhang and Bo Zheng and Humen Zhong and Jingren Zhou and Fan Zhou and Jing Zhou and Yuanzhi Zhu and Ke Zhu},
  year    = {2025},
  journal = {arXiv preprint arXiv:2511.21631}
}

@article{vebrain,
  title   = {Visual embodied brain: Let multimodal large language models see, think, and control in spaces},
  author  = {Luo, Gen and Yang, Ganlin and Gong, Ziyang and Chen, Guanzhou and Duan, Haonan and Cui, Erfei and Tong, Ronglei and Hou, Zhi and Zhang, Tianyi and Chen, Zhe and others},
  journal = {arXiv preprint arXiv:2506.00123},
  year    = {2025}
}

@article{vlaser,
  title   = {Vlaser: Vision-Language-Action Model with Synergistic Embodied Reasoning},
  author  = {Yang, Ganlin and Zhang, Tianyi and Hao, Haoran and Wang, Weiyun and Liu, Yibin and Wang, Dehui and Chen, Guanzhou and Cai, Zijian and Chen, Junting and Su, Weijie and others},
  journal = {arXiv preprint arXiv:2510.11027},
  year    = {2025}
}

@inproceedings{robobrain,
  title     = {Robobrain: A unified brain model for robotic manipulation from abstract to concrete},
  author    = {Ji, Yuheng and Tan, Huajie and Shi, Jiayu and Hao, Xiaoshuai and Zhang, Yuan and Zhang, Hengyuan and Wang, Pengwei and Zhao, Mengdi and Mu, Yao and An, Pengju and others},
  booktitle = {Proceedings of the Computer Vision and Pattern Recognition Conference},
  pages     = {1724--1734},
  year      = {2025}
}

@article{robobrain2,
  title   = {Robobrain 2.0 technical report},
  author  = {Team, BAAI RoboBrain and Cao, Mingyu and Tan, Huajie and Ji, Yuheng and Chen, Xiansheng and Lin, Minglan and Li, Zhiyu and Cao, Zhou and Wang, Pengwei and Zhou, Enshen and others},
  journal = {arXiv preprint arXiv:2507.02029},
  year    = {2025}
}

@article{robobrain2.5,
  title   = {RoboBrain 2.5: Depth in Sight, Time in Mind},
  author  = {Tan, Huajie and Zhou, Enshen and Li, Zhiyu and Xu, Yijie and Ji, Yuheng and Chen, Xiansheng and Chi, Cheng and Wang, Pengwei and Jia, Huizhu and Ao, Yulong and others},
  journal = {arXiv preprint arXiv:2601.14352},
  year    = {2026}
}

@article{pelican-vl,
  title   = {Pelican-VL 1.0: A Foundation Brain Model for Embodied Intelligence},
  author  = {Zhang, Yi and Liu, Che and Ren, Xiancong and Ni, Hanchu and Zhang, Shuai and Ding, Zeyuan and Hu, Jiayu and Shan, Hanzhe and Niu, Zhenwei and Liu, Zhaoyang and others},
  journal = {arXiv preprint arXiv:2511.00108},
  year    = {2025}
}

@article{mimo-embodied,
  title   = {MiMo-Embodied: X-Embodied Foundation Model Technical Report},
  author  = {Hao, Xiaoshuai and Zhou, Lei and Huang, Zhijian and Hou, Zhiwen and Tang, Yingbo and Zhang, Lingfeng and Li, Guang and Lu, Zheng and Ren, Shuhuai and Meng, Xianhui and others},
  journal = {arXiv preprint arXiv:2511.16518},
  year    = {2025}
}

@article{cambrian-s,
  title   = {Cambrian-s: Towards spatial supersensing in video},
  author  = {Yang, Shusheng and Yang, Jihan and Huang, Pinzhi and Brown, Ellis and Yang, Zihao and Yu, Yue and Tong, Shengbang and Zheng, Zihan and Xu, Yifan and Wang, Muhan and others},
  journal = {arXiv preprint arXiv:2511.04670},
  year    = {2025}
}

@article{cambrian,
  title   = {Cambrian-1: A fully open, vision-centric exploration of multimodal llms},
  author  = {Tong, Peter and Brown, Ellis and Wu, Penghao and Woo, Sanghyun and IYER, Adithya Jairam Vedagiri and Akula, Sai Charitha and Yang, Shusheng and Yang, Jihan and Middepogu, Manoj and Wang, Ziteng and others},
  journal = {Advances in Neural Information Processing Systems},
  volume  = {37},
  pages   = {87310--87356},
  year    = {2024}
}

@article{uav-vl-r1,
  title   = {Uav-vl-r1: Generalizing vision-language models via supervised fine-tuning and multi-stage grpo for uav visual reasoning},
  author  = {Guan, Jiajin and Mei, Haibo and Zhang, Bonan and Liu, Dan and Fu, Yuanshuang and Zhang, Yue},
  journal = {arXiv preprint arXiv:2508.11196},
  year    = {2025}
}

@article{airspatialbot,
  title     = {AirSpatialBot: A Spatially-Aware Aerial Agent for Fine-Grained Vehicle Attribute Recognization and Retrieval},
  author    = {Zhou, Yue and Ding, Ran and Yang, Xue and Jiang, Xue and Liu, Xingzhao},
  journal   = {IEEE Transactions on Geoscience and Remote Sensing},
  year      = {2025},
  publisher = {IEEE}
}

@article{hrvqa,
  title     = {HRVQA: A Visual Question Answering benchmark for high-resolution aerial images},
  author    = {Li, Kun and Vosselman, George and Yang, Michael Ying},
  journal   = {ISPRS Journal of Photogrammetry and Remote Sensing},
  volume    = {214},
  pages     = {65--81},
  year      = {2024},
  publisher = {Elsevier}
}

@inproceedings{zhang2025open3d,
  title     = {Open3D-VQA: A Benchmark for Embodied Spatial Concept Reasoning with Multimodal Large Language Model in Open Space},
  author    = {Zhang, Weichen and Zhou, Zile and Zeng, Xin and Xuchen, Liu and Fang, Jianjie and Gao, Chen and Cui, Jinqiang and Li, Yong and Chen, Xinlei and Zhang, Xiao-Ping},
  booktitle = {Proceedings of the 33rd ACM International Conference on Multimedia},
  pages     = {12784--12791},
  year      = {2025}
}

@article{zha2025aircopbench,
  title   = {AirCopBench: A Benchmark for Multi-drone Collaborative Embodied Perception and Reasoning},
  author  = {Zha, Jirong and Fan, Yuxuan and Zhang, Tianyu and Chen, Geng and Chen, Yingfeng and Gao, Chen and Chen, Xinlei},
  journal = {arXiv preprint arXiv:2511.11025},
  year    = {2025}
}

@article{zhao2025urbanvideo,
  title   = {Urbanvideo-bench: Benchmarking vision-language models on embodied intelligence with video data in urban spaces},
  author  = {Zhao, Baining and Fang, Jianjie and Dai, Zichao and Wang, Ziyou and Zha, Jirong and Zhang, Weichen and Gao, Chen and Wang, Yue and Cui, Jinqiang and Chen, Xinlei and others},
  journal = {arXiv preprint arXiv:2503.06157},
  year    = {2025}
}

@article{zhou2025multimodal,
  title     = {Multimodal mathematical reasoning embedded in aerial vehicle imagery: Benchmarking, analysis, and exploration},
  author    = {Zhou, Yue and Feng, Litong and Lan, Mengcheng and Yang, Xue and Li, Qingyun and Ke, Yiping and Jiang, Xue and Zhang, Wayne},
  journal   = {ISPRS Journal of Photogrammetry and Remote Sensing},
  volume    = {230},
  pages     = {289--303},
  year      = {2025},
  publisher = {Elsevier}
}

@article{bashmal2023capera,
  title     = {Capera: Captioning events in aerial videos},
  author    = {Bashmal, Laila and Bazi, Yakoub and Al Rahhal, Mohamad Mahmoud and Zuair, Mansour and Melgani, Farid},
  journal   = {Remote Sensing},
  volume    = {15},
  number    = {8},
  pages     = {2139},
  year      = {2023},
  publisher = {MDPI}
}

@inproceedings{cao2021visdrone,
  title     = {VisDrone-DET2021: The vision meets drone object detection challenge results},
  author    = {Cao, Yaru and He, Zhijian and Wang, Lujia and Wang, Wenguan and Yuan, Yixuan and Zhang, Dingwen and Zhang, Jinglin and Zhu, Pengfei and Van Gool, Luc and Han, Junwei and others},
  booktitle = {Proceedings of the IEEE/CVF International conference on computer vision},
  pages     = {2847--2854},
  year      = {2021}
}

@article{zhou2021egoplanner,
  author  = {Zhou, Xin and Wang, Zhepei and Ye, Hongkai and Xu, Chao and Gao, Fei},
  journal = {IEEE Robotics and Automation Letters},
  title   = {EGO-Planner: An ESDF-Free Gradient-Based Local Planner for Quadrotors},
  year    = {2021},
  volume  = {6},
  number  = {2},
  pages   = {478-485}
}

@article{cao2025proximal,
  title     = {Proximal cooperative aerial manipulation with vertically stacked drones},
  author    = {Cao, Huazi and Shen, Jiahao and Zhang, Yin and Fu, Zheng and Liu, Cunjia and Sun, Sihao and Zhao, Shiyu},
  journal   = {Nature},
  volume    = {646},
  number    = {8085},
  pages     = {576--583},
  year      = {2025},
  publisher = {Nature Publishing Group UK London}
}

@inproceedings{robovqa,
  title        = {Robovqa: Multimodal long-horizon reasoning for robotics},
  author       = {Sermanet, Pierre and Ding, Tianli and Zhao, Jeffrey and Xia, Fei and Dwibedi, Debidatta and Gopalakrishnan, Keerthana and Chan, Christine and Dulac-Arnold, Gabriel and Maddineni, Sharath and Joshi, Nikhil J and others},
  booktitle    = {2024 IEEE International Conference on Robotics and Automation (ICRA)},
  pages        = {645--652},
  year         = {2024},
  organization = {IEEE}
}

@article{robo2vlm,
  title   = {Robo2vlm: Visual question answering from large-scale in-the-wild robot manipulation datasets},
  author  = {Chen, Kaiyuan and Xie, Shuangyu and Ma, Zehan and Sanketi, Pannag R and Goldberg, Ken},
  journal = {arXiv preprint arXiv:2505.15517},
  year    = {2025}
}

@inproceedings{muep,
  title     = {MuEP: A Multimodal Benchmark for Embodied Planning with Foundation Models},
  author    = {Li, Kanxue and Yu, Baosheng and Zheng, Qi and Zhan, Yibing and Zhang, Yuhui and Zhang, Tianle and Yang, Yijun and Chen, Yue and Sun, Lei and Cao, Qiong and Shen, Li and Li, Lusong and Tao, Dapeng and He, Xiaodong},
  booktitle = {Proceedings of the Thirty-Third International Joint Conference on
               Artificial Intelligence, {IJCAI-24}},
  publisher = {International Joint Conferences on Artificial Intelligence Organization},
  editor    = {Kate Larson},
  pages     = {129--138},
  year      = {2024},
  month     = {8},
  note      = {Main Track},
  doi       = {10.24963/ijcai.2024/15},
  url       = {https://doi.org/10.24963/ijcai.2024/15}
}

@article{owmm,
  title   = {OWMM-Agent: Open World Mobile Manipulation With Multi-modal Agentic Data Synthesis},
  author  = {Chen, Junting and Liang, Haotian and Du, Lingxiao and Wang, Weiyun and Hu, Mengkang and Mu, Yao and Wang, Wenhai and Dai, Jifeng and Luo, Ping and Shao, Wenqi and others},
  journal = {arXiv preprint arXiv:2506.04217},
  year    = {2025}
}

@misc{ebalfred,
  title         = {World-aware Planning Narratives Enhance Large Vision-Language Model Planner},
  author        = {Junhao Shi and Zhaoye Fei and Siyin Wang and Qipeng Guo and Jingjing Gong and Xipeng Qiu},
  year          = {2025},
  eprint        = {2506.21230},
  archiveprefix = {arXiv},
  primaryclass  = {cs.AI},
  url           = {https://arxiv.org/abs/2506.21230}
}

@inproceedings{eb_habitat,
  title     = {Large Language Models as Generalizable Policies for Embodied Tasks},
  author    = {Andrew Szot and Max Schwarzer and Harsh Agrawal and Bogdan Mazoure and Rin Metcalf and Walter Talbott and Natalie Mackraz and R Devon Hjelm and Alexander T Toshev},
  booktitle = {The Twelfth International Conference on Learning Representations},
  year      = {2024},
  url       = {https://openreview.net/forum?id=u6imHU4Ebu}
}

@article{egoplanbench,
  title   = {Egoplan-bench: Benchmarking egocentric embodied planning with multimodal large language models},
  author  = {Chen, Yi and Ge, Yuying and Ge, Yixiao and Ding, Mingyu and Li, Bohao and Wang, Rui and Xu, Ruifeng and Shan, Ying and Liu, Xihui},
  journal = {CoRR},
  year    = {2023}
}

@inproceedings{egocot,
  title     = {Embodied{GPT}: Vision-Language Pre-Training via Embodied Chain of Thought},
  author    = {Yao Mu and Qinglong Zhang and Mengkang Hu and Wenhai Wang and Mingyu Ding and Jun Jin and Bin Wang and Jifeng Dai and Yu Qiao and Ping Luo},
  booktitle = {Thirty-seventh Conference on Neural Information Processing Systems},
  year      = {2023},
  url       = {https://openreview.net/forum?id=IL5zJqfxAa}
}

@inproceedings{robotron-drive,
  title     = {RoboTron-Drive: All-in-One Large Multimodal Model for Autonomous Driving},
  author    = {Huang, Zhijian and Feng, Chengjian and Yan, Feng and Xiao, Baihui and Jie, Zequn and Zhong, Yujie and Liang, Xiaodan and Ma, Lin},
  booktitle = {Proceedings of the IEEE/CVF International Conference on Computer Vision},
  pages     = {8011--8021},
  year      = {2025}
}

@inproceedings{robotron-manip,
  title     = {RoboTron-Mani: All-in-One Multimodal Large Model for Robotic Manipulation},
  author    = {Yan, Feng and Liu, Fanfan and Huang, Yiyang and Guan, Zechao and Zheng, Liming and Zhong, Yufeng and Feng, Chengjian and Ma, Lin},
  booktitle = {Proceedings of the IEEE/CVF International Conference on Computer Vision},
  pages     = {13707--13718},
  year      = {2025}
}

@inproceedings{drivelm,
  title        = {Drivelm: Driving with graph visual question answering},
  author       = {Sima, Chonghao and Renz, Katrin and Chitta, Kashyap and Chen, Li and Zhang, Hanxue and Xie, Chengen and Bei{\ss}wenger, Jens and Luo, Ping and Geiger, Andreas and Li, Hongyang},
  booktitle    = {European conference on computer vision},
  pages        = {256--274},
  year         = {2024},
  organization = {Springer}
}

@article{drivemm,
  title   = {Drivemm: All-in-one large multimodal model for autonomous driving},
  author  = {Huang, Zhijian and Feng, Chengjian and Yan, Feng and Xiao, Baihui and Jie, Zequn and Zhong, Yujie and Liang, Xiaodan and Ma, Lin},
  journal = {arXiv preprint arXiv:2412.07689},
  year    = {2024}
}

@article{gemini_robotics,
  title   = {Gemini robotics: Bringing ai into the physical world},
  author  = {Team, Gemini Robotics and Abeyruwan, Saminda and Ainslie, Joshua and Alayrac, Jean-Baptiste and Arenas, Montserrat Gonzalez and Armstrong, Travis and Balakrishna, Ashwin and Baruch, Robert and Bauza, Maria and Blokzijl, Michiel and others},
  journal = {arXiv preprint arXiv:2503.20020},
  year    = {2025}
}

@article{drivegpt4,
  title     = {Drivegpt4: Interpretable end-to-end autonomous driving via large language model},
  author    = {Xu, Zhenhua and Zhang, Yujia and Xie, Enze and Zhao, Zhen and Guo, Yong and Wong, Kwan-Yee K and Li, Zhenguo and Zhao, Hengshuang},
  journal   = {IEEE Robotics and Automation Letters},
  year      = {2024},
  publisher = {IEEE}
}

@article{drivepi,
  title={DrivePI: Spatial-aware 4D MLLM for Unified Autonomous Driving Understanding, Perception, Prediction and Planning},
  author={Liu, Zhe and Huang, Runhui and Yang, Rui and Yan, Siming and Wang, Zining and Hou, Lu and Lin, Di and Bai, Xiang and Zhao, Hengshuang},
  journal={arXiv preprint arXiv:2512.12799},
  year={2025}
}

@article{drivelmm-o1,
  title   = {Drivelmm-o1: A step-by-step reasoning dataset and large multimodal model for driving scenario understanding},
  author  = {Ishaq, Ayesha and Lahoud, Jean and More, Ketan and Thawakar, Omkar and Thawkar, Ritesh and Dissanayake, Dinura and Ahsan, Noor and Li, Yuhao and Khan, Fahad Shahbaz and Cholakkal, Hisham and others},
  journal = {arXiv preprint arXiv:2503.10621},
  year    = {2025}
}

@article{sensenova,
  title   = {Scaling spatial intelligence with multimodal foundation models},
  author  = {Cai, Zhongang and Wang, Ruisi and Gu, Chenyang and Pu, Fanyi and Xu, Junxiang and Wang, Yubo and Yin, Wanqi and Yang, Zhitao and Wei, Chen and Sun, Qingping and others},
  journal = {arXiv preprint arXiv:2511.13719},
  year    = {2025}
}

@article{sensenova_1.5,
  title   = {Holistic Evaluation of Multimodal LLMs on Spatial Intelligence},
  author  = {Cai, Zhongang and Wang, Yubo and Sun, Qingping and Wang, Ruisi and Gu, Chenyang and Yin, Wanqi and Lin, Zhiqian and Yang, Zhitao and Wei, Chen and Qian, Oscar and others},
  journal = {arXiv preprint arXiv:2508.13142},
  year    = {2025}
}

@article{internspatial,
  title   = {InternSpatial: A Comprehensive Dataset for Spatial Reasoning in Vision-Language Models},
  author  = {Deng, Nianchen and Gu, Lixin and Ye, Shenglong and He, Yinan and Chen, Zhe and Li, Songze and Wang, Haomin and Wei, Xingguang and Yang, Tianshuo and Dou, Min and others},
  journal = {arXiv preprint arXiv:2506.18385},
  year    = {2025}
}

@article{spacevista,
  title   = {Spacevista: All-scale visual spatial reasoning from mm to km},
  author  = {Sun, Peiwen and Lang, Shiqiang and Wu, Dongming and Ding, Yi and Feng, Kaituo and Liu, Huadai and Ye, Zhen and Liu, Rui and Liu, Yun-Hui and Wang, Jianan and others},
  journal = {arXiv preprint arXiv:2510.09606},
  year    = {2025}
}

@article{spatialtuning,
  title   = {Visual spatial tuning},
  author  = {Yang, Rui and Zhu, Ziyu and Li, Yanwei and Huang, Jingjia and Yan, Shen and Zhou, Siyuan and Liu, Zhe and Li, Xiangtai and Li, Shuangye and Wang, Wenqian and others},
  journal = {arXiv preprint arXiv:2511.05491},
  year    = {2025}
}

@inproceedings{rovi,
  title     = {Robotic visual instruction},
  author    = {Li, Yanbang and Gong, Ziyang and Li, Haoyang and Huang, Xiaoqi and Kang, Haolan and Bai, Guangping and Ma, Xianzheng},
  booktitle = {Proceedings of the Computer Vision and Pattern Recognition Conference},
  pages     = {12155--12165},
  year      = {2025}
}

@article{llarva,
  title   = {Llarva: Vision-action instruction tuning enhances robot learning},
  author  = {Niu, Dantong and Sharma, Yuvan and Biamby, Giscard and Quenum, Jerome and Bai, Yutong and Shi, Baifeng and Darrell, Trevor and Herzig, Roei},
  journal = {arXiv preprint arXiv:2406.11815},
  year    = {2024}
}

@article{gpt4robo,
  title     = {Gpt-4v (ision) for robotics: Multimodal task planning from human demonstration},
  author    = {Wake, Naoki and Kanehira, Atsushi and Sasabuchi, Kazuhiro and Takamatsu, Jun and Ikeuchi, Katsushi},
  journal   = {IEEE Robotics and Automation Letters},
  year      = {2024},
  publisher = {IEEE}
}

@inproceedings{robospatial,
  title     = {Robospatial: Teaching spatial understanding to 2d and 3d vision-language models for robotics},
  author    = {Song, Chan Hee and Blukis, Valts and Tremblay, Jonathan and Tyree, Stephen and Su, Yu and Birchfield, Stan},
  booktitle = {Proceedings of the Computer Vision and Pattern Recognition Conference},
  pages     = {15768--15780},
  year      = {2025}
}

@inproceedings{roboafford,
  title     = {Roboafford: A dataset and benchmark for enhancing object and spatial affordance learning in robot manipulation},
  author    = {Tang, Yingbo and Zhang, Lingfeng and Zhang, Shuyi and Zhao, Yinuo and Hao, Xiaoshuai},
  booktitle = {Proceedings of the 33rd ACM International Conference on Multimedia},
  pages     = {12706--12713},
  year      = {2025}
}

@article{embodiedr1,
  title   = {Embodied-r1: Reinforced embodied reasoning for general robotic manipulation},
  author  = {Yuan, Yifu and Cui, Haiqin and Huang, Yaoting and Chen, Yibin and Ni, Fei and Dong, Zibin and Li, Pengyi and Zheng, Yan and Hao, Jianye},
  journal = {arXiv preprint arXiv:2508.13998},
  year    = {2025}
}

@article{3dr1,
  title   = {3d-r1: Enhancing reasoning in 3d vlms for unified scene understanding},
  author  = {Huang, Ting and Zhang, Zeyu and Tang, Hao},
  journal = {arXiv preprint arXiv:2507.23478},
  year    = {2025}
}

@article{tiger,
  title   = {TIGeR: Tool-Integrated Geometric Reasoning in Vision-Language Models for Robotics},
  author  = {Han, Yi and Chi, Cheng and Zhou, Enshen and Rong, Shanyu and An, Jingkun and Wang, Pengwei and Wang, Zhongyuan and Sheng, Lu and Zhang, Shanghang},
  journal = {arXiv preprint arXiv:2510.07181},
  year    = {2025}
}

@article{multi_spatialmllm,
  title   = {Multi-spatialmllm: Multi-frame spatial understanding with multi-modal large language models},
  author  = {Xu, Runsen and Wang, Weiyao and Tang, Hao and Chen, Xingyu and Wang, Xiaodong and Chu, Fu-Jen and Lin, Dahua and Feiszli, Matt and Liang, Kevin J},
  journal = {arXiv preprint arXiv:2505.17015},
  year    = {2025}
}

@article{spatial_ssrl,
  title   = {Spatial-ssrl: Enhancing spatial understanding via self-supervised reinforcement learning},
  author  = {Liu, Yuhong and Zhang, Beichen and Zang, Yuhang and Cao, Yuhang and Xing, Long and Dong, Xiaoyi and Duan, Haodong and Lin, Dahua and Wang, Jiaqi},
  journal = {arXiv preprint arXiv:2510.27606},
  year    = {2025}
}

@article{leovl,
  title   = {LEO-VL: Towards 3D Vision-Language Generalists via Data Scaling with Efficient Representation},
  author  = {Huang, Jiangyong and Ma, Xiaojian and Linghu, Xiongkun and Fan, Yue and He, Junchao and Tan, Wenxin and Li, Qing and Zhu, Song-Chun and Chen, Yixin and Jia, Baoxiong and others},
  journal = {arXiv preprint arXiv:2506.09935},
  year    = {2025}
}

@article{mllm_uav_swarm,
  title   = {Multimodal Large Language Models-Enabled UAV Swarm: Towards Efficient and Intelligent Autonomous Aerial Systems},
  author  = {Ping, Yuqi and Liang, Tianhao and Ding, Huahao and Lei, Guangyu and Wu, Junwei and Zou, Xuan and Shi, Kuan and Shao, Rui and Zhang, Chiya and Zhang, Weizheng and others},
  journal = {arXiv preprint arXiv:2506.12710},
  year    = {2025}
}

@article{geonav,
  title   = {Geonav: Empowering mllms with explicit geospatial reasoning abilities for language-goal aerial navigation},
  author  = {Xu, Haotian and Hu, Yue and Gao, Chen and Zhu, Zhengqiu and Zhao, Yong and Li, Yong and Yin, Quanjun},
  journal = {arXiv preprint arXiv:2504.09587},
  year    = {2025}
}

@article{dvgbench,
  title     = {DVGBench: Implicit-to-explicit visual grounding benchmark in UAV imagery with large vision--language models},
  author    = {Zhou, Yue and Chen, Jue and Zhang, Zilun and Huang, Penghui and Ding, Ran and Zou, Zhentao and Gao, PengFei and Wei, Yuchen and Li, Ke and Yang, Xue and others},
  journal   = {ISPRS Journal of Photogrammetry and Remote Sensing},
  volume    = {232},
  pages     = {831--847},
  year      = {2026},
  publisher = {Elsevier}
}

@article{aircopbench,
  title   = {Aircopbench: A benchmark for multi-drone collaborative embodied perception and reasoning},
  author  = {Zha, Jirong and Fan, Yuxuan and Zhang, Tianyu and Chen, Geng and Chen, Yingfeng and Gao, Chen and Chen, Xinlei},
  journal = {arXiv preprint arXiv:2511.11025},
  year    = {2025}
}

@misc{chatgpt4o,
  title        = {GPT-4o System Card},
  author       = {OpenAI},
  howpublished = {\url{https://openai.com/index/gpt-4o-system-card/}},
  year         = {2025}
}

@article{qwen2vl,
  title   = {Qwen2-vl: Enhancing vision-language model's perception of the world at any resolution},
  author  = {Wang, Peng and Bai, Shuai and Tan, Sinan and Wang, Shijie and Fan, Zhihao and Bai, Jinze and Chen, Keqin and Liu, Xuejing and Wang, Jialin and Ge, Wenbin and others},
  journal = {arXiv preprint arXiv:2409.12191},
  year    = {2024}
}

@article{qwen2.5-vl,
  title   = {Qwen2. 5-vl technical report},
  author  = {Bai, Shuai and Chen, Keqin and Liu, Xuejing and Wang, Jialin and Ge, Wenbin and Song, Sibo and Dang, Kai and Wang, Peng and Wang, Shijie and Tang, Jun and others},
  journal = {arXiv preprint arXiv:2502.13923},
  year    = {2025}
}

@article{qwen3,
  title   = {Qwen3 technical report},
  author  = {Yang, An and Li, Anfeng and Yang, Baosong and Zhang, Beichen and Hui, Binyuan and Zheng, Bo and Yu, Bowen and Gao, Chang and Huang, Chengen and Lv, Chenxu and others},
  journal = {arXiv preprint arXiv:2505.09388},
  year    = {2025}
}

@inproceedings{internvl,
  title     = {Internvl: Scaling up vision foundation models and aligning for generic visual-linguistic tasks},
  author    = {Chen, Zhe and Wu, Jiannan and Wang, Wenhai and Su, Weijie and Chen, Guo and Xing, Sen and Zhong, Muyan and Zhang, Qinglong and Zhu, Xizhou and Lu, Lewei and others},
  booktitle = {Proceedings of the IEEE/CVF conference on computer vision and pattern recognition},
  pages     = {24185--24198},
  year      = {2024}
}

@article{internvl2,
  title   = {Expanding performance boundaries of open-source multimodal models with model, data, and test-time scaling},
  author  = {Chen, Zhe and Wang, Weiyun and Cao, Yue and Liu, Yangzhou and Gao, Zhangwei and Cui, Erfei and Zhu, Jinguo and Ye, Shenglong and Tian, Hao and Liu, Zhaoyang and others},
  journal = {arXiv preprint arXiv:2412.05271},
  year    = {2024}
}

@article{internvl3,
  title   = {Internvl3: Exploring advanced training and test-time recipes for open-source multimodal models},
  author  = {Zhu, Jinguo and Wang, Weiyun and Chen, Zhe and Liu, Zhaoyang and Ye, Shenglong and Gu, Lixin and Tian, Hao and Duan, Yuchen and Su, Weijie and Shao, Jie and others},
  journal = {arXiv preprint arXiv:2504.10479},
  year    = {2025}
}

@article{internvl3.5,
  title   = {Internvl3. 5: Advancing open-source multimodal models in versatility, reasoning, and efficiency},
  author  = {Wang, Weiyun and Gao, Zhangwei and Gu, Lixin and Pu, Hengjun and Cui, Long and Wei, Xingguang and Liu, Zhaoyang and Jing, Linglin and Ye, Shenglong and Shao, Jie and others},
  journal = {arXiv preprint arXiv:2508.18265},
  year    = {2025}
}

@article{gemini,
  title   = {Gemini: a family of highly capable multimodal models},
  author  = {Team, Gemini and Anil, Rohan and Borgeaud, Sebastian and Alayrac, Jean-Baptiste and Yu, Jiahui and Soricut, Radu and Schalkwyk, Johan and Dai, Andrew M and Hauth, Anja and Millican, Katie and others},
  journal = {arXiv preprint arXiv:2312.11805},
  year    = {2023}
}

@article{glm-4.1v,
  title   = {GLM-4.1 V-Thinking: Towards Versatile Multimodal Reasoning with Scalable Reinforcement Learning},
  author  = {Hong, Wenyi and Yu, Wenmeng and Gu, Xiaotao and Wang, Guo and Gan, Guobing and Tang, Haomiao and Cheng, Jiale and Qi, Ji and Ji, Junhui and Pan, Lihang and others},
  journal = {arXiv preprint arXiv:2507.01006},
  year    = {2025}
}

@article{kimi-k2.5,
  title   = {Kimi K2. 5: Visual Agentic Intelligence},
  author  = {Team, Kimi and Bai, Tongtong and Bai, Yifan and Bao, Yiping and Cai, SH and Cao, Yuan and Charles, Y and Che, HS and Chen, Cheng and Chen, Guanduo and others},
  journal = {arXiv preprint arXiv:2602.02276},
  year    = {2026}
}

@inproceedings{blip2,
  title        = {Blip-2: Bootstrapping language-image pre-training with frozen image encoders and large language models},
  author       = {Li, Junnan and Li, Dongxu and Savarese, Silvio and Hoi, Steven},
  booktitle    = {International conference on machine learning},
  pages        = {19730--19742},
  year         = {2023},
  organization = {PMLR}
}

@inproceedings{sautenkov2025uav,
  title        = {UAV-VLA: Vision-language-action system for large scale aerial mission generation},
  author       = {Sautenkov, Oleg and Yaqoot, Yasheerah and Lykov, Artem and Mustafa, Muhammad Ahsan and Tadevosyan, Grik and Akhmetkazy, Aibek and Cabrera, Miguel Altamirano and Martynov, Mikhail and Karaf, Sausar and Tsetserukou, Dzmitry},
  booktitle    = {2025 20th ACM/IEEE International Conference on Human-Robot Interaction (HRI)},
  pages        = {1588--1592},
  year         = {2025},
  organization = {IEEE}
}

@article{uav_survey,
  title     = {UAVs meet LLMs: Overviews and perspectives towards agentic low-altitude mobility},
  author    = {Tian, Yonglin and Lin, Fei and Li, Yiduo and Zhang, Tengchao and Zhang, Qiyao and Fu, Xuan and Huang, Jun and Dai, Xingyuan and Wang, Yutong and Tian, Chunwei and others},
  journal   = {Information Fusion},
  volume    = {122},
  pages     = {103158},
  year      = {2025},
  publisher = {Elsevier}
}

@inproceedings{airvista,
  title        = {Airvista: Empowering uavs with 3d spatial reasoning abilities through a multimodal large language model agent},
  author       = {Lin, Fei and Tian, Yonglin and Wang, Yunzhe and Zhang, Tengchao and Zhang, Xinyuan and Wang, Fei-Yue},
  booktitle    = {2024 IEEE 27th International Conference on Intelligent Transportation Systems (ITSC)},
  pages        = {476--481},
  year         = {2024},
  organization = {IEEE}
}

@inproceedings{airvista2,
  author    = {Lin, Fei and Tian, Yonglin and Zhang, Tengchao and Huang, Jun and Guan, Sangtian and Wang, Fei-Yue},
  booktitle = {2025 IEEE International Conference on Systems, Man, and Cybernetics (SMC)},
  title     = {AirVista-II: An Agentic System for Embodied UAVs Toward Dynamic Scene Semantic Understanding},
  year      = {2025},
  pages     = {6319-6324},
  doi       = {10.1109/SMC58881.2025.11342598}
}

@inproceedings{wang2025jtd,
  title     = {JTD-UAV: MLLM-Enhanced Joint Tracking and Description Framework for Anti-UAV Systems},
  author    = {Wang, Yifan and Zhao, Jian and Fan, Zhaoxin and Zhang, Xin and Wu, Xuecheng and Zhang, Yudian and Jin, Lei and Li, Xinyue and Wang, Gang and Jia, Mengxi and others},
  booktitle = {Proceedings of the Computer Vision and Pattern Recognition Conference},
  pages     = {1633--1644},
  year      = {2025}
}

@inproceedings{chu2024towards,
  title        = {Towards natural language-guided drones: GeoText-1652 benchmark with spatial relation matching},
  author       = {Chu, Meng and Zheng, Zhedong and Ji, Wei and Wang, Tingyu and Chua, Tat-Seng},
  booktitle    = {European Conference on Computer Vision},
  pages        = {213--231},
  year         = {2024},
  organization = {Springer}
}

@inproceedings{Habitat2,
  author    = {Szot, Andrew and Clegg, Alexander and Undersander, Eric and Wijmans, Erik and Zhao, Yili and Turner, John and Maestre, Noah and Mukadam, Mustafa and Chaplot, Devendra Singh and Maksymets, Oleksandr and Gokaslan, Aaron and Vondru\v{s}, Vladim\'{\i}r and Dharur, Sameer and Meier, Franziska and Galuba, Wojciech and Chang, Angel and Kira, Zsolt and Koltun, Vladlen and Malik, Jitendra and Savva, Manolis and Batra, Dhruv},
  booktitle = {Advances in Neural Information Processing Systems},
  editor    = {M. Ranzato and A. Beygelzimer and Y. Dauphin and P.S. Liang and J. Wortman Vaughan},
  pages     = {251--266},
  publisher = {Curran Associates, Inc.},
  title     = {Habitat 2.0: Training Home Assistants to Rearrange their Habitat},
  url       = {https://proceedings.neurips.cc/paper_files/paper/2021/file/021bbc7ee20b71134d53e20206bd6feb-Paper.pdf},
  volume    = {34},
  year      = {2021}
}

@article{zhang2025your,
  title   = {Is your VLM Sky-Ready? A Comprehensive Spatial Intelligence Benchmark for UAV Navigation},
  author  = {Zhang, Lingfeng and Zhang, Yuchen and Li, Hongsheng and Fu, Haoxiang and Tang, Yingbo and Ye, Hangjun and Chen, Long and Liang, Xiaojun and Hao, Xiaoshuai and Ding, Wenbo},
  journal = {arXiv preprint arXiv:2511.13269},
  year    = {2025}
}

@article{guo2025bedi,
  title   = {Bedi: A comprehensive benchmark for evaluating embodied agents on uavs},
  author  = {Guo, Mingning and Wu, Mengwei and He, Jiarun and Li, Shaoxian and Li, Haifeng and Tao, Chao},
  journal = {arXiv preprint arXiv:2505.18229},
  year    = {2025}
}

@article{wang2025uav,
  title   = {UAV-Flow Colosseo: A Real-World Benchmark for Flying-on-a-Word UAV Imitation Learning},
  author  = {Wang, Xiangyu and Yang, Donglin and Liao, Yue and Zheng, Wenhao and Dai, Bin and Li, Hongsheng and Liu, Si and others},
  journal = {arXiv preprint arXiv:2505.15725},
  year    = {2025}
}

@article{yao2024aeroverse,
  title   = {Aeroverse: Uav-agent benchmark suite for simulating, pre-training, finetuning, and evaluating aerospace embodied world models},
  author  = {Yao, Fanglong and Yue, Yuanchang and Liu, Youzhi and Sun, Xian and Fu, Kun},
  journal = {arXiv preprint arXiv:2408.15511},
  year    = {2024}
}

@misc{nvidia2025cosmos,
  title   = {Cosmos-Reason1: From Physical Common Sense To Embodied Reasoning},
  author  = {NVIDIA and Azzolini, Alisson and Brandon, Hannah and Chattopadhyay, Prithvijit and Chen, Huayu and Chu, Jinju and Cui, Yin and Diamond, Jenna and Ding, Yifan and Ferroni, Francesco and Govindaraju, Rama and Gu, Jinwei and Gururani, Siddharth and El Hanafi, Imad and Hao, Zekun and Huffman, Jacob and Jin, Jingyi and Johnson, Brendan and Khan, Rizwan and Kurian, George and Lantz, Elena and Lee, Nayeon and Li, Zhaoshuo and Li, Xuan and Lin, Tsung-Yi and Lin, Yen-Chen and Liu, Ming-Yu and Mathau, Andrew and Ni, Yun and Pavao, Lindsey and Ping, Wei and Romero, David W. and Smelyanskiy, Misha and Song, Shuran and Tchapmi, Lyne and Wang, Andrew Z. and Wang, Boxin and Wang, Haoxiang and Wei, Fangyin and Xu, Jiashu and Xu, Yao and Yang, Xiaodong and Yang, Zhuolin and Zeng, Xiaohui and Zhang, Zhe},
  journal = {arXiv preprint arXiv:2503.15558},
  year    = {2025},
  url     = {https://arxiv.org/abs/2503.15558}
}

@article{eo1,
  title   = {EO-1: Interleaved Vision-Text-Action Pretraining for General Robot Control},
  author  = {Delin Qu and Haoming Song and Qizhi Chen and Zhaoqing Chen and Xianqiang Gao and Xinyi Ye and Qi Lv and Modi Shi and Guanghui Ren and Cheng Ruan and Maoqing Yao and Haoran Yang and Jiacheng Bao and Bin Zhao and Dong Wang},
  journal = {arXiv preprint},
  year    = {2025},
  url     = {https://arxiv.org/abs/2508.21112}
}

@article{fang2025robix,
  title   = {{Robix: A Unified Model for Robot Interaction, Reasoning and Planning}},
  author  = {Huang Fang and Mengxi Zhang and Heng Dong and Wei Li, Zixuan Wang and Qifeng Zhang and Xueyun Tian and Yucheng Hu and Hang Li},
  journal = {arXiv preprint arXiv:2509.01106},
  year    = {2025}
}

@inproceedings{vsi,
  title={Thinking in space: How multimodal large language models see, remember, and recall spaces},
  author={Yang, Jihan and Yang, Shusheng and Gupta, Anjali W and Han, Rilyn and Fei-Fei, Li and Xie, Saining},
  booktitle={Proceedings of the Computer Vision and Pattern Recognition Conference},
  pages={10632--10643},
  year={2025}
}

@article{sat,
  title={Sat: Dynamic spatial aptitude training for multimodal language models},
  author={Ray, Arijit and Duan, Jiafei and Brown, Ellis and Tan, Reuben and Bashkirova, Dina and Hendrix, Rose and Ehsani, Kiana and Kembhavi, Aniruddha and Plummer, Bryan A and Krishna, Ranjay and others},
  journal={arXiv preprint arXiv:2412.07755},
  year={2024}
}

@article{vica322k,
  title={Visuospatial cognitive assistant},
  author={Feng, Qi},
  journal={arXiv preprint arXiv:2505.12312},
  year={2025}
}

@article{gpt4scene,
  title={Gpt4scene: Understand 3d scenes from videos with vision-language models},
  author={Qi, Zhangyang and Zhang, Zhixiong and Fang, Ye and Wang, Jiaqi and Zhao, Hengshuang},
  journal={arXiv preprint arXiv:2501.01428},
  year={2025}
}

@article{scene_30k,
  title={3d-r1: Enhancing reasoning in 3d vlms for unified scene understanding},
  author={Huang, Ting and Zhang, Zeyu and Tang, Hao},
  journal={arXiv preprint arXiv:2507.23478},
  year={2025}
}

@article{vlm_3r,
  title={Vlm-3r: Vision-language models augmented with instruction-aligned 3d reconstruction},
  author={Fan, Zhiwen and Zhang, Jian and Li, Renjie and Zhang, Junge and Chen, Runjin and Hu, Hezhen and Wang, Kevin and Qu, Huaizhi and Wang, Dilin and Yan, Zhicheng and others},
  journal={arXiv preprint arXiv:2505.20279},
  year={2025}
}

@inproceedings{lingoqa,
  title={Lingoqa: Visual question answering for autonomous driving},
  author={Marcu, Ana-Maria and Chen, Long and H{\"u}nermann, Jan and Karnsund, Alice and Hanotte, Benoit and Chidananda, Prajwal and Nair, Saurabh and Badrinarayanan, Vijay and Kendall, Alex and Shotton, Jamie and others},
  booktitle={European Conference on Computer Vision},
  pages={252--269},
  year={2024},
  organization={Springer}
}

@inproceedings{embspatial,
  title={Embspatial-bench: Benchmarking spatial understanding for embodied tasks with large vision-language models},
  author={Du, Mengfei and Wu, Binhao and Li, Zejun and Huang, Xuan-Jing and Wei, Zhongyu},
  booktitle={Proceedings of the 62nd Annual Meeting of the Association for Computational Linguistics (Volume 2: Short Papers)},
  pages={346--355},
  year={2024}
}

@article{spacer_151k,
  title={Spacer: Reinforcing mllms in video spatial reasoning},
  author={Ouyang, Kun and Liu, Yuanxin and Wu, Haoning and Liu, Yi and Zhou, Hao and Zhou, Jie and Meng, Fandong and Sun, Xu},
  journal={arXiv preprint arXiv:2504.01805},
  year={2025}
}

@inproceedings{lmms_eval,
  title={Lmms-eval: Reality check on the evaluation of large multimodal models},
  author={Zhang, Kaichen and Li, Bo and Zhang, Peiyuan and Pu, Fanyi and Cahyono, Joshua Adrian and Hu, Kairui and Liu, Shuai and Zhang, Yuanhan and Yang, Jingkang and Li, Chunyuan and others},
  booktitle={Findings of the Association for Computational Linguistics: NAACL 2025},
  pages={881--916},
  year={2025}
}

@inproceedings{nuscenesqa,
  title={Nuscenes-qa: A multi-modal visual question answering benchmark for autonomous driving scenario},
  author={Qian, Tianwen and Chen, Jingjing and Zhuo, Linhai and Jiao, Yang and Jiang, Yu-Gang},
  booktitle={Proceedings of the AAAI Conference on Artificial Intelligence},
  volume={38},
  number={5},
  pages={4542--4550},
  year={2024}
}

@inproceedings{multi3drefer,
  title={Multi3drefer: Grounding text description to multiple 3d objects},
  author={Zhang, Yiming and Gong, ZeMing and Chang, Angel X},
  booktitle={Proceedings of the IEEE/CVF International Conference on Computer Vision},
  pages={15225--15236},
  year={2023}
}

@inproceedings{maplm,
  title={Maplm: A real-world large-scale vision-language benchmark for map and traffic scene understanding},
  author={Cao, Xu and Zhou, Tong and Ma, Yunsheng and Ye, Wenqian and Cui, Can and Tang, Kun and Cao, Zhipeng and Liang, Kaizhao and Wang, Ziran and Rehg, James M and others},
  booktitle={Proceedings of the IEEE/CVF conference on computer vision and pattern recognition},
  pages={21819--21830},
  year={2024}
}

@article{driveaction,
  title={Driveaction: A benchmark for exploring human-like driving decisions in vla models},
  author={Hao, Yuhan and Li, Zhengning and Sun, Lei and Wang, Weilong and Yi, Naixin and Song, Sheng and Qin, Caihong and Zhou, Mofan and Zhan, Yifei and Lang, Xianpeng},
  journal={arXiv preprint arXiv:2506.05667},
  year={2025}
}

@article{embodiedbench,
  title={Embodiedbench: Comprehensive benchmarking multi-modal large language models for vision-driven embodied agents},
  author={Yang, Rui and Chen, Hanyang and Zhang, Junyu and Zhao, Mark and Qian, Cheng and Wang, Kangrui and Wang, Qineng and Koripella, Teja Venkat and Movahedi, Marziyeh and Li, Manling and others},
  journal={arXiv preprint arXiv:2502.09560},
  year={2025}
}

@article{fyhn2004spatial,
  title={Spatial representation in the entorhinal cortex},
  author={Fyhn, Marianne and Molden, Sturla and Witter, Menno P and Moser, Edvard I and Moser, May-Britt},
  journal={Science},
  volume={305},
  number={5688},
  pages={1258--1264},
  year={2004},
  publisher={American Association for the Advancement of Science}
}

@article{hafting2005microstructure,
  title={Microstructure of a spatial map in the entorhinal cortex},
  author={Hafting, Torkel and Fyhn, Marianne and Molden, Sturla and Moser, May-Britt and Moser, Edvard I},
  journal={Nature},
  volume={436},
  number={7052},
  pages={801--806},
  year={2005},
  publisher={Nature Publishing Group UK London}
}

@article{tian2026domain,
  title={Domain-specific schema reuse supports flexible learning to learn in the primate brain},
  author={Tian, Kaixi and Zhao, Zhiping and Chen, Yang and Ge, Ningling and Cao, Shenghao and Han, Xinyong and Gu, Jianwen and Yu, Shan},
  journal={Nature Communications},
  year={2026},
  publisher={Nature Publishing Group}
}

@article{goudar2023schema,
  title={Schema formation in a neural population subspace underlies learning-to-learn in flexible sensorimotor problem-solving},
  author={Goudar, Vishwa and Peysakhovich, Barbara and Freedman, David J and Buffalo, Elizabeth A and Wang, Xiao-Jing},
  journal={Nature Neuroscience},
  volume={26},
  number={5},
  pages={879--890},
  year={2023},
  publisher={Nature Publishing Group US New York}
}

@book{mohri2018foundations,
  title={Foundations of machine learning},
  author={Mohri, Mehryar and Rostamizadeh, Afshin and Talwalkar, Ameet},
  year={2018},
  publisher={MIT press}
}

@book{shalev2014understanding,
  title={Understanding machine learning: From theory to algorithms},
  author={Shalev-Shwartz, Shai and Ben-David, Shai},
  year={2014},
  publisher={Cambridge university press}
}

@article{zhou2022domain,
  title={Domain generalization: A survey},
  author={Zhou, Kaiyang and Liu, Ziwei and Qiao, Yu and Xiang, Tao and Loy, Chen Change},
  journal={IEEE transactions on pattern analysis and machine intelligence},
  volume={45},
  number={4},
  pages={4396--4415},
  year={2022},
  publisher={IEEE}
}

@misc{mimovl,
      title={MiMo-VL Technical Report}, 
      author={Core Team and Zihao Yue and Zhenru Lin and Yifan Song and Weikun Wang and Shuhuai Ren and Shuhao Gu and Shicheng Li and Peidian Li and Liang Zhao and Lei Li and Kainan Bao and Hao Tian and Hailin Zhang and Gang Wang and Dawei Zhu and Cici and Chenhong He and Bowen Ye and Bowen Shen and Zihan Zhang and Zihan Jiang and Zhixian Zheng and Zhichao Song and Zhenbo Luo and Yue Yu and Yudong Wang and Yuanyuan Tian and Yu Tu and Yihan Yan and Yi Huang and Xu Wang and Xinzhe Xu and Xingchen Song and Xing Zhang and Xing Yong and Xin Zhang and Xiangwei Deng and Wenyu Yang and Wenhan Ma and Weiwei Lv and Weiji Zhuang and Wei Liu and Sirui Deng and Shuo Liu and Shimao Chen and Shihua Yu and Shaohui Liu and Shande Wang and Rui Ma and Qiantong Wang and Peng Wang and Nuo Chen and Menghang Zhu and Kangyang Zhou and Kang Zhou and Kai Fang and Jun Shi and Jinhao Dong and Jiebao Xiao and Jiaming Xu and Huaqiu Liu and Hongshen Xu and Heng Qu and Haochen Zhao and Hanglong Lv and Guoan Wang and Duo Zhang and Dong Zhang and Di Zhang and Chong Ma and Chang Liu and Can Cai and Bingquan Xia},
      year={2025},
      eprint={2506.03569},
      archivePrefix={arXiv},
      primaryClass={cs.CL},
      url={https://arxiv.org/abs/2506.03569}, 
}

@article{mmsi,
  title={Mmsi-bench: A benchmark for multi-image spatial intelligence},
  author={Yang, Sihan and Xu, Runsen and Xie, Yiman and Yang, Sizhe and Li, Mo and Lin, Jingli and Zhu, Chenming and Chen, Xiaochen and Duan, Haodong and Yue, Xiangyu and others},
  journal={arXiv preprint arXiv:2505.23764},
  year={2025}
}

@inproceedings{blink,
  title={Blink: Multimodal large language models can see but not perceive},
  author={Fu, Xingyu and Hu, Yushi and Li, Bangzheng and Feng, Yu and Wang, Haoyu and Lin, Xudong and Roth, Dan and Smith, Noah A and Ma, Wei-Chiu and Krishna, Ranjay},
  booktitle={European Conference on Computer Vision},
  pages={148--166},
  year={2024},
  organization={Springer}
}

@inproceedings{site,
  title={Site: towards spatial intelligence thorough evaluation},
  author={Wang, Wenqi and Tan, Reuben and Zhu, Pengyue and Yang, Jianwei and Yang, Zhengyuan and Wang, Lijuan and Kolobov, Andrey and Gao, Jianfeng and Gong, Boqing},
  booktitle={Proceedings of the IEEE/CVF International Conference on Computer Vision},
  pages={9058--9069},
  year={2025}
}

@article{claude4,
  title={Claude Sonnet 4},
  author={Anthropic},
  year={2025}
}

@inproceedings{scannet,
  title={Scannet: Richly-annotated 3d reconstructions of indoor scenes},
  author={Dai, Angela and Chang, Angel X and Savva, Manolis and Halber, Maciej and Funkhouser, Thomas and Nie{\ss}ner, Matthias},
  booktitle={Proceedings of the IEEE conference on computer vision and pattern recognition},
  pages={5828--5839},
  year={2017}
}

@article{arkitscenes,
  title={Arkitscenes: A diverse real-world dataset for 3d indoor scene understanding using mobile rgb-d data},
  author={Baruch, Gilad and Chen, Zhuoyuan and Dehghan, Afshin and Dimry, Tal and Feigin, Yuri and Fu, Peter and Gebauer, Thomas and Joffe, Brandon and Kurz, Daniel and Schwartz, Arik and others},
  journal={arXiv preprint arXiv:2111.08897},
  year={2021}
}

@inproceedings{scannet++,
  title={Scannet++: A high-fidelity dataset of 3d indoor scenes},
  author={Yeshwanth, Chandan and Liu, Yueh-Cheng and Nie{\ss}ner, Matthias and Dai, Angela},
  booktitle={Proceedings of the IEEE/CVF International Conference on Computer Vision},
  pages={12--22},
  year={2023}
}

@article{shen2025efficient,
  title={Efficient and effective weight-ensembling mixture of experts for multi-task model merging},
  author={Shen, Li and Tang, Anke and Yang, Enneng and Guo, Guibing and Luo, Yong and Zhang, Lefei and Cao, Xiaochun and Du, Bo and Tao, Dacheng},
  journal={IEEE Transactions on Pattern Analysis and Machine Intelligence},
  year={2025},
  publisher={IEEE}
}

@article{matterport3d,
  title={Matterport3d: Learning from rgb-d data in indoor environments},
  author={Chang, Angel and Dai, Angela and Funkhouser, Thomas and Halber, Maciej and Niessner, Matthias and Savva, Manolis and Song, Shuran and Zeng, Andy and Zhang, Yinda},
  journal={arXiv preprint arXiv:1709.06158},
  year={2017}
}

@article{procthor,
  title={ProcTHOR: Large-Scale Embodied AI Using Procedural Generation},
  author={Deitke, Matt and VanderBilt, Eli and Herrasti, Alvaro and Weihs, Luca and Salvador, Jordi and Ehsani, Kiana and Han, Winson and Kolve, Eric and Farhadi, Ali and Kembhavi, Aniruddha and others},
  journal={arXiv preprint arXiv:2206.06994},
  year={2022}
}

@inproceedings{llava,
  title={Improved baselines with visual instruction tuning},
  author={Liu, Haotian and Li, Chunyuan and Li, Yuheng and Lee, Yong Jae},
  booktitle={Proceedings of the IEEE/CVF conference on computer vision and pattern recognition},
  pages={26296--26306},
  year={2024}
}
\clearpage

\appendix
\section{Appendix}

\subsection{Mathematical Foundations of the Spatial Scaffold Mechanism}

\paragraph{Notation.}
Let $\mathcal{M}=\{m_1,\dots,m_K\}$ denote a set of morphologies. Each morphology $m \in \mathcal{M}$ induces a data distribution $D_{m}$ over tuples $(o,c,y)$, where $o \in \mathcal{O}_{m}$ represents multimodal observations, $c \in \mathcal{C}$ denotes task conditioning, and $y \in \mathcal{Y}_{m}$ is the target output. We evaluate the generalization performance on morphology $m$ via the expected risk:
\[
    R_m(\theta) \;:=\; \mathbb{E}_{(o,c,y)\sim D_m}\!\left[\ell_\theta(o,c,y)\right].
\]
We further define the gradient of the risk as
\[
    g_m(\theta) \;:=\; \nabla_\theta R_m(\theta).
\]
The per-sample loss is defined as the negative log-likelihood under an autoregressive model:
\[
    \ell_\theta(o,c,y) \;:=\; -\log p_\theta(y \mid o,c), 
    \quad \text{where} \quad 
    p_\theta(y \mid o,c) \;=\; \prod_{t=1}^{T} p_\theta(y_t \mid y_{<t}, o, c),
\]
which yields the equivalent form:
\[
    \ell_\theta(o,c,y) \;=\; -\sum_{t=1}^{T} \log p_\theta(y_t \mid y_{<t}, o, c).
\]
Let $\theta_{\mathrm{base}}$ denote the base model. During the isolation stage, we obtain an expert $\theta_m$ for each morphology $m$. We then define the \textit{task vector} as
\[
    \tau_m \;:=\; \theta_m - \theta_{\mathrm{base}}.
\]
Adopting a layer-wise decomposition with $L$ layers, we denote $\theta_m=(\theta_{m,1},\dots,\theta_{m,L})$ and $\tau_m=(\tau_{m,1},\dots,\tau_{m,L})$, where $\tau_{m,l} := \theta_{m,l} - \theta_{\mathrm{base},l}$ for each layer $l \in \{1,\dots,L\}$.

\paragraph{Shared geometric scaffold as a universal spatial bridge.}
We interpret spatial intelligence as the ability to infer and reuse a morphology-invariant
spatial scaffold across diverse embodiments.
Concretely, there exists a shared latent spatial variable $g\in\mathcal G$ that captures 3D spatial relationships
(layout, relative pose, depth, topology) and remains invariant across morphologies, and a morphology-specific latent
$a_m\in\mathcal A_m$ that captures embodiment-dependent factors (sensor intrinsics, dynamics constraints, actuation semantics).
For samples $(o,c,y)\sim D_m$, we posit the mechanism
\[
o=\Psi_m(g,a_m),\qquad y=\Phi(g,c),
\]
where $g \sim P_G$ and $a_m\sim P_{A|m}$. Thus $g$ serves as a \textbf{universal bridge} for cross-embodiment transfer, while morphology-specific variation is absorbed
by $\Psi_m$ and $a_m$.

\begin{assumption}[Recoverable spatial scaffold representation after \textsc{Scaffold}]
\label{ass:recoverability_aligned}
Let $\theta_{\mathrm{spatial}}$ denote the parameters after the \textsc{Scaffold} stage and define the induced representation
$z_{\mathrm{sp}}:=h_{\theta_{\mathrm{spatial}}}(o,c)$, where $o\in\mathcal{O}_m$ and $c\in \mathcal{C}$. There exist a decoder ${\sf Dec}(\cdot)$ and a constant $\varepsilon_g\ge 0$ such that
\[
\mathbb E\Big[\big\|{\sf Dec}(z_{\mathrm{sp}})-g\big\|\Big]\le\varepsilon_g.
\]
\end{assumption}

Assumption \ref{ass:recoverability_aligned} is well founded because it formalizes a decodability property that serves as the primary objective of the Scaffold stage. By making geometric information readable from the representation, the shared scaffold $g$ becomes a stable and transferable spatial core. Consequently, downstream learning remains focused on morphology specific factors rather than the redundant re-learning of geometry. The low recoverability error $\varepsilon_g$ is empirically supported by functional decodability where the Spatial Expert demonstrates a marked improvement on spatial benchmarks from a baseline of $51.6 \rightarrow 72.5$. This spatial foundation catalyzes learning in other morphologies as seen in the substantial gains for autonomous driving and aerial tasks when initialized from the spatial model. These results indicate the effective reuse of a shared scaffold instead of task specific overfitting. This concept also aligns with neuroscientific findings regarding reusable internal spatial codes that remain readable despite changing sensory conditions as described in Remark \ref{remark:Neuroscientific}.

\begin{remark}
Assumption~\ref{ass:recoverability_aligned} states that \textsc{Scaffold} training induces a reusable representation
from which the shared spatial scaffold $g$ can be reliably recovered (up to error $\varepsilon_g$).
This makes $g$ a stable, transferable spatial core across morphologies, and isolates subsequent adaptation to
morphology-specific factors rather than relearning geometry.
\end{remark}

\begin{remark}[Neuroscientific evidence supporting reusable spatial scaffolds and representation space separation]\label{remark:Neuroscientific}
The mechanistic interpretation of \textsc{Scaffold} treats $g$ as a morphology-invariant spatial scaffold supporting decision-relevant reasoning while morphology-dependent factors are absorbed by the mapping $\Psi_m$. This perspective is strongly supported by two primary lines of neuroscientific evidence.

First, research on hippocampal and entorhinal systems suggests that mammals maintain an internal and reusable spatial coordinate representation. This mechanism is frequently identified as an internal global positioning system that enables localization and navigation across changing sensory conditions. Such biological findings provide a foundation for a shared geometry-like variable $g$ serving as a cross-domain substrate for planning and spatial reasoning \cite{fyhn2004spatial,hafting2005microstructure}.

Second, population-level recordings and modeling studies indicate that reusable task structures are organized in stable low-dimensional subspaces which remain nearly separable from context-specific or stimulus-specific variation. For example, recent primate recordings identify decision-related and stimulus-related subspaces with a near-orthogonal relationship. In this framework, the reuse of schema-like patterns in the decision subspace facilitates subsequent learning and reduces cross-domain interference \cite{tian2026domain}. Complementary evidence shows that reusable schemas emerge within low-dimensional subspaces of population activity. The reuse of such schemas in new problems accelerates learning by limiting weight changes \cite{goudar2023schema}. These empirical results motivate the emphasis on geometric recoverability. Specifically, \textsc{Scaffold} aims to induce a representation $h_{\theta_{\mathrm{spatial}}}$ from which $g$ is reliably decodable so that downstream adaptation can focus on morphology-specific factors rather than relearning the shared spatial core.
\end{remark}

\subsection{Gradient Interference and the Necessity of Isolation Training}
\label{app:thm1_aligned}

This subsection formalizes a key obstacle in joint cross-morphology training: gradients from different morphologies can be misaligned,
so a single shared update may fail to simultaneously decrease all morphology-specific risks.
We show that negative cross-gradient inner products naturally quantify such interference, which slows down, stalls, or even reverses
progress on some morphologies. This provides a theoretical justification for isolation training, which removes cross-terms by construction.

\begin{assumption}[Smoothness]
\label{ass:smooth_aligned}
For each $m\in\mathcal M$, the risk $R_m(\theta)$ is $L$-smooth:
for all $\theta,\theta'$,
\begin{equation}
R_m(\theta') \le R_m(\theta) + \langle g_m(\theta), \theta'-\theta\rangle + \frac{L}{2}\|\theta'-\theta\|^2.
\end{equation}
\end{assumption}

Assumption \ref{ass:smooth_aligned} is appropriate because the training setup utilizes the Adam optimizer with a modest learning rate coupled with robust stabilizers such as gradient clipping and weight decay. Such measures ensure that parameter updates remain within a locally well-behaved region where the objective function is effectively approximated as nearly quadratic. Under these conditions it is reasonable to upper bound the local curvature with a single smoothness constant $L$. Additionally, $L$ smoothness is a standard regularity assumption in the field of optimization and is frequently adopted to formulate stability arguments for gradient-based approaches \cite{shalev2014understanding,mohri2018foundations}.

\begin{theorem}[One-step interference bound]
\label{thm:interference_aligned}
Let $w\in\Delta_K$ be nonnegative weights, and define the joint risk
$R_w(\theta):=\sum_{j=1}^K w_j R_j(\theta)$.
Consider one shared update
\begin{equation}
\theta^+ \;=\; \theta - \eta \nabla R_w(\theta)
\;=\; \theta - \eta\sum_{j=1}^K w_j g_j(\theta).
\end{equation}
Under Assumption~\ref{ass:smooth_aligned}, for any morphology $i\in\mathcal M$,
\begin{align}
R_i(\theta^+) 
\;\le\; R_i(\theta)
-\eta\Big(w_i\|g_i(\theta)\|^2 + \sum_{j\neq i} w_j\langle g_i(\theta),g_j(\theta)\rangle\Big)
+\frac{L\eta^2}{2}\Big\|\sum_{j=1}^K w_j g_j(\theta)\Big\|^2.
\label{eq:interference_aligned}
\end{align}
\end{theorem}

\begin{proof}[Proof of Theorem \ref{thm:interference_aligned}]
Fix an arbitrary morphology $i\in\mathcal M$.
We apply the $L$-smoothness inequality (Assumption~\ref{ass:smooth_aligned}) to the function $R_i$,
with the pair $(\theta,\theta')=(\theta,\theta^+)$.
By definition of the shared update, we have
\[
\theta^+ - \theta \;=\; -\eta \sum_{j=1}^K w_j g_j(\theta).
\]
Substituting $\theta'=\theta^+$ and $\theta'-\theta=\theta^+-\theta$ into the smoothness inequality yields
\begin{align}
R_i(\theta^+)
&\le R_i(\theta)
+ \left\langle g_i(\theta),\, \theta^+-\theta \right\rangle
+ \frac{L}{2}\|\theta^+-\theta\|^2.
\label{eq:step1_smooth}
\end{align}

Using $\theta^+-\theta=-\eta\sum_{j=1}^K w_j g_j(\theta)$, the inner product in
Eq. (\ref{eq:step1_smooth}) becomes
\begin{align}
\left\langle g_i(\theta),\, \theta^+-\theta \right\rangle
&= \left\langle g_i(\theta),\, -\eta\sum_{j=1}^K w_j g_j(\theta) \right\rangle \nonumber\\
&= -\eta \sum_{j=1}^K w_j \left\langle g_i(\theta), g_j(\theta)\right\rangle.
\label{eq:step2_inner}
\end{align}
Now separate the $j=i$ term from the rest:
\begin{align}
\sum_{j=1}^K w_j \left\langle g_i(\theta), g_j(\theta)\right\rangle
&= w_i \langle g_i(\theta), g_i(\theta)\rangle
+ \sum_{j\neq i} w_j \langle g_i(\theta), g_j(\theta)\rangle \nonumber\\
&= w_i \|g_i(\theta)\|^2
+ \sum_{j\neq i} w_j \langle g_i(\theta), g_j(\theta)\rangle.
\label{eq:step3_split}
\end{align}
Combining Eq. (\ref{eq:step2_inner}) and Eq. (\ref{eq:step3_split}), we obtain
\begin{align}
\left\langle g_i(\theta),\, \theta^+-\theta \right\rangle
&= -\eta\Big( w_i \|g_i(\theta)\|^2
+ \sum_{j\neq i} w_j \langle g_i(\theta), g_j(\theta)\rangle \Big).
\label{eq:step4_linear_final}
\end{align}

Again using $\theta^+-\theta=-\eta\sum_{j=1}^K w_j g_j(\theta)$, we have
\begin{align}
\|\theta^+-\theta\|^2
&= \left\| -\eta\sum_{j=1}^K w_j g_j(\theta)\right\|^2
= \eta^2 \left\|\sum_{j=1}^K w_j g_j(\theta)\right\|^2.
\label{eq:step5_quad}
\end{align}
Therefore,
\begin{align}
\frac{L}{2}\|\theta^+-\theta\|^2
&= \frac{L\eta^2}{2}\left\|\sum_{j=1}^K w_j g_j(\theta)\right\|^2.
\label{eq:step6_quad_final}
\end{align}

Substituting Eq. (\ref{eq:step4_linear_final}) and Eq. (\ref{eq:step6_quad_final}) into Eq. (\ref{eq:step1_smooth}),
we conclude
\[
R_i(\theta^+) 
\le R_i(\theta)
-\eta\Big( w_i \|g_i(\theta)\|^2
+ \sum_{j\neq i} w_j \langle g_i(\theta), g_j(\theta)\rangle \Big)
+\frac{L\eta^2}{2}\Big\|\sum_{j=1}^K w_j g_j(\theta)\Big\|^2,
\]
which is exactly Eq. (\ref{eq:interference_aligned}). This completes the proof.
\end{proof}

\begin{remark}[Connection to isolation and gradient interference]
Equation (\ref{eq:interference_aligned}) decomposes the change of $R_i$ into a beneficial self term and certain cross terms. Specifically, the self term is defined as $w_i \|g_i(\theta)\|^2$ while the cross terms are expressed as the inner product $\langle g_i(\theta), g_j(\theta) \rangle$. These cross terms quantify the extent to which updates favored by other morphologies influence task $i$. When these terms remain persistently negative throughout the training process, the shared update contains a nontrivial component pointing against the descent direction of $R_i$. This conflict reduces the effective descent and may stall or even increase $R_i$ depending on the respective magnitudes of the terms. This phenomenon represents the core mechanism of gradient interference within shared parameters. Isolation training removes these cross terms by construction because each expert is optimized using only its own gradient. Consequently, this approach eliminates the dominant source of interference before the reconciliation or merging phases.
\end{remark}

\subsection{Spatial Scaffold as a Universal Bridge: A Transfer Bound}
\label{app:thm2_aligned}

\noindent
This subsection formalizes the universal bridge claim of spatial intelligence by deriving a transfer bound from
spatial-scaffold training to an unseen target morphology. The result shows that the target risk is controlled by the
recoverability of the shared geometric scaffold and the geometric mismatch between the target morphology and the scaffold
training distribution, thereby providing a theoretical justification for the effectiveness of the \textsc{Scaffold} stage.



\begin{assumption}[Lipschitz Dependence on Geometry]
\label{ass:lipschitz_aligned}
There exists $L_g>0$ such that for any $c$ and any $g,g'\in\mathcal G$,
\[
\big|\ell(g,c;\theta) - \ell(g',c;\theta)\big|
\;\le\; L_g \|g-g'\|,
\]
where $\ell(g,c;\theta)$ denotes the loss evaluated on samples whose semantics are governed by $g$.
\end{assumption}
Assumption \ref{ass:lipschitz_aligned} is grounded in the inherent continuity of spatial supervision signals such as distance, orientation, and temporal metrics. This smoothness persists near reachability boundaries where minor geometric deviations produce bounded loss variations. A Lipschitz condition formally encapsulates this stability property. Furthermore such regularity assumptions represent standard theoretical frameworks extensively utilized in machine learning \cite{shalev2014understanding,mohri2018foundations}.

\begin{assumption}[Geometric Distribution Shift]
\label{ass:shift_aligned}
Let $P_G^{(m)}$ be the marginal distribution of $g$ induced by $D_m$.
Let $P_G^{(\mathrm{sp})}$ be the geometric marginal induced by the spatial scaffold distribution.
Assume there exists a discrepancy measure $\mathrm{Disc}(\cdot,\cdot)$ and a constant $\delta_m\ge 0$ such that
\[
\mathrm{Disc}\big(P_G^{(m)}, P_G^{(\mathrm{sp})}\big)\;\le\;\delta_m.
\]
\end{assumption}
Assumption \ref{ass:shift_aligned} represents morphological diversity as an interpretable scalar $\delta_m$ using a geometric discrepancy bound. This formulation simplifies complex distributional shifts by assuming a valid measure exists to satisfy $Disc(P_G^{(m)}, P_G^{(sp)}) \le \delta_m$. Theoretically, $\delta_m$ quantifies the transfer impact of geometric mismatches which is supported by the observed performance asymmetries across different viewpoints and motion patterns. Practical success in broadening the scaffold distribution further confirms that reducing this discrepancy directly enhances cross embodiment generalization. This strategy aligns with standard transfer learning theories where a discrepancy metric bounds the shift between source and target domains \cite{zhou2022domain}.

\begin{theorem}[Scaffold-to-morphology transfer bound]
\label{thm:scaffold_transfer_aligned}
Assume Assumptions~\ref{ass:recoverability_aligned}--\ref{ass:shift_aligned}.
Let $\theta_{\mathrm{spatial}}$ be obtained from scaffold spatial training.
Then for any target morphology $m\in\mathcal M$,
\begin{equation}
R_m(\theta_{\mathrm{spatial}})
\le
R_{\mathrm{sp}}(\theta_{\mathrm{spatial}})
+ C_m\delta_m
+2L_g\varepsilon_g+ \varepsilon_m,
\label{eq:scaffold_transfer_aligned}
\end{equation}
where $R_{\mathrm{sp}}(\theta)$ is the risk under the scaffold distribution,
$C_m$ can be taken as a constant proportional to $L_g$,
$\delta_m$ quantifies geometric shift, and $\varepsilon_m$ aggregates residual morphology-dependent effects.
\end{theorem}

\begin{proof}[Proof of Theorem \ref{thm:scaffold_transfer_aligned}]
Fix a target morphology $m\in\mathcal M$ and consider $\theta=\theta_{\mathrm{spatial}}$.
We prove Eq. (\ref{eq:scaffold_transfer_aligned}) by decomposing the target risk into:
(i) a scaffold risk term, (ii) a representation-to-geometry recovery error term, and
(iii) a geometric marginal shift term, plus residual morphology-dependent effects.

From shared geometric scaffold mechanism, we have $o=\Psi_m(g,a_m)$ and $y=\Phi(g,c)$.
Let the learning system incur a per-sample loss $\ell_\theta(o,c,y)$.
Define the target risk
\[
R_m(\theta) = \mathbb E_{(o,c,y)\sim D_m}\big[ \ell_\theta(o,c,y)\big],
\]
and similarly define $R_{\mathrm{sp}}(\theta)$ under the scaffold (spatial) distribution.

Then, conditional on $(g,c)$, the semantics of $y$ is
approximately morphology-invariant and most morphology-dependence is absorbed by the observation channel.
Formally, we introduce an ideal semantic loss $\ell(g,c;\theta)$ such that
\begin{equation}
\Big|\mathbb E\left[\ell_\theta(o,c,y)\mid g,c\right] - \ell(g,c;\theta)\Big|
\le \Delta_m(g,c;\theta),
\label{eq:def_residual}
\end{equation}
where $\Delta_m$ quantifies the residual dependence on morphology-specific factors (e.g., imperfect invariance,
weak dependence on $a_m$, or dependence on the observation channel not fully mediated by $g$).
Define the aggregated residual term
\[
\varepsilon_m := \mathbb E_{(g,c)\sim P_{G,C}^{(m)}}\big[\Delta_m(g,c;\theta_{\mathrm{spatial}})\big],
\]
which is finite by construction.

Taking expectation of Eq. (\ref{eq:def_residual}) over $(g,c)$ induced by $D_m$, we obtain
\begin{align}
R_m(\theta_{\mathrm{spatial}})
&= \mathbb E_{(g,c)\sim P_{G,C}^{(m)}}\left[\mathbb E[\ell_{\theta_{\mathrm{spatial}}}(o,c,y)\mid g,c]\right] \nonumber\\
&\le \mathbb E_{(g,c)\sim P_{G,C}^{(m)}}\left[\ell(g,c;\theta_{\mathrm{spatial}})\right] + \varepsilon_m.
\label{eq:rm_to_semantic}
\end{align}
Thus it suffices to upper-bound the semantic risk under the target geometry distribution.

Let $\widehat g := {\sf Dec}(h_{\theta_{\mathrm{spatial}}}(o,c))$ be the recovered geometry from the learned representation.
By Assumption~\ref{ass:lipschitz_aligned}, for any $(g,c)$,
\begin{equation}
\big|\ell(g,c;\theta_{\mathrm{spatial}}) - \ell(\widehat g,c;\theta_{\mathrm{spatial}})\big|
\le L_g \|g-\widehat g\|.
\label{eq:lipschitz_subst}
\end{equation}
Taking expectation over $(o,c)$ drawn from the scaffold (spatial) distribution and using
Assumption~\ref{ass:recoverability_aligned}, we get
\begin{align}
\mathbb E_{\mathrm{sp}}\!\left[\big|\ell(g,C;\theta_{\mathrm{spatial}}) - \ell(\widehat g,C;\theta_{\mathrm{spatial}})\big|\right]
&\le L_g \,\mathbb E_{\mathrm{sp}}\!\left[\|g-\widehat g\|\right] \nonumber\\
&\le L_g\,\varepsilon_g.
\label{eq:recover_bound_sp}
\end{align}

To relate the target semantic expectation $\mathbb E_{(g,c)\sim P_{G,C}^{(m)}}[\ell(g,c;\theta_{\mathrm{spatial}})]$
to scaffold quantities, we add and subtract the recovered-geometry loss and use triangle inequality:
\begin{align}
\mathbb E_{(g,c)\sim P_{G,C}^{(m)}}[\ell(g,c;\theta_{\mathrm{spatial}})]
&\le \mathbb E_{(g,c)\sim P_{G,C}^{(m)}}[\ell(\widehat g,c;\theta_{\mathrm{spatial}})]
+ \mathbb E_{(g,c)\sim P_{G,C}^{(m)}}\left[\big|\ell(g,c;\theta_{\mathrm{spatial}})-\ell(\widehat g,c;\theta_{\mathrm{spatial}})\big|\right].
\label{eq:add_sub_hatg}
\end{align}
The second term is a geometry-recovery error term under target distribution.
Then, we upper-bound it by a constant times $\varepsilon_g$ and absorb any
target-vs-scaffold mismatch into the constant.
Concretely, there exists a constant $C_m^{(1)}$ proportional to $L_g$ such that
\begin{equation}
\mathbb E_{(g,C)\sim P_{G,C}^{(m)}}\!\left[\big|\ell(g,C;\theta_{\mathrm{spatial}})-\ell(\widehat g,C;\theta_{\mathrm{spatial}})\big|\right]
\;\le\; C_m^{(1)}\,\varepsilon_g.
\label{eq:recover_bound_target}
\end{equation}

It remains to relate
$\mathbb E_{(g,c)\sim P_{G,C}^{(m)}}[\ell(\widehat g,c;\theta_{\mathrm{spatial}})]$
to the corresponding expectation under the scaffold distribution.
Let $\mathcal F_m$ denote the function class induced by the learned system at $\theta_{\mathrm{spatial}}$.
By Assumption~\ref{ass:shift_aligned}, we have
\[
\mathrm{Disc}\big(P_G^{(m)}, P_G^{(\mathrm{sp})}\big)\le \delta_m.
\]
By the defining property of $\mathrm{Disc}$ as an IPM-like discrepancy over $\mathcal F_m$,
there exists a constant $C_m^{(2)}$ such that
\begin{equation}
\Big|
\mathbb E_{g\sim P_G^{(m)}}[f(g)] \;-\; \mathbb E_{g\sim P_G^{(\mathrm{sp})}}[f(g)]
\Big|
\le C_m^{(2)}\delta_m,
\qquad \forall f\in\mathcal F_m.
\label{eq:disc_control}
\end{equation}
Applying Eq. (\ref{eq:disc_control}) to the particular $f$ induced by $\ell(\cdot,\cdot;\theta_{\mathrm{spatial}})$ yields
\begin{equation}
\mathbb E_{(g,c)\sim P_{G,C}^{(m)}}[\ell(g,c;\theta_{\mathrm{spatial}})]
\le
\mathbb E_{(g,c)\sim P_{G,C}^{(\mathrm{sp})}}[\ell(g,c;\theta_{\mathrm{spatial}})]
+ C_m^{(2)}\,\delta_m.
\label{eq:shift_bound}
\end{equation}

Under the scaffold distribution, the semantic loss upper-bounds the realized loss:
by reversing the construction in Eq. (\ref{eq:def_residual}) for the scaffold domain and absorbing the corresponding residual
into constants, we can write
\begin{equation}
|R_{\mathrm{sp}}(\theta_{\mathrm{spatial}})-\mathbb E_{(g,c)\sim P_{G,C}^{(\mathrm{sp})}}[\ell(g,c;\theta_{\mathrm{spatial}})]|
\le  \varepsilon_{\mathrm{sp}},
\label{eq:semantic_to_rsp}
\end{equation}
where $\varepsilon_{\mathrm{sp}}$ is a typically small approximation term for the scaffold domain.
Since the theorem statement allows an aggregate residual term for morphology-dependent effects, we absorb
$\varepsilon_{\mathrm{sp}}$ into $\varepsilon_m$ without loss of generality.

From Eq. (\ref{eq:rm_to_semantic}) and Eq. (\ref{eq:shift_bound}),
\begin{align}
R_m(\theta_{\mathrm{spatial}})
&\le \mathbb E_{m}[\ell(g,C;\theta_{\mathrm{spatial}})] + \varepsilon_m \\
&\le \mathbb E_{\mathrm{sp}}[\ell(g,C;\theta_{\mathrm{spatial}})] + C_m^{(2)}\delta_m + \varepsilon_m.
\end{align}
Now add and subtract $\ell(\widehat g,C;\theta_{\mathrm{spatial}})$ under the scaffold distribution:
\begin{align}
\mathbb E_{\mathrm{sp}}[\ell(g,C;\theta_{\mathrm{spatial}})]
&\le \mathbb E_{\mathrm{sp}}[\ell(\widehat g,C;\theta_{\mathrm{spatial}})]
+ \mathbb E_{\mathrm{sp}}\big[|\ell(g,C;\theta_{\mathrm{spatial}})-\ell(\widehat g,C;\theta_{\mathrm{spatial}})|\big] \\
&\le \mathbb E_{\mathrm{sp}}[\ell(\widehat g,C;\theta_{\mathrm{spatial}})]
+ L_g\,\mathbb E_{\mathrm{sp}}\|g-\widehat g\|
\le \mathbb E_{\mathrm{sp}}[\ell(\widehat g,C;\theta_{\mathrm{spatial}})]
+ L_g\varepsilon_g,
\end{align}
where we used Assumptions~\ref{ass:lipschitz_aligned} and \ref{ass:recoverability_aligned}.
Finally, relate $\mathbb E_{\mathrm{sp}}[\ell(\widehat g,C;\theta_{\mathrm{spatial}})]$ back to $R_{\mathrm{sp}}$:
by Eq. (\ref{eq:semantic_to_rsp}) and the same Lipschitz-recoverability step,
\begin{align}
\mathbb E_{\mathrm{sp}}[\ell(\widehat g,C;\theta_{\mathrm{spatial}})]
&\le \mathbb E_{\mathrm{sp}}[\ell(g,C;\theta_{\mathrm{spatial}})] + L_g\varepsilon_g
\le R_{\mathrm{sp}}(\theta_{\mathrm{spatial}}) + \varepsilon_{\mathrm{sp}} + L_g\varepsilon_g.
\end{align}
Combining the above yields
\[
R_m(\theta_{\mathrm{spatial}})
\le R_{\mathrm{sp}}(\theta_{\mathrm{spatial}})
+ C_m^{(2)}\delta_m
+ 2L_g\varepsilon_g
+ (\varepsilon_m+\varepsilon_{\mathrm{sp}}).
\]
Absorb constants into $C_m$ and absorb $\varepsilon_{\mathrm{sp}}$ into $\varepsilon_m$ to obtain
Eq. (\ref{eq:scaffold_transfer_aligned}). This completes the proof.

\end{proof}

\begin{remark}[What the bound captures and its experimental implications]
Eq.~(\ref{eq:scaffold_transfer_aligned}) decomposes the target-morphology risk after the \textsc{Scaffold} stage into a
recoverability term, a geometric-shift term, and a morphology-specific residual.
The term $L_g\varepsilon_g$ isolates the effect of geometric recoverability: as spatial training reduces the decoding
error $\varepsilon_g$ of the shared scaffold $g$, the target risk decreases proportionally to the sensitivity constant
$L_g$. This provides a direct theoretical explanation for the empirical trend that stronger geometry-aware pretraining
signals and improved 3D supervision yield systematic gains in zero-shot transfer across diverse embodiments.

The term $C_m\delta_m$ characterizes transfer sensitivity to geometric distribution mismatch between the target morphology
and the scaffold training distribution. For a fixed $\varepsilon_g$, differences in $\delta_m$ induce cross-embodiment
performance asymmetries, predicting strong transfer under shared spatial regimes and degraded transfer when the target
exhibits substantially different viewpoints, motion patterns, scale statistics, or scene topology. This aligns with the
observed improvement from broadening the scaffold distribution, which effectively reduces $\delta_m$.

The residual $\varepsilon_m$ aggregates factors not fully mediated by geometry, including embodiment-specific sensing
artifacts, dynamics constraints, actuation semantics, and output-space idiosyncrasies. It explains why spatial scaffold
pretraining can deliver substantial gains yet typically does not close the gap to morphology-specialized training, and it
clarifies that the subsequent \textsc{Reconcile} stage primarily targets this remaining morphology-dependent component
rather than relearning the shared spatial core.
\end{remark}

\clearpage

\subsection{Visualization on Spatial Benchmarks}

In this section, we present qualitative visualizations on spatial benchmarks to illustrate the spatial cognition of \OurMethod{}. Specifically, we include examples from six representative spatial benchmarks: VSI~\cite{vsi} (Fig.~\ref{vsi_1},~\ref{vsi_2},~\ref{vsi_3},~\ref{vsi_4}, and~\ref{vsi_5}), MMSI~\cite{mmsi} (Fig.~\ref{mmsi}), BLINK~\cite{blink} (Fig.~\ref{blink}), SAT~\cite{sat} (Fig.~\ref{sat_12} and~\ref{sat_34}), SITE (Fig.~\ref{site1} and~\ref{site2}), and MindCube~\cite{yin2025spatial} (Fig.~\ref{mindcube},~\ref{mindcube2}, and~\ref{mindcube3}).

Specifically, for examples 1-4 of VSI shown in Fig.~\ref{vsi_1},~\ref{vsi_2},~\ref{vsi_3}, and ~\ref{vsi_4}, \OurMethod{} successfully predicts the room area (32.4 m²), selects the correct navigation path (``Turn Left, Turn Right''), accurately estimates the distance between the washer and the toilet (0.8 m), and measures the longest dimension of the lamp (184 cm), jointly demonstrating its capability for scale-aware metric estimation and consistent multi-view 3D spatial reasoning across diverse indoor scenarios. However, in Example 5, the model predicts 2 suitcases instead of the ground-truth 3, which likely stems from the ambiguous everyday definition of the blue object, an item that can reasonably be interpreted as either a suitcase or a backpack, showing the vague categorization rather than limitation in spatial comprehension.

For Examples 1 and 2 of MMSI shown in Fig.~\ref{mmsi}, \OurMethod{} correctly infers that the dining table is located to the southwest of the stairs in Example 1, demonstrating object-to-object directional comprehension. However, in Example 2, it predicts ``front right'' instead of the ground-truth ``front left'', suggesting that viewpoint transformation remains challenging in complex indoor layouts.

For Examples 1-4 of BLINK shown in Fig.~\ref{blink}, \OurMethod{} correctly identifies closer objects from monocular depth cues, resolves cross-scene spatial relations, verifies object presence under partial occlusion, and determines whether the dog is on the motorcycle, jointly demonstrating its capability for relative depth estimation, and visually grounded verification in diverse real-world scenarios.

For examples 1 and 2 of SAT shown in Fig.~\ref{sat_12}, \OurMethod{} correctly detects object displacement across frames, identifying that the chair moves left and away from the camera and that the bottle moves left and towards the camera, demonstrating its capability for temporal spatial observations. For Examples 3 and 4 of SAT shown in Fig.~\ref{sat_34}, \OurMethod{} correctly infers the rotation of viewpoint, demonstrating its capability for camera motion understanding.

For examples 1 and 2 of SITE shown in Fig.~\ref{site1} and~\ref{site2}, \OurMethod{} correctly selects the images that match the given spatial descriptions (``bed in the bottom-left corner'' and ``motorbike at the top-right''), demonstrating its capability for understanding visual layouts between multiple misleading samples.

For examples 1–3 of MindCube shown in Fig.~\ref{mindcube},~\ref{mindcube2},
and~\ref{mindcube3}, \OurMethod{} correctly infers that the large recycling bin is behind the small one from a different viewpoint, predicts that turning right and moving forward would increase the distance to the vending machine, and identifies the white table as the nearest object to the left when repositioned at the sofa’s location, jointly demonstrating its capability for relative size understanding, occlusion-aware spatial observation, and viewpoint-consistent comprehension.

\subsection{Visualization on AD Benchmarks}
In this section, we present qualitative visualizations on autonomous driving benchmarks to illustrate the driving-related perception and reasoning capabilities of \OurMethod{}. Specifically, we include examples from four representative benchmarks: NuscenesQA~\cite{nuscenesqa} (Fig.~\ref{nuscenesqa1},~\ref{nuscenesqa2}, and~\ref{nuscenesqa3}), NuPlanQA~\cite{park2025nuplanqa} (Fig.~\ref{nuplan1}), MAPLM~\cite{maplm}, and LingoQA~\cite{lingoqa}, covering multi-view scene understanding, dynamic object reasoning, planning-aware comprehension, and language-grounded driving semantics.

Specifically, for examples 1-3 of NuscenesQA shown in Fig.~\ref{nuscenesqa1},~\ref{nuscenesqa2}, and~\ref{nuscenesqa3}, \OurMethod{} correctly identifies that no moving bicycles are visible, localizes the pedestrian positioned at the back-right of the bus, and counts 10 objects sharing the same status as the truck across multiple views. These examples demonstrate its capability for cross-view dynamic object recognition, spatially constrained localization, and counting in complex urban driving scenarios.

For example 1 of NuPlanQA shown in Fig.~\ref{nuplan1}, \OurMethod{} correctly identifies that a red traffic light is visible ahead by integrating multi-view scene observations with temporal vehicle state signals (velocity and steering angles over the past 1.5 seconds). This example illustrates its capability for kinematics understanding, and traffic signal perception in complex urban scenes.

For examples 1 and 2 of MAPLM shown in Fig.~\ref{maplm1} and~\ref{maplm2}, \OurMethod{} correctly classifies the road scenes as a normal city road and an undeveloped road, respectively, by integrating LiDAR bird’s-eye-view representations with multi-view images. These examples illustrate its capability for cross-modal scene understanding, road-type discrimination, and robust environmental interpretation under varying lighting and conditions in autonomous driving scenarios.

For examples 1 and 2 of LingoQA shown in Fig.~\ref{lingoqa1} and~\ref{lingoqa2}, \OurMethod{} correctly identifies a temporary traffic light displaying green across temporal frames and infers that the ego-vehicle is accelerating. However, although the output in example 2 is correct, the generated justification does not fully capture the most critical dynamic factor that the leading truck is moving off, indicating that while the model understands the scene, its explanation generation could benefit from stronger alignment with causal reasoning data.  These examples illustrate \OurMethod{}'s capability for temporal traffic signal recognition, causal scene interpretation, and basic justification for explainable autonomous driving decision understanding.

\subsection{Visualization on UAV Benchmarks}
In this section, we present qualitative visualizations on low-altitude UAV benchmarks to illustrate the aerial perception and spatial reasoning capability of \OurMethod{}. Specifically, we include examples from three representative benchmarks: UrbanVideo-Bench~\cite{zhao2025urbanvideo} (Fig.~\ref{urbanvideo}), AirCopBench~\cite{aircopbench} (Fig.~\ref{aircop12} and~\ref{aircop34}), and AVI-Math~\cite{zhou2025multimodal} (Fig.~\ref{avimath14} and~\ref{avimath57}).

For examples 1 and 2 of UrbanVideo-Bench shown in Fig.~\ref{urbanvideo}, \OurMethod{} correctly reconstructs the upward-then-horizontal flight trajectory toward the largest nearby intersection and identifies the biggest crossroads within the field of view at the current location. These examples demonstrate its capability for large-scale landmark recognition, and viewpoint-consistent spatial comprehension from multi-frame drone observations.

For examples 1 and 2 of AirCopBench shown in Fig.~\ref{aircop12}, \OurMethod{} correctly determines that the white van is in the same lane as the white truck but ahead of it, and identifies the white car approaching the intersection from the bottom as the vehicle directly opposite the yellow bus. For examples 3 and 4 of AirCopBench shown in Fig.~\ref{aircop34}, \OurMethod{} correctly infers that there are no pedestrians on the zebra crossing and identifies the red car as being in a parallel lane to the white van. These examples illustrate its capability for aerial intersection monitoring, traffic participant interaction analysis, and lane-level relational reasoning in complex urban road topologies.

For Examples 1–4 of AVI-Math shown in Fig.~\ref{avimath14}, \OurMethod{} correctly confirms the presence of Monica’s car and identifies two vehicles requiring regular refueling, while more notably, in Examples 3 and 4, it predicts the center coordinates of the vehicle within a specified ground-plane camera coordinate system. In Example 3, it accurately outputs the correct coordinate ([-23, -7, 49]), and in Example 4, although the predicted value ([0, 9.54, 30.28]) is not exactly identical to the ground truth ([0, 9, 32]), it remains geometrically consistent under large-resolution (4000 $\times$ 2250) aerial images. These examples highlight its capability for structured numerical reasoning under explicit coordinate systems, geometry-aware point localization, and complex spatial computation from aerial viewpoints. For Examples 5–7 of AVI-Math shown in Fig.~\ref{avimath57}, \OurMethod{} correctly estimates the drone’s altitude (30 m) and reasonably approximates the time required to reach the closest vehicle based on the given flight speed, while also predicting pixel-level coordinates of a specified 3D projection point. Although the predicted image coordinates are not exactly identical to the ground truth, they remain close in magnitude and direction, reflecting geometrically coherent reasoning under the defined camera model. These examples further demonstrate its capability for metric estimation, coordinate transformation, and geometry-aware understanding from aerial imagery.

\subsection{Visualization on Embodied Benchmarks}

In this section, we present qualitative visualizations on embodied benchmarks to illustrate the embodied perception, action prediction, and interaction understanding capabilities of \OurMethod{}. We include examples from three representative benchmarks: RoboVQA~\cite{robovqa} (Fig.~\ref{robovqa1},~\ref{robovqa2}, and~\ref{robovqa3}), OpenEQA~\cite{majumdar2024openeqa} (Fig.~\ref{openeqa1} and~\ref{openeqa2}), and EgoPlan-Bench2~\cite{egoplanbench} (Fig.~\ref{egoplan1} and~\ref{egoplan2}), covering egocentric manipulation understanding, embodied question answering, and goal-oriented planning in interactive environments.

For examples 1–3 of RoboVQA shown in Fig.~\ref{robovqa1},~\ref{robovqa2}, and~\ref{robovqa3}, \OurMethod{} correctly infers that the orange is placed into the bowl, recognizes that the feasible next action involves moving the chair away from the table (despite generating a more general description ``adjust the chair''), and identifies that the pen is placed into the stand. These examples demonstrate its capability for temporal action understanding, object affordance recognition, and fine-grained manipulation reasoning from egocentric visual sequences in embodied scenarios.

For examples 1 and 2 of OpenEQA shown in Fig.~\ref{openeqa1} and~\ref{openeqa2}, \OurMethod{} correctly identifies the bedroom as the room containing the car by integrating multi-view indoor observations, and infers that cotton swabs are most likely stored in the sink cabinet within the bathroom. These examples illustrate its capability for room-level localization, indoor topology understanding, and commonsense reasoning about object-function associations in embodied navigation scenarios.

For examples 1 and 2 of EgoPlan-Bench2 shown in Fig.~\ref{egoplan1} and~\ref{egoplan2}, \OurMethod{} correctly infers the next required action by integrating prior video progress with the current egocentric observation, identifying ``turn off sink tap'' in the kitchen scenario and ``pick water bottle'' in the rowing scenario. These examples demonstrate its capability for long-horizon procedural reasoning, temporal context integration, and goal-directed action prediction in dynamic embodied tasks.

\begin{figure*}
     \centering
    \includegraphics[width=1\textwidth]{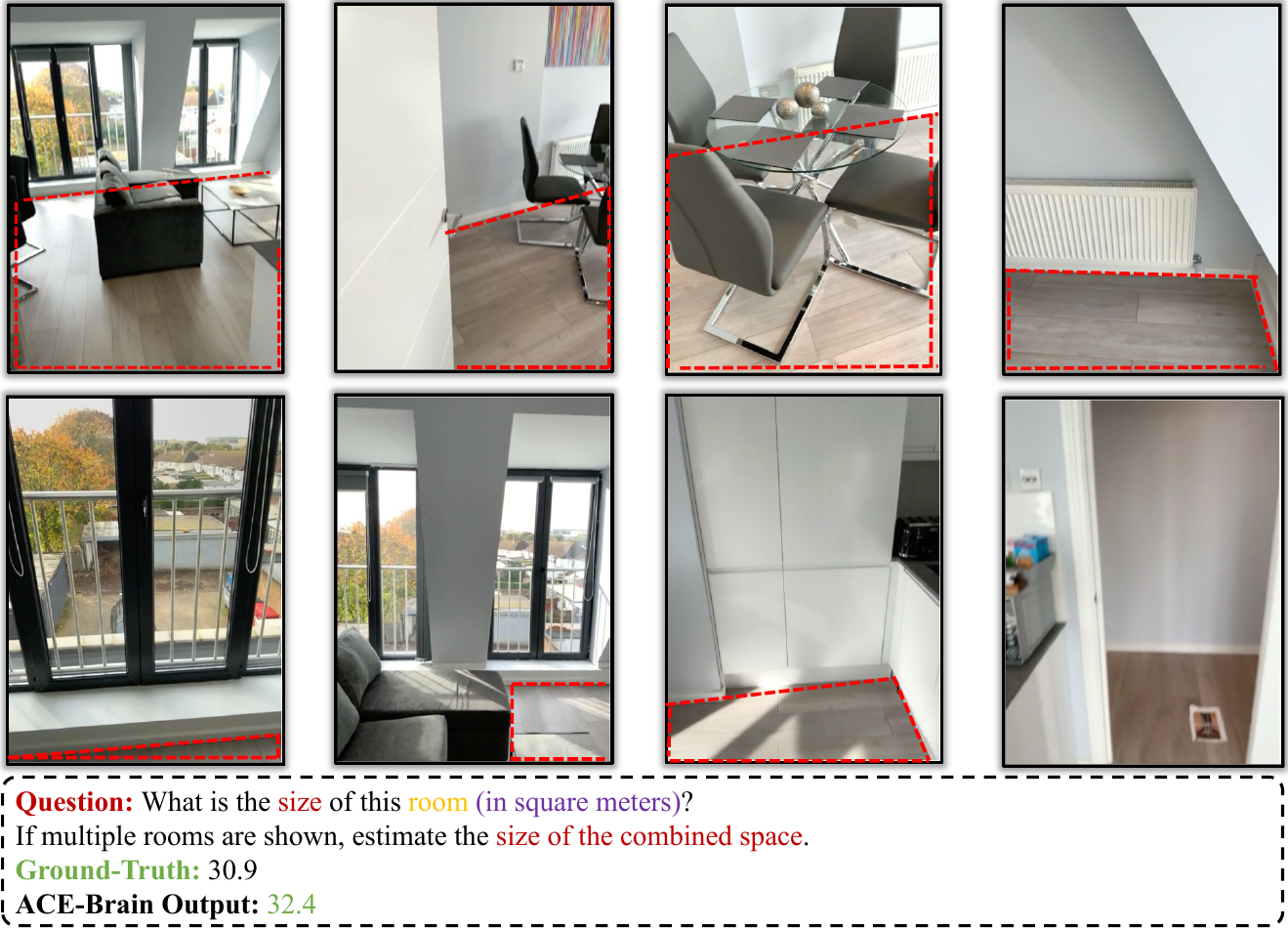}
    \caption{\textbf{Example 1 of VSI Benchmark.}}
    \label{vsi_1}
\end{figure*}
\begin{figure*}
     \centering
    \includegraphics[width=1\textwidth]{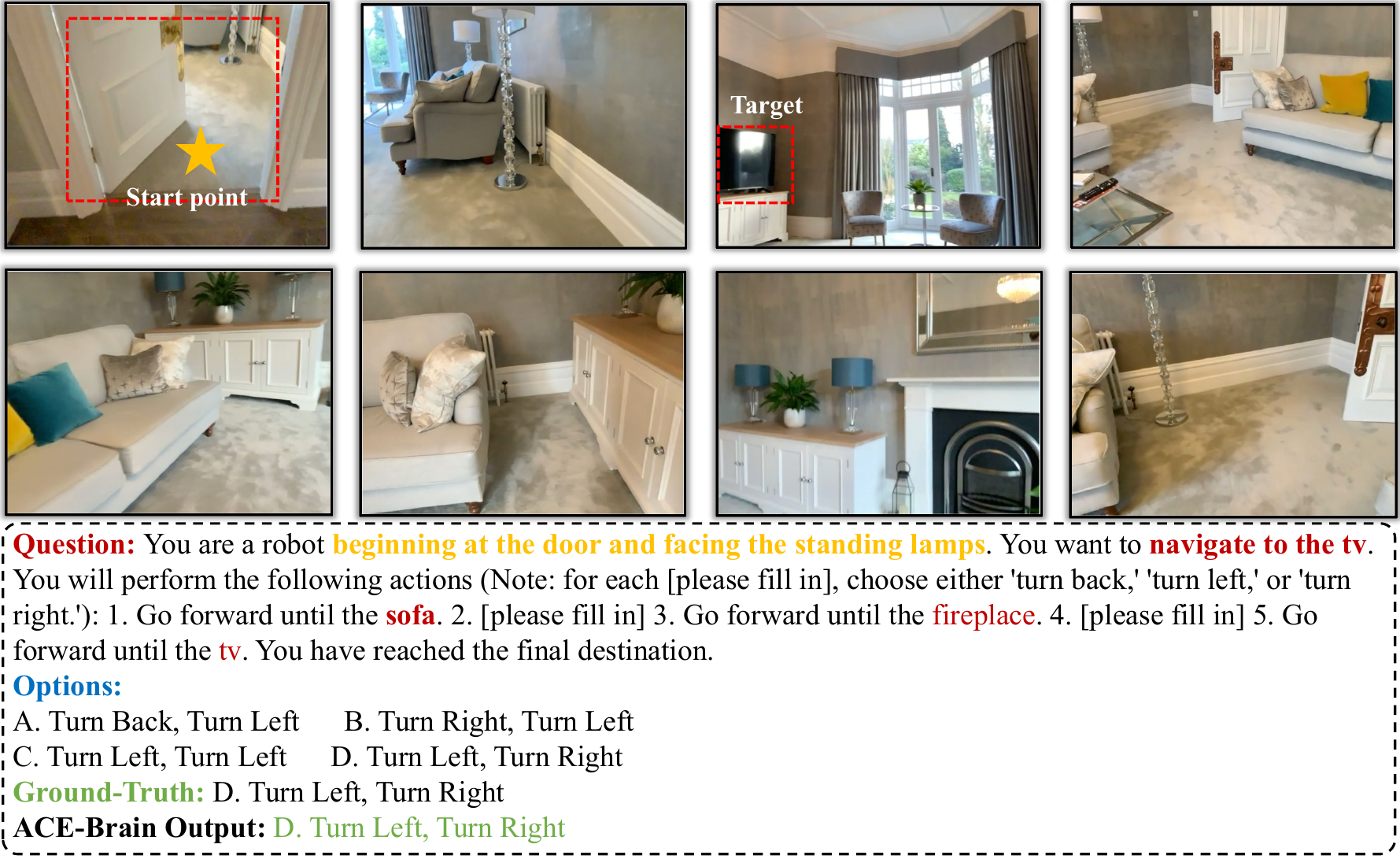}
    \caption{\textbf{Example 2 of VSI Benchmark.}}
    \label{vsi_2}
\end{figure*}
\begin{figure*}
     \centering
    \includegraphics[width=1\textwidth]{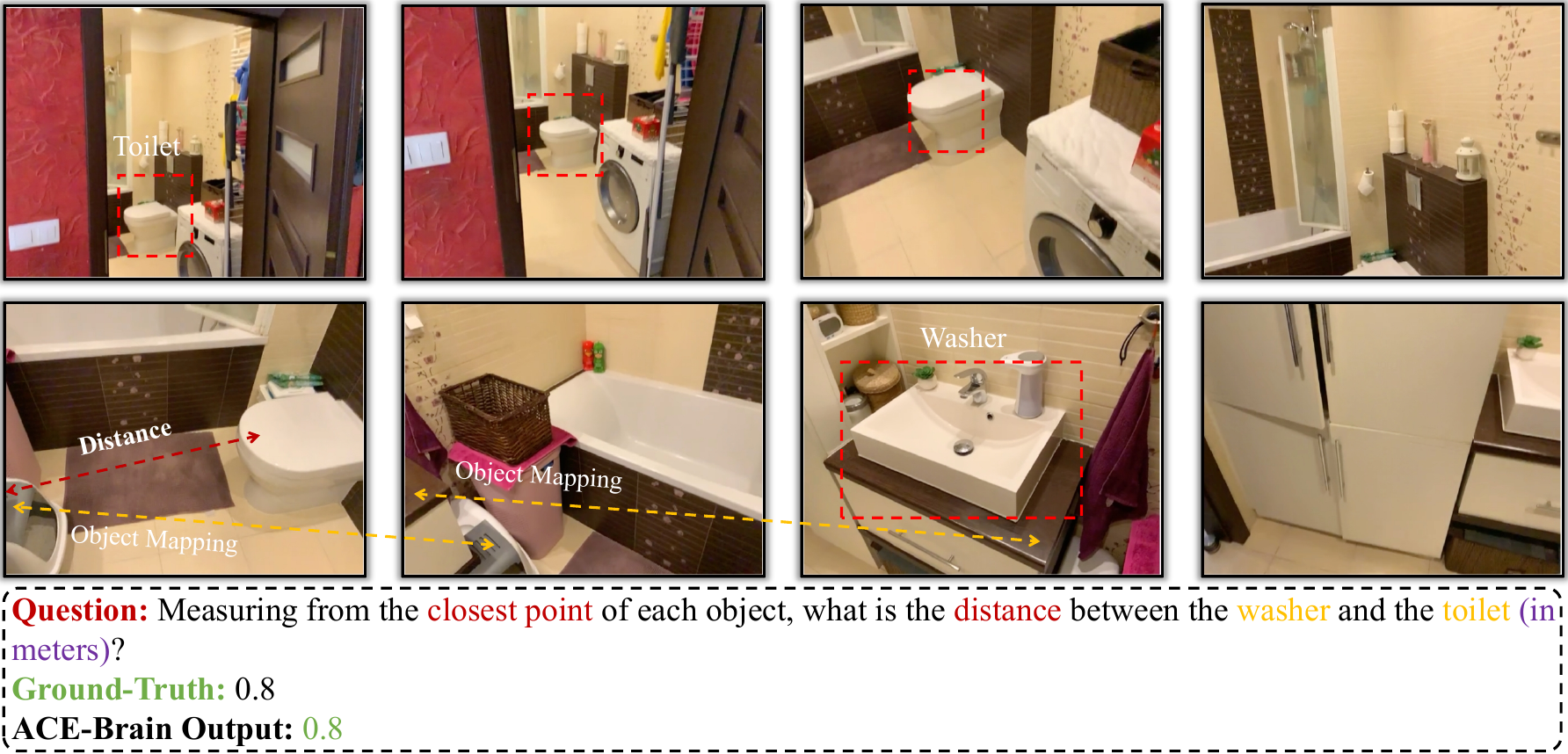}
    \caption{\textbf{Example 3 of VSI Benchmark.} }
    \label{vsi_3}
\end{figure*}
\begin{figure*}
     \centering
    \includegraphics[width=1\textwidth]{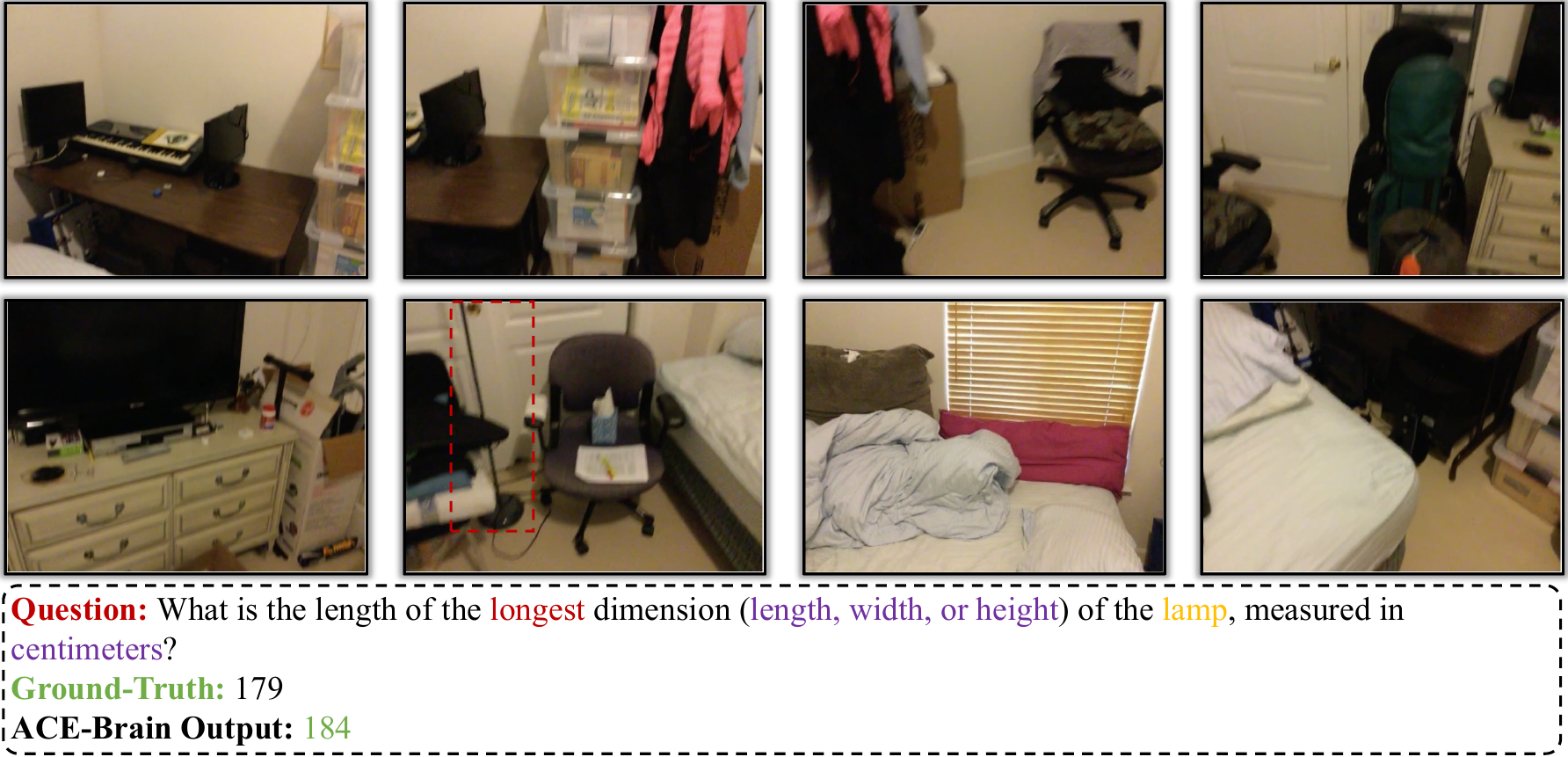}
    \caption{\textbf{Example 4 of VSI Benchmark.} }
    \label{vsi_4}
\end{figure*}
\begin{figure*}
     \centering
    \includegraphics[width=1\textwidth]{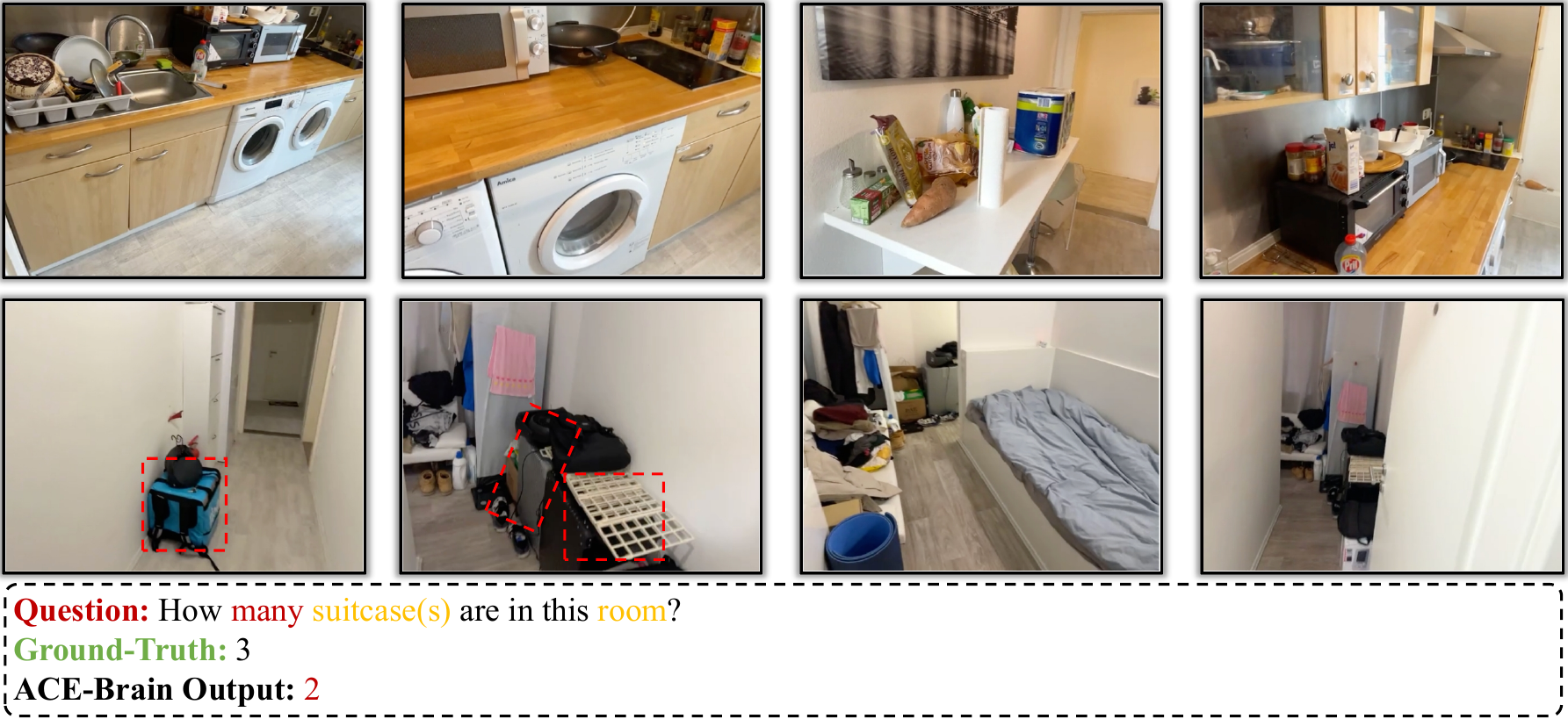}
    \caption{\textbf{Example 5 of VSI Benchmark.} }
    \label{vsi_5}
\end{figure*}
\begin{figure*}
     \centering
    \includegraphics[width=1\textwidth]{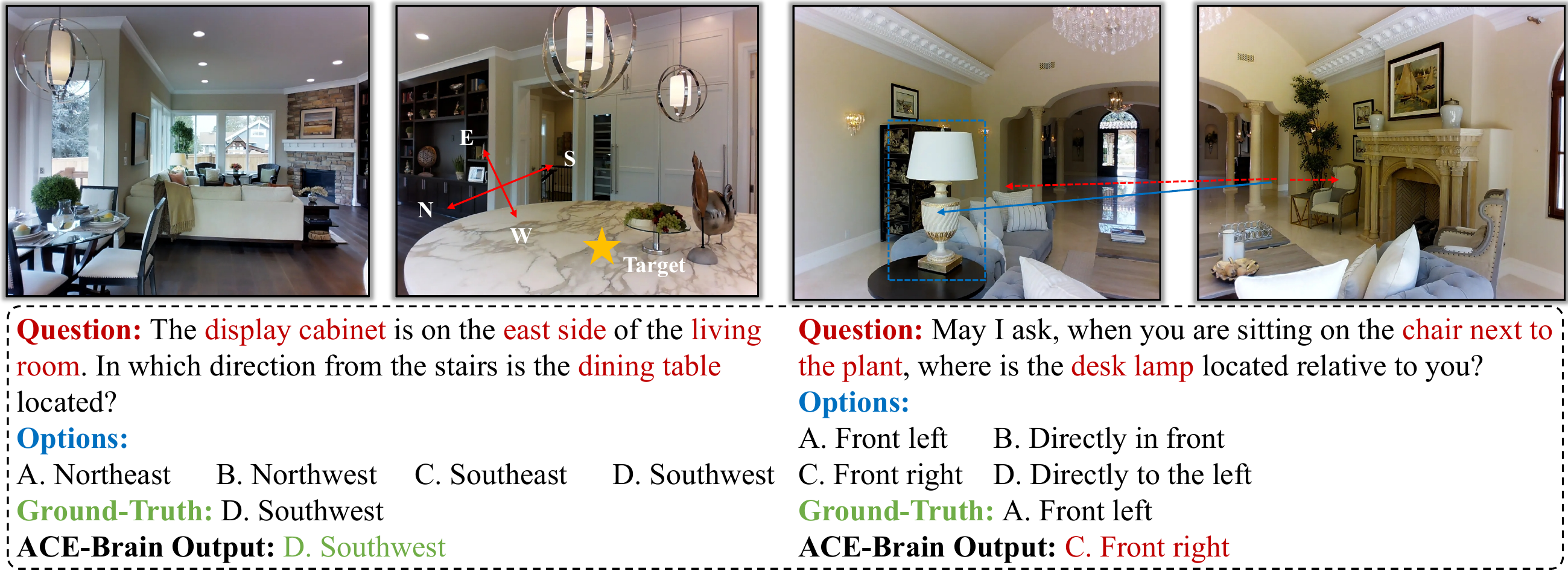}
    \caption{\textbf{Example 1-2 of MMSI Benchmark.}}
    \label{mmsi}
\end{figure*}

\begin{figure*}
     \centering
    \includegraphics[width=1\textwidth]{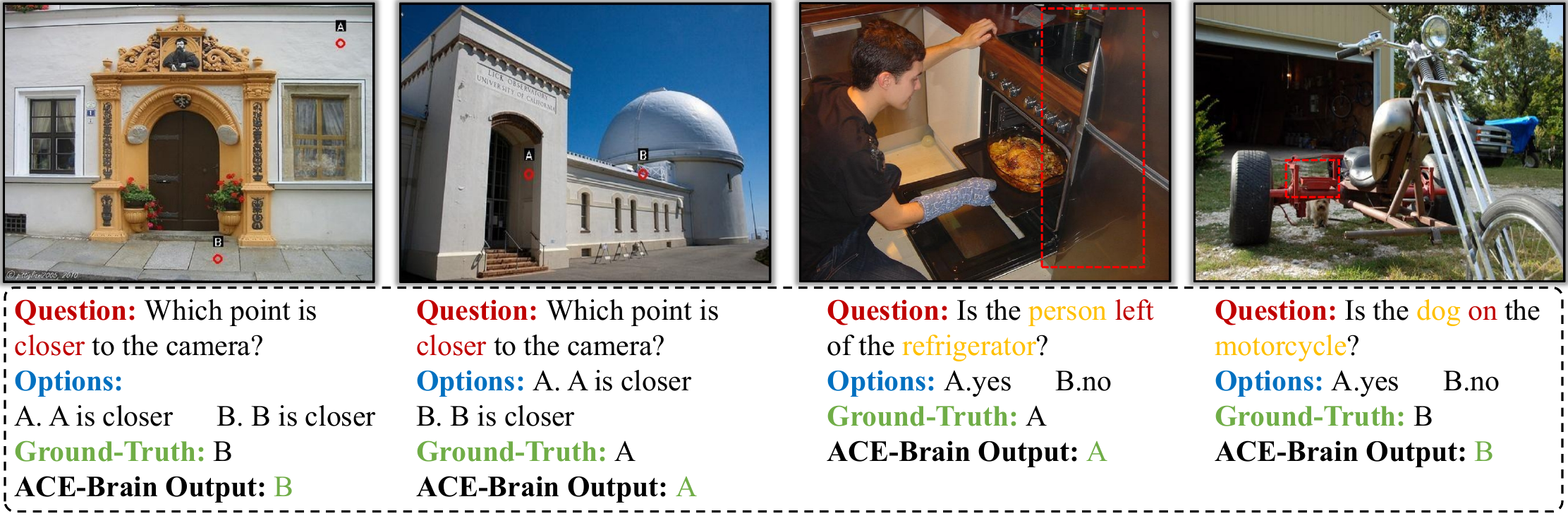}
    \caption{\textbf{Example 1-4 of BLINK Benchmark.}}
    \label{blink}
\end{figure*}

\begin{figure*}
     \centering
    \includegraphics[width=1\textwidth]{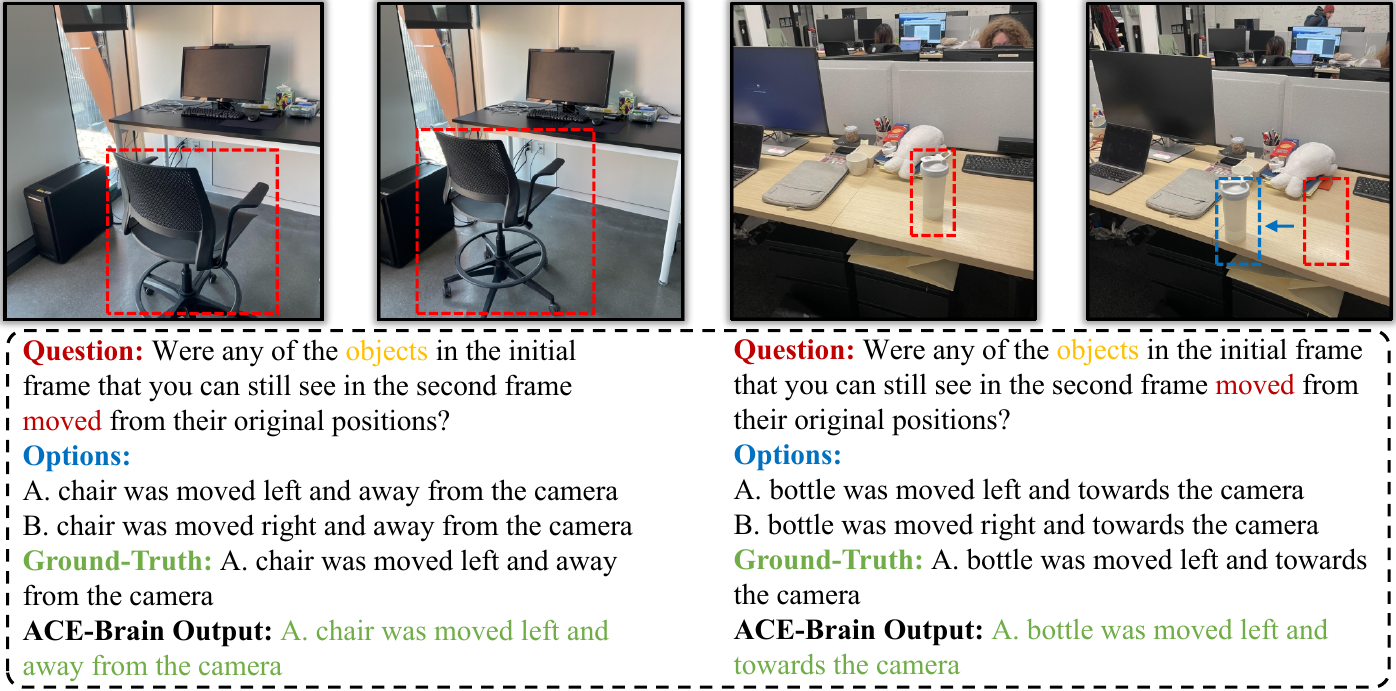}
    \caption{\textbf{Example 1-2 of SAT Benchmark.}}
    \label{sat_12}
\end{figure*}
\begin{figure*}
     \centering
    \includegraphics[width=1\textwidth]{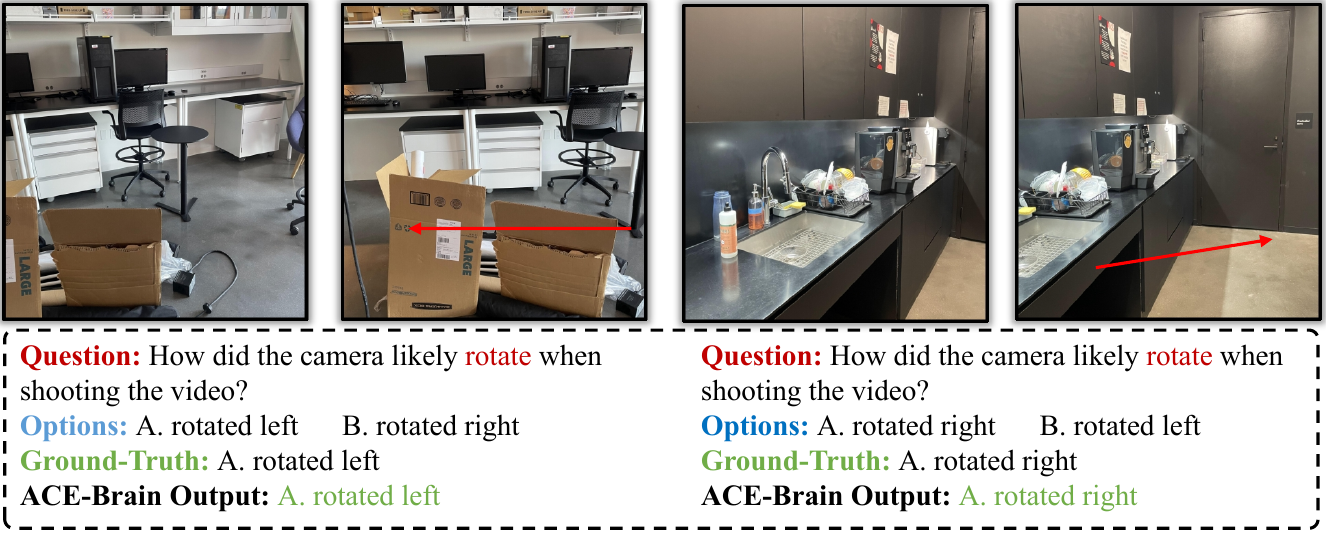}
    \caption{\textbf{Example 3-4 of SAT Benchmark.}}
    \label{sat_34}
\end{figure*}
\begin{figure*}
     \centering
    \includegraphics[width=1\textwidth]{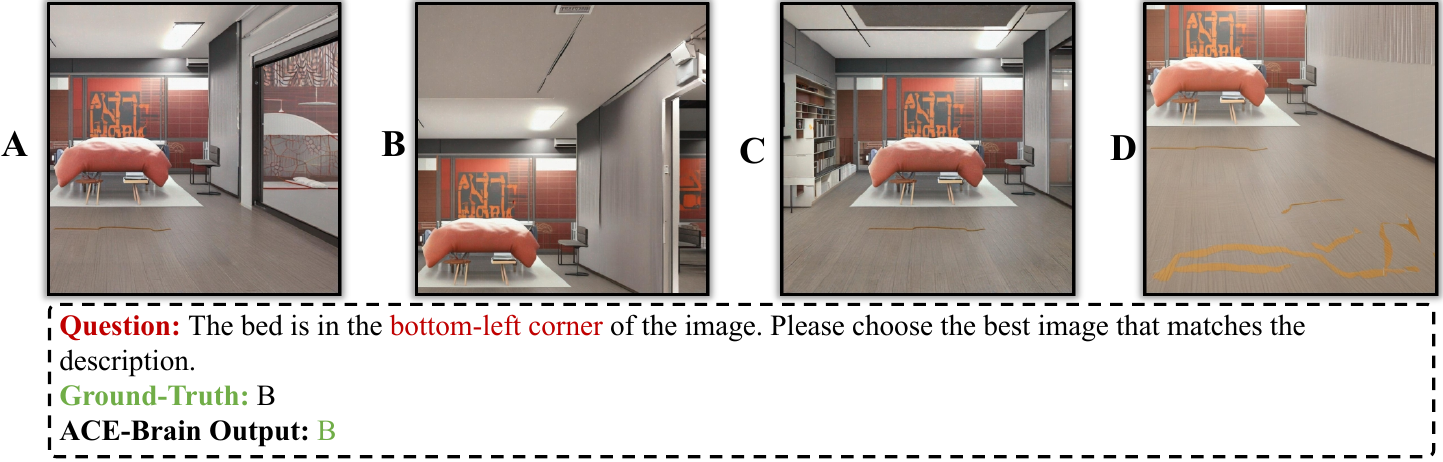}
    \caption{\textbf{Example 1 of SITE Benchmark.}}
    \label{site1}
\end{figure*}
\begin{figure*}
     \centering
    \includegraphics[width=0.99\textwidth]{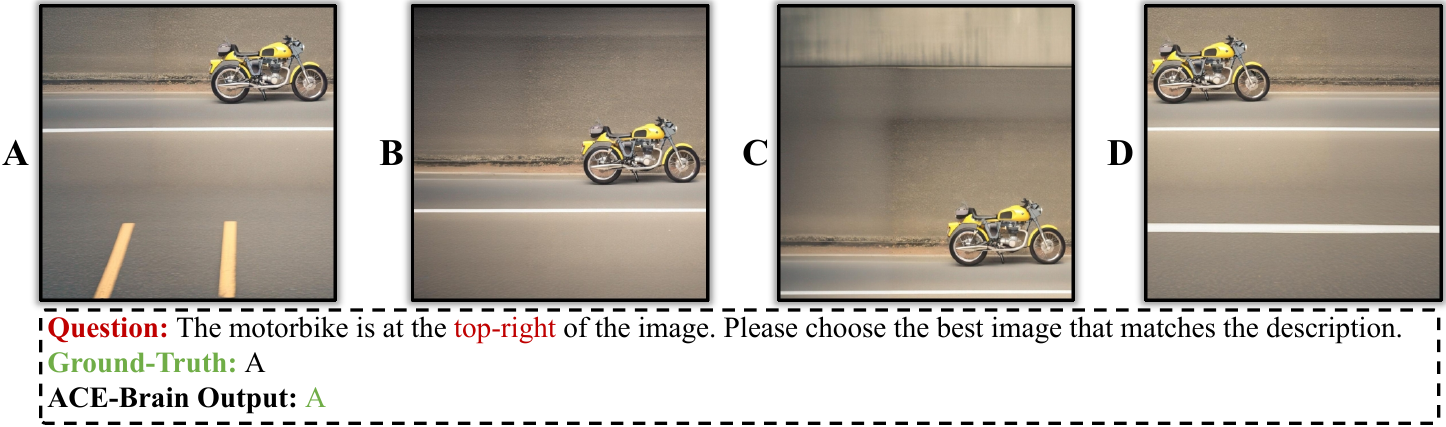}
    \caption{\textbf{Example 2 of SITE Benchmark.}}
    \label{site2}
\end{figure*}
\begin{figure*}
     \centering
    \includegraphics[width=1\textwidth]{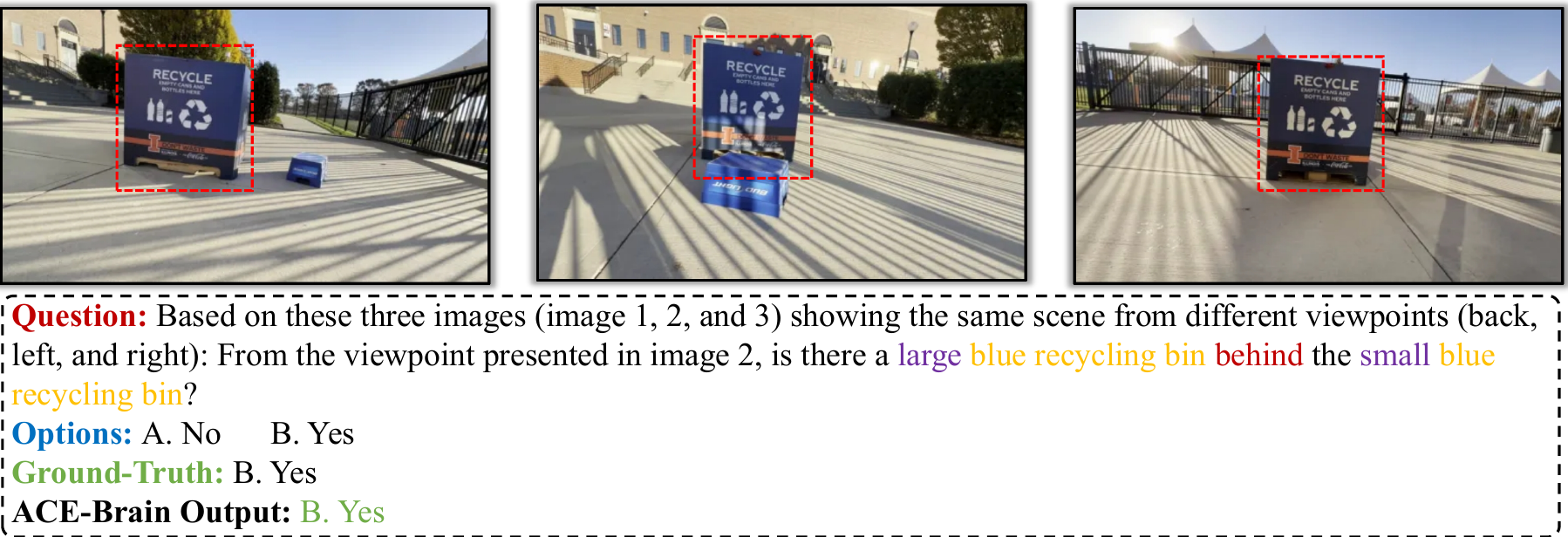}
    \caption{\textbf{Example 1 of MindCube Benchmark.}}
    \label{mindcube}
\end{figure*}
\begin{figure*}
     \centering
    \includegraphics[width=1\textwidth]{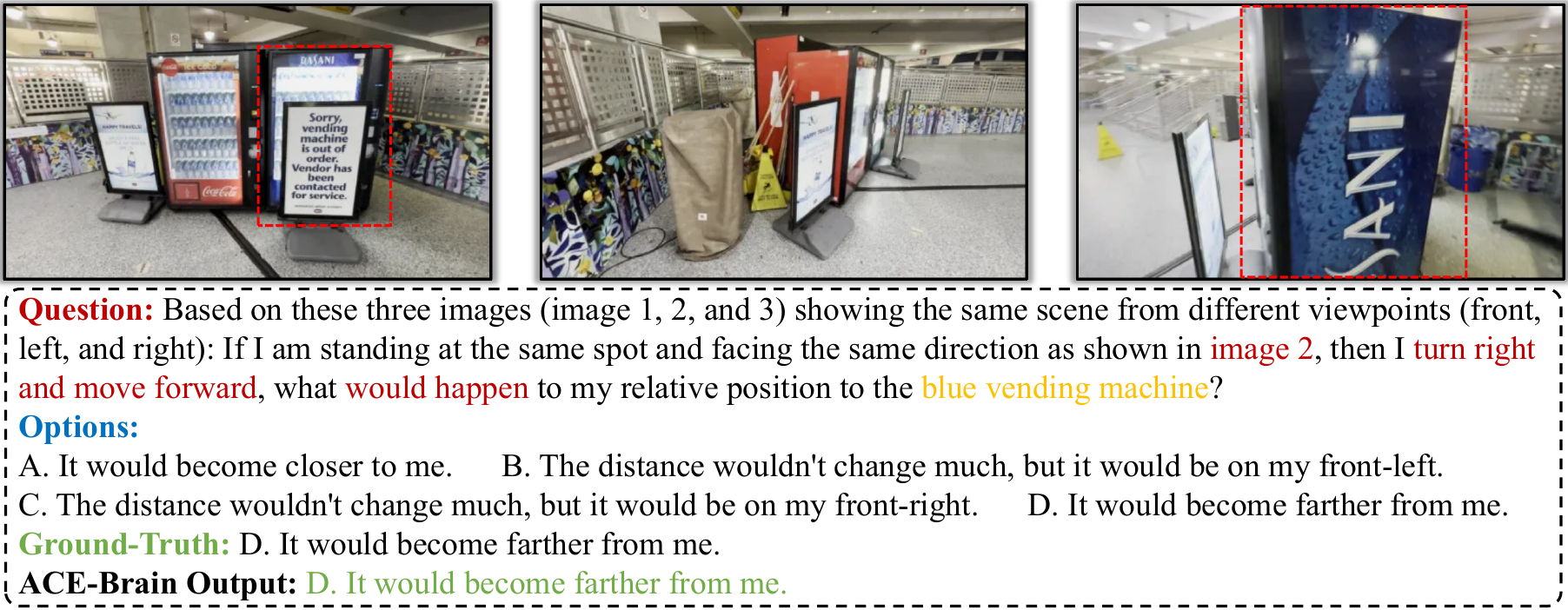}
    \caption{\textbf{Example 2 of MindCube Benchmark.}}
    \label{mindcube2}
\end{figure*}
\begin{figure*}
     \centering
    \includegraphics[width=1\textwidth]{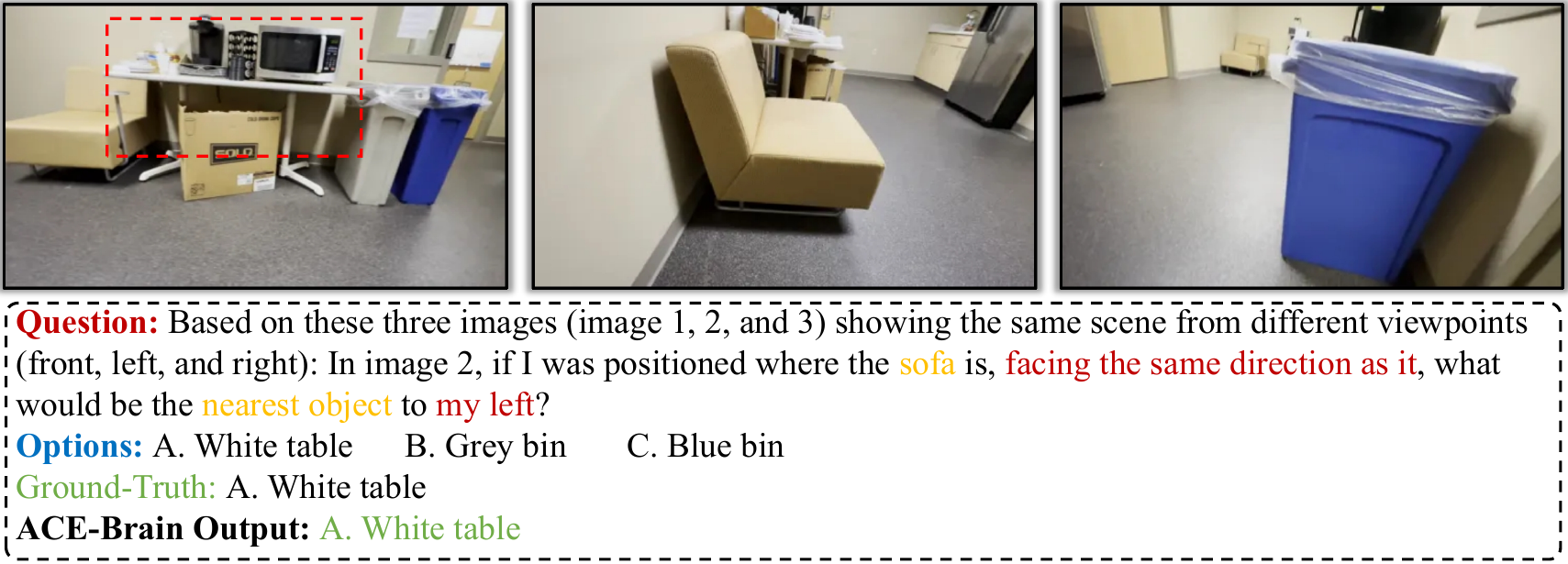}
    \caption{\textbf{Example 3 of MindCube Benchmark.}}
    \label{mindcube3}
\end{figure*}

\begin{figure*}
     \centering
    \includegraphics[width=1\textwidth]{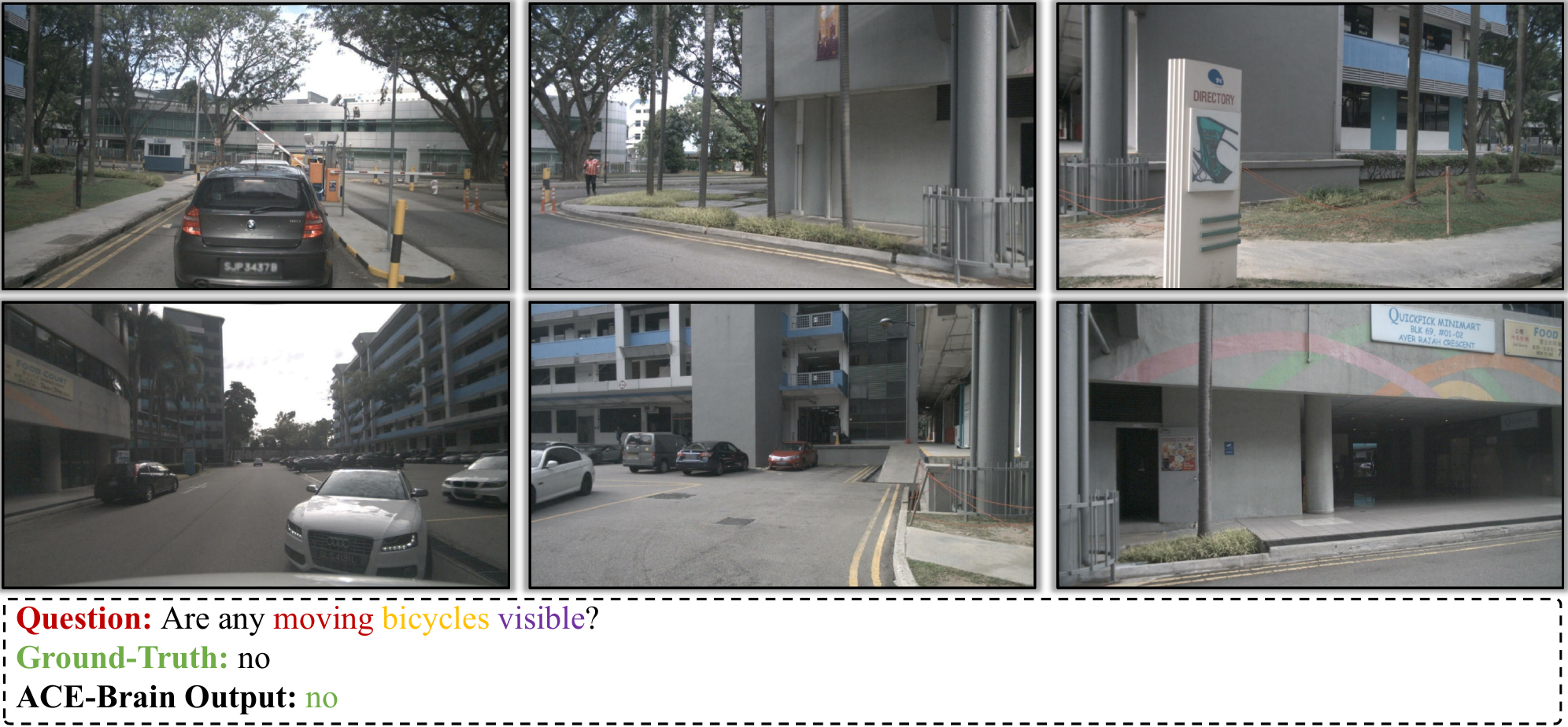}
    \caption{\textbf{Example 1 of NuscenesQA Benchmark.}}
    \label{nuscenesqa1}
\end{figure*}
\begin{figure*}
     \centering
    \includegraphics[width=1\textwidth]{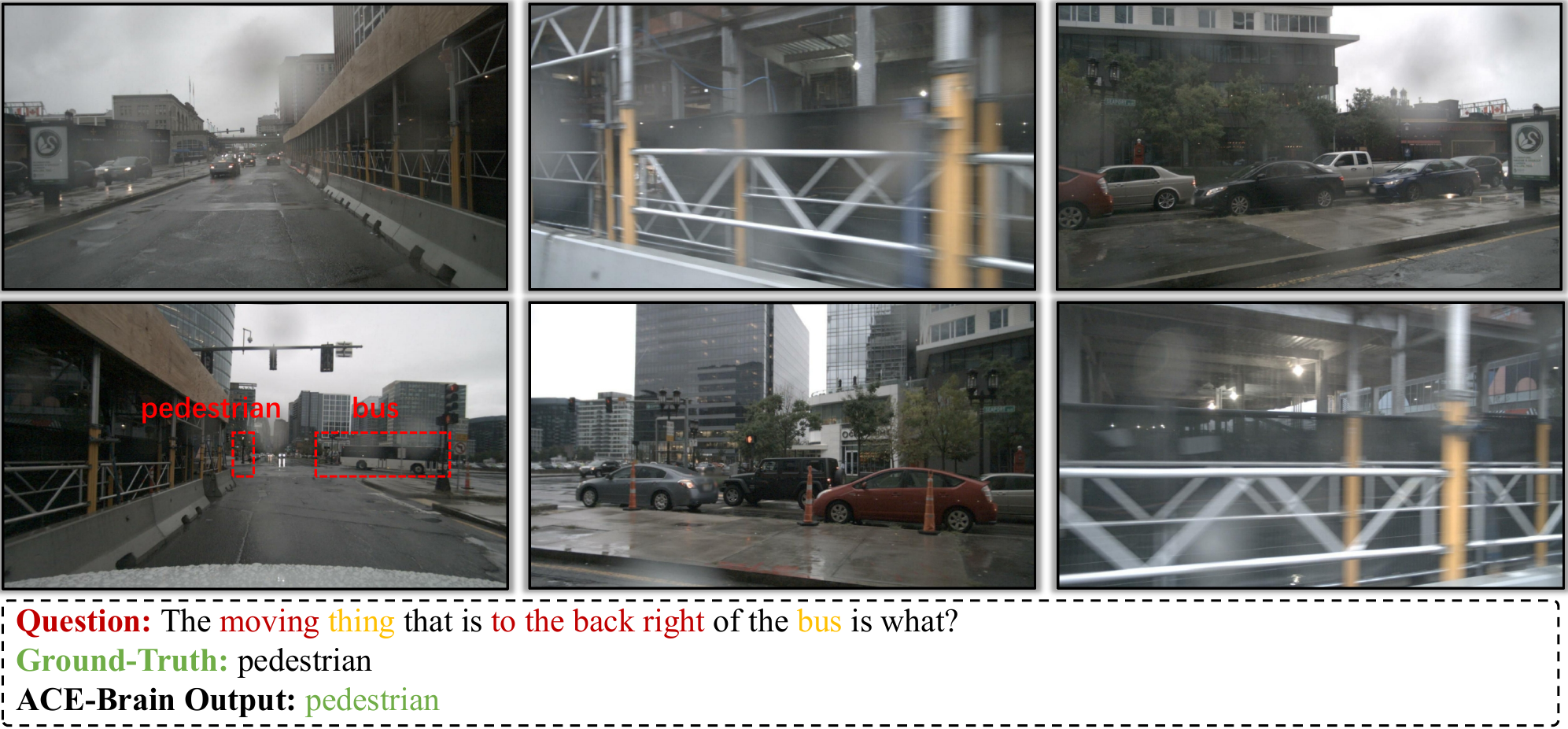}
    \caption{\textbf{Example 2 of NuscenesQA Benchmark.}}
    \label{nuscenesqa2}
\end{figure*}
\begin{figure*}
     \centering
    \includegraphics[width=1\textwidth]{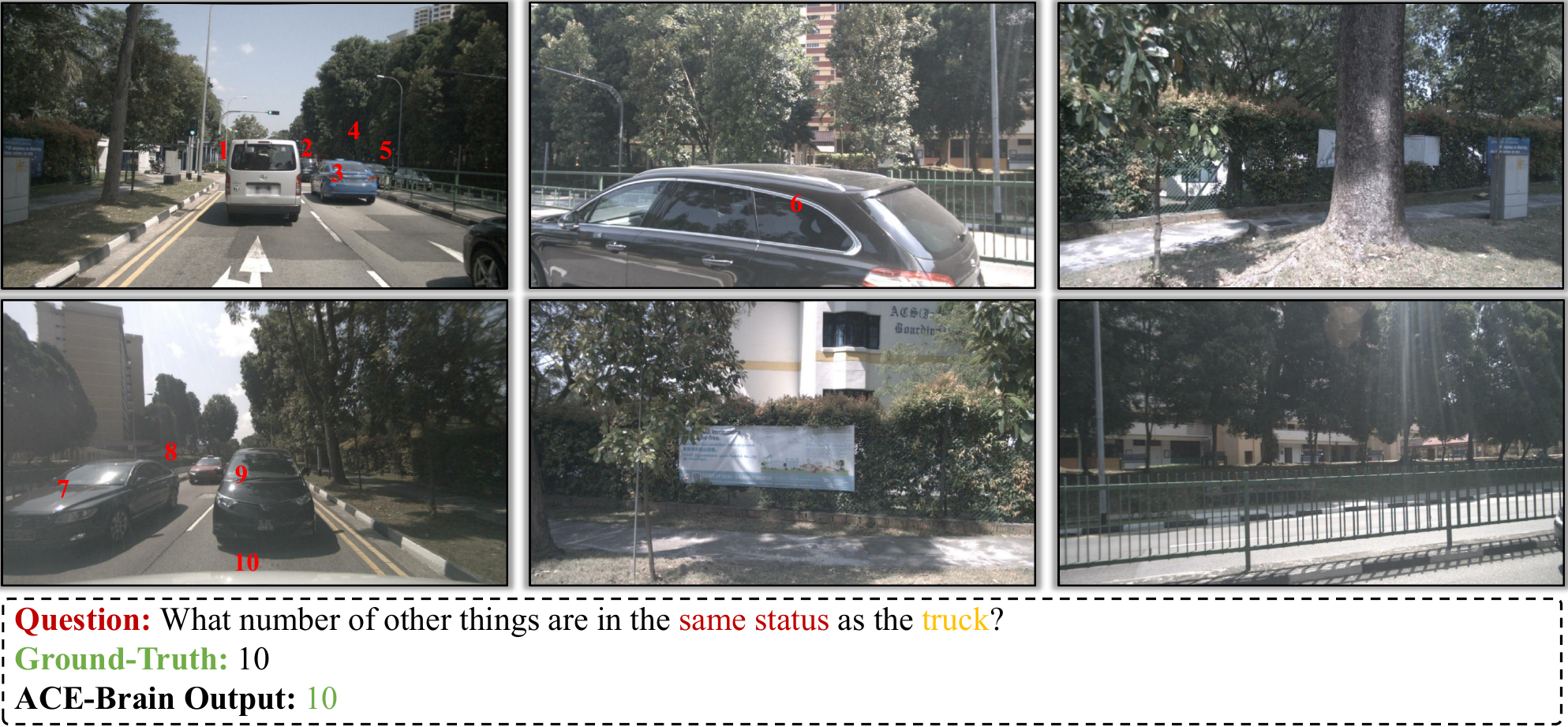}
    \caption{\textbf{Example 3 of NuscenesQA Benchmark.}}
    \label{nuscenesqa3}
\end{figure*}
\begin{figure*}
     \centering
    \includegraphics[width=1\textwidth]{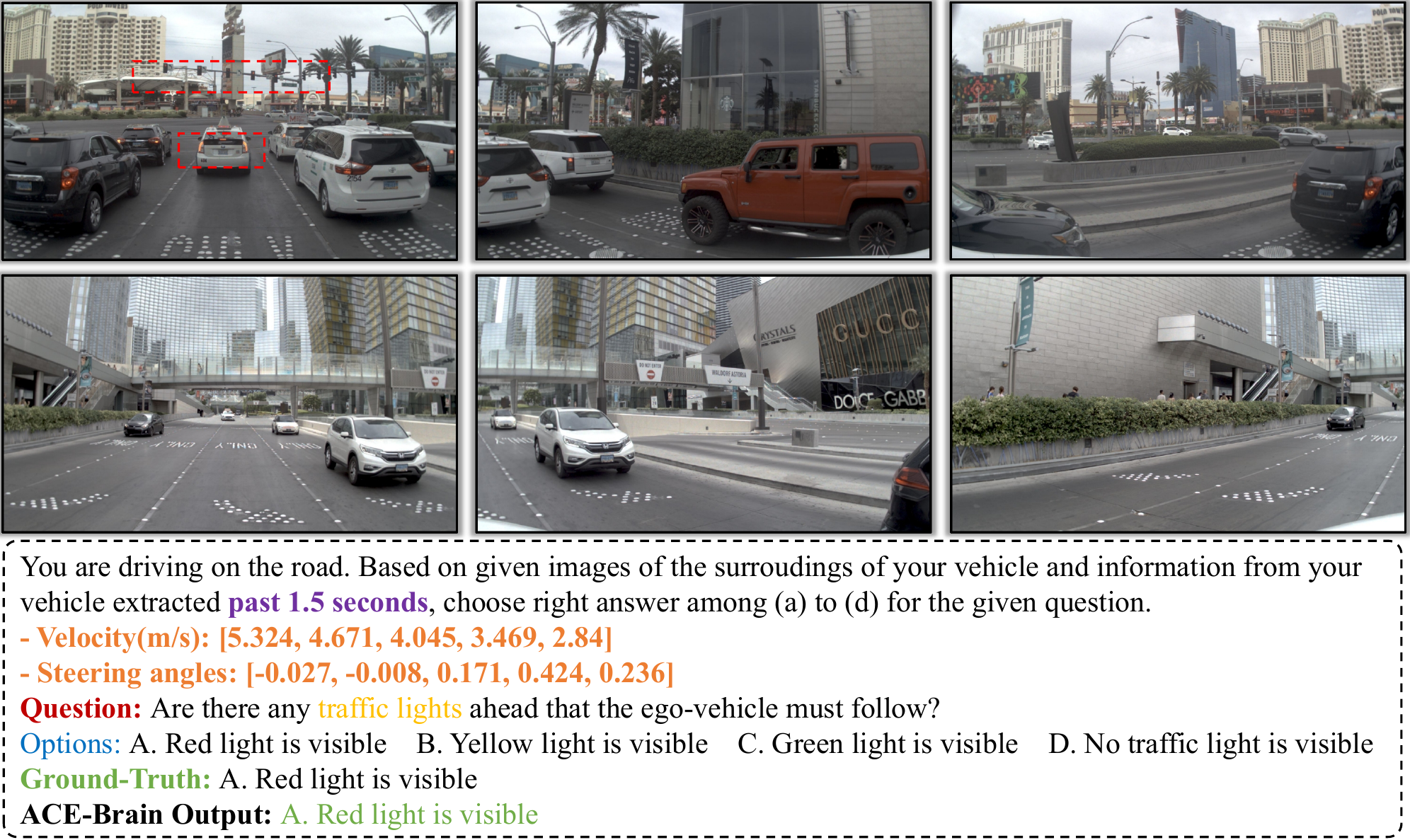}
    \caption{\textbf{Example 1 of NuPlanQA Benchmark.}}
    \label{nuplan1}
\end{figure*}

\begin{figure*}
     \centering
    \includegraphics[width=1\textwidth]{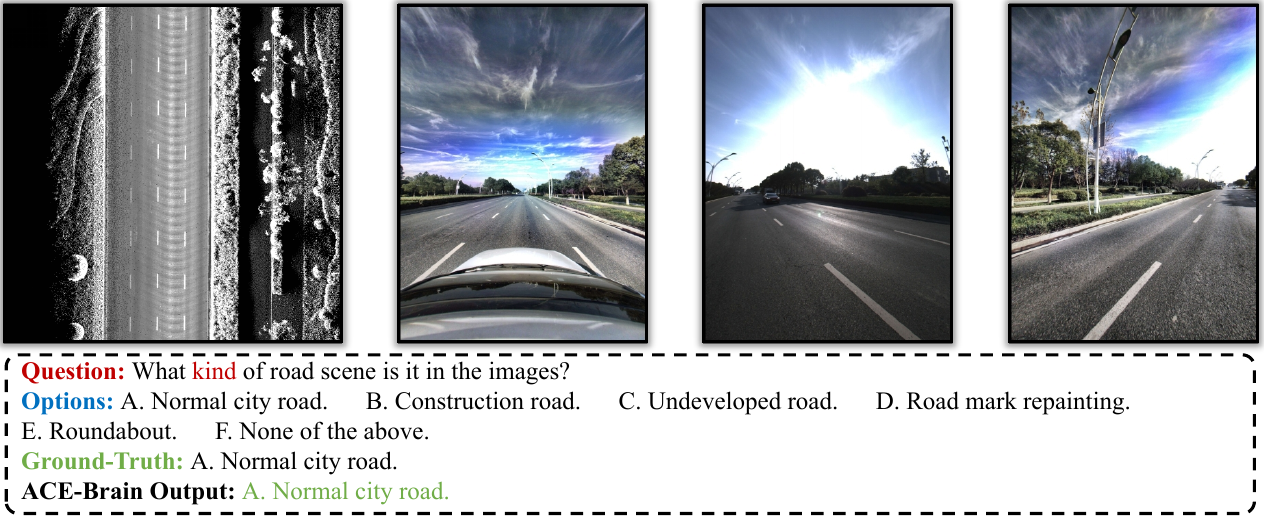}
    \caption{\textbf{Example 1 of MAPLM Benchmark.}}
    \label{maplm1}
\end{figure*}
\begin{figure*}
     \centering
    \includegraphics[width=1\textwidth]{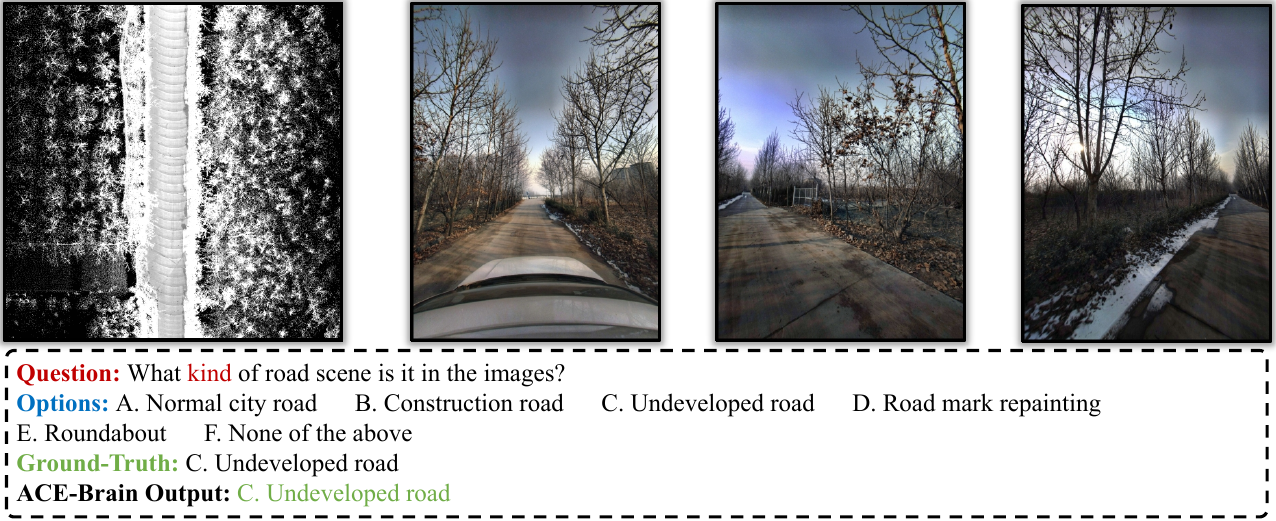}
    \caption{\textbf{Example 2 of MAPLM Benchmark.}}
    \label{maplm2}
\end{figure*}
\begin{figure*}
     \centering
    \includegraphics[width=1\textwidth]{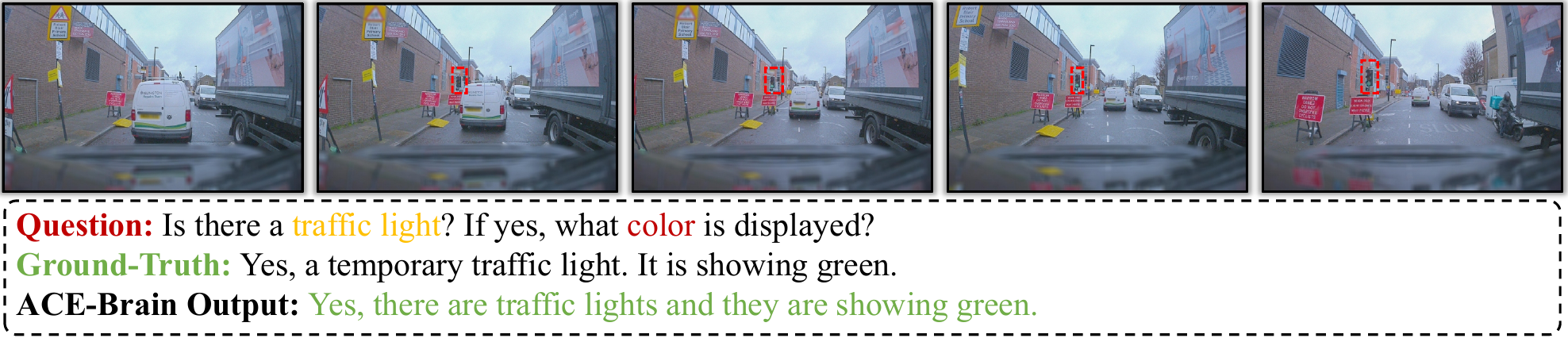}
    \caption{\textbf{Example 1 of LingoQA Benchmark.}}
    \label{lingoqa1}
\end{figure*}
\begin{figure*}
     \centering
    \includegraphics[width=1\textwidth]{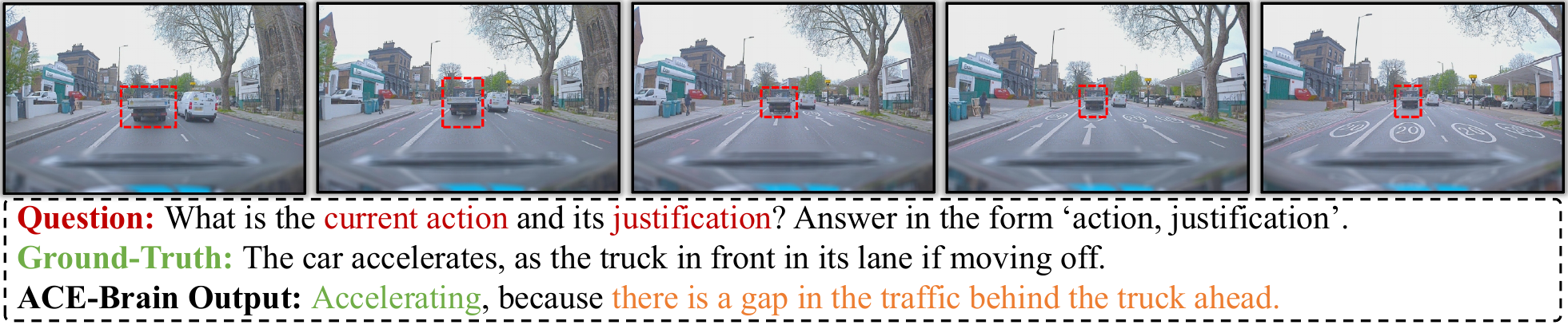}
    \caption{\textbf{Example 2 of LingoQA Benchmark.}}
    \label{lingoqa2}
\end{figure*}
\begin{figure*}
     \centering
    \includegraphics[width=1\textwidth]{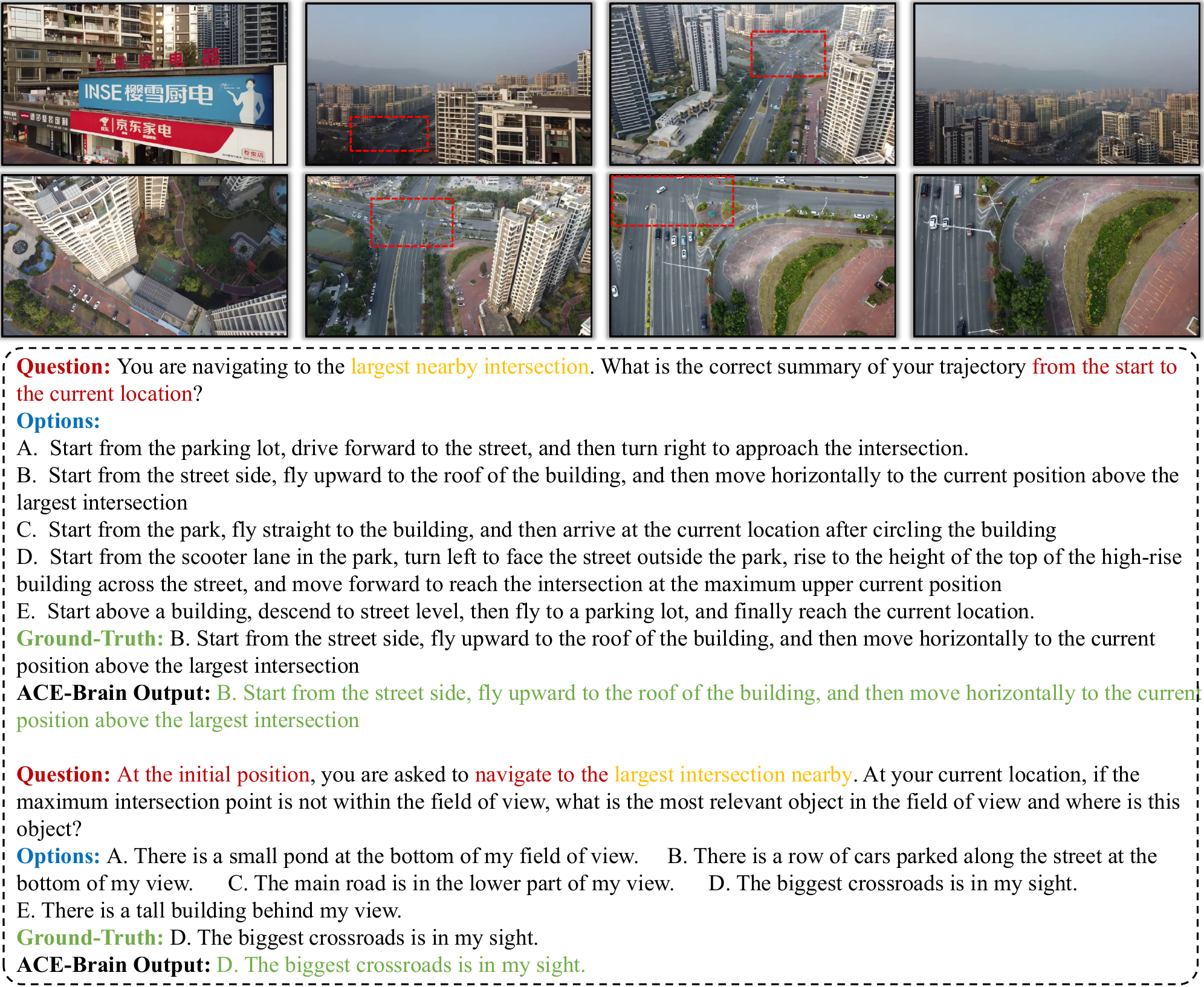}
    \caption{\textbf{Example 1 and 2 of UrbanVideo Benchmark.}}
    \label{urbanvideo}
\end{figure*}
\begin{figure*}
     \centering
    \includegraphics[width=1\textwidth]{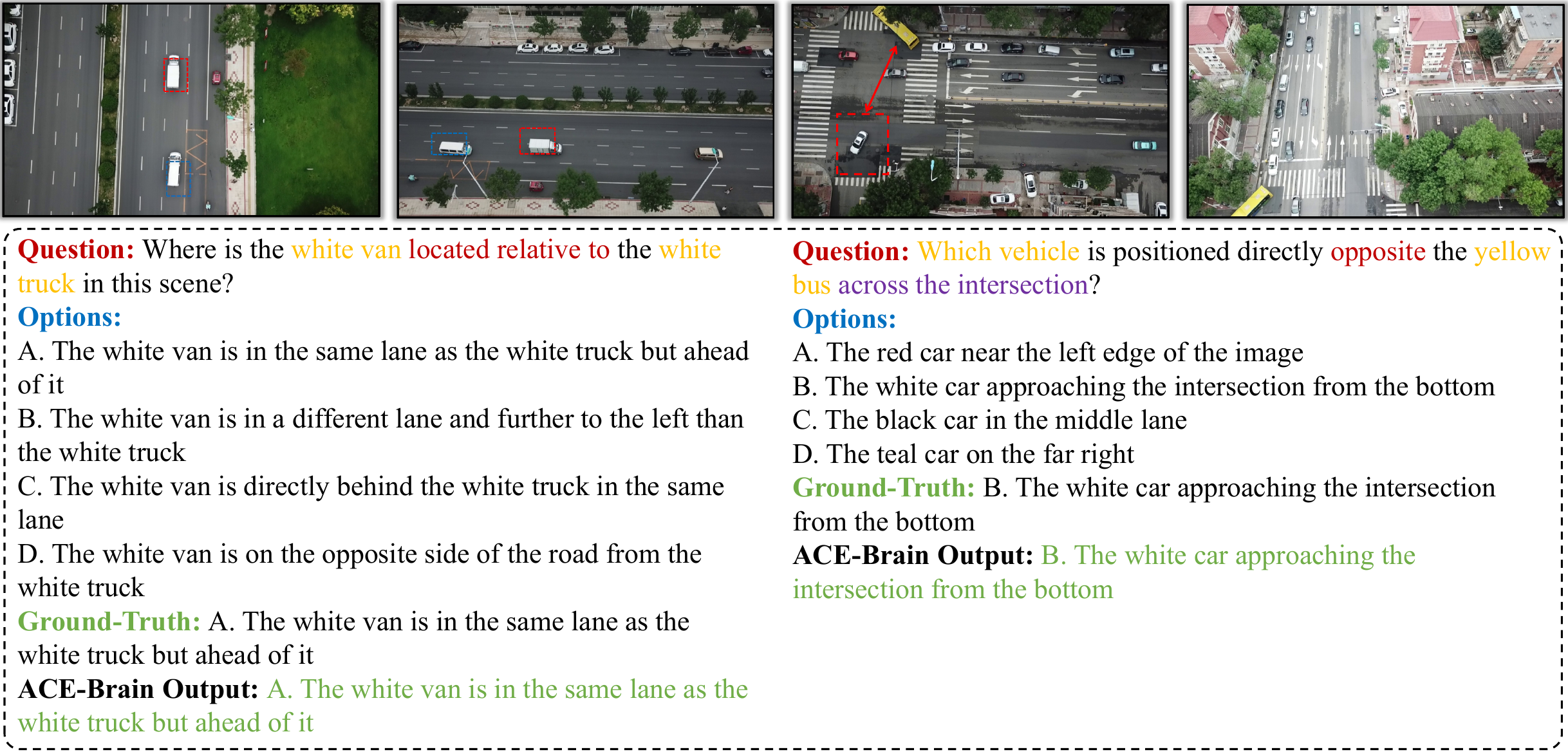}
    \caption{\textbf{Example 1 and 2 of AirCopBench.}}
    \label{aircop12}
\end{figure*}

\begin{figure*}
     \centering
    \includegraphics[width=1\textwidth]{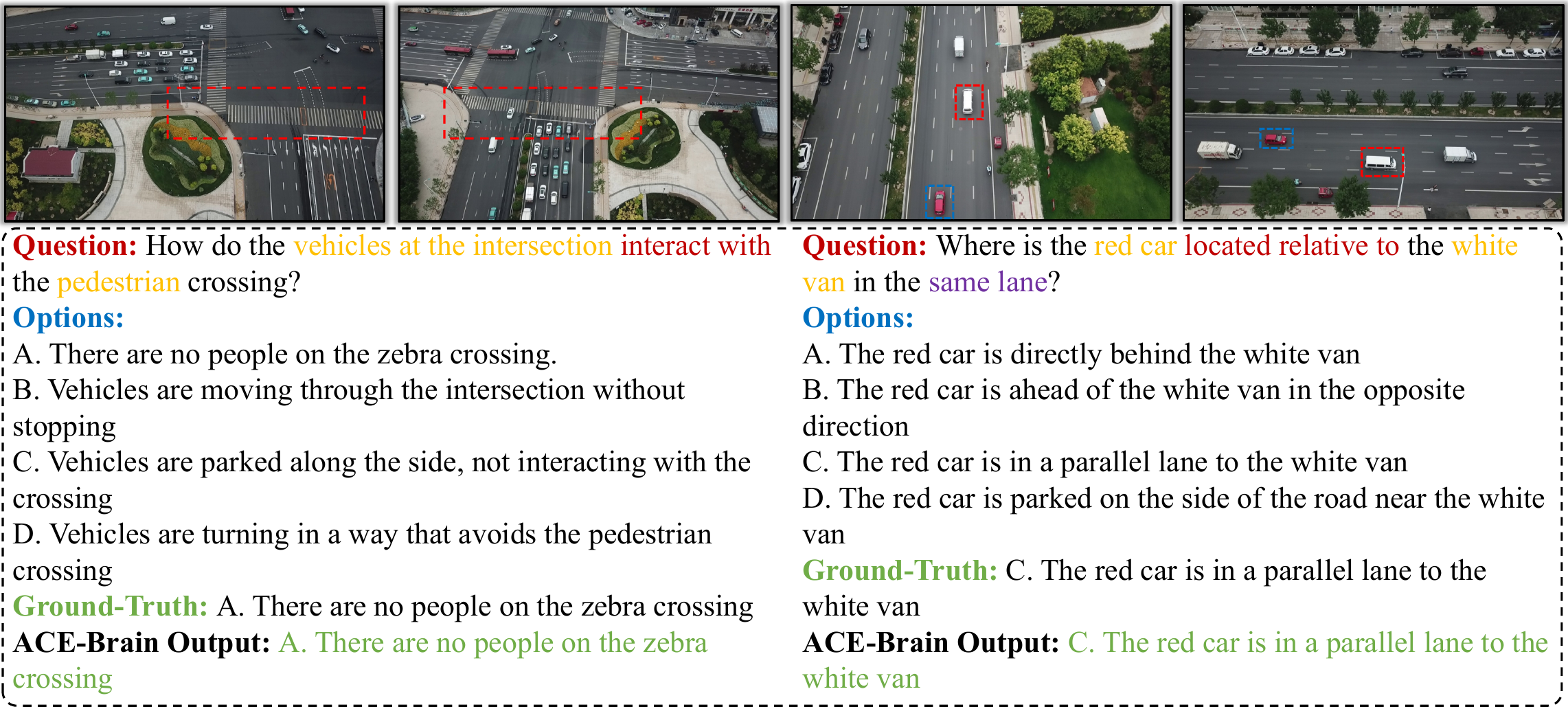}
    \caption{\textbf{Example 3 and 4 of AirCopBench.}}
    \label{aircop34}
\end{figure*}

\begin{figure*}
     \centering
    \includegraphics[width=1\textwidth]{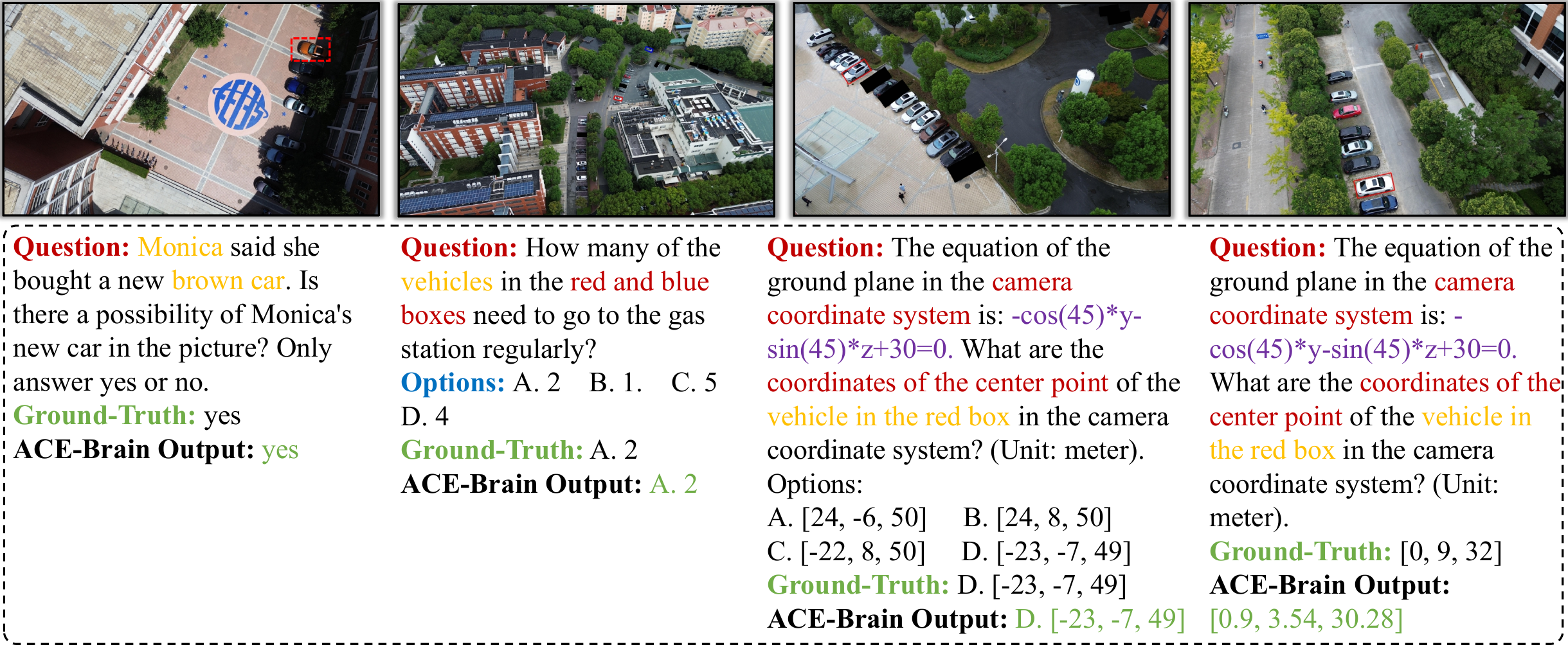}
    \caption{\textbf{Example 1-4 of AVI-Math.}}
    \label{avimath14}
\end{figure*}

\begin{figure*}
     \centering
    \includegraphics[width=1\textwidth]{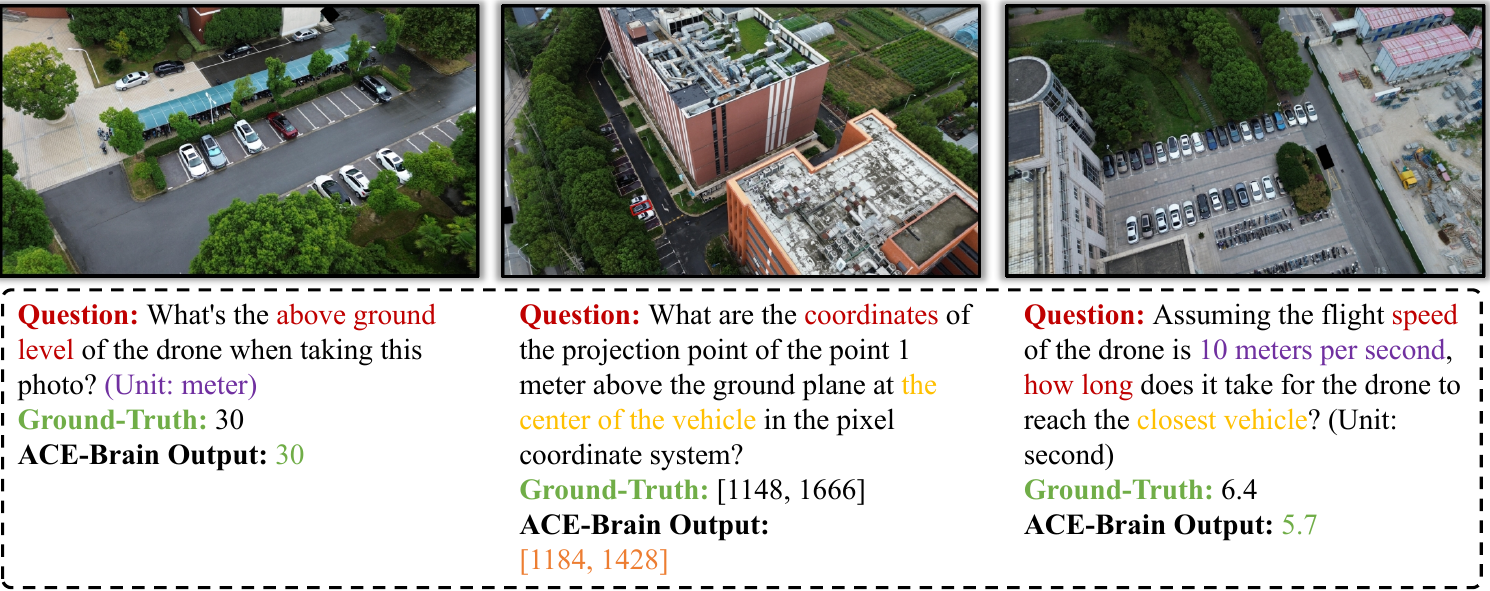}
    \caption{\textbf{Example 5-7 of AVI-Math.}}
    \label{avimath57}
\end{figure*}
\begin{figure*}
     \centering
    \includegraphics[width=1\textwidth]{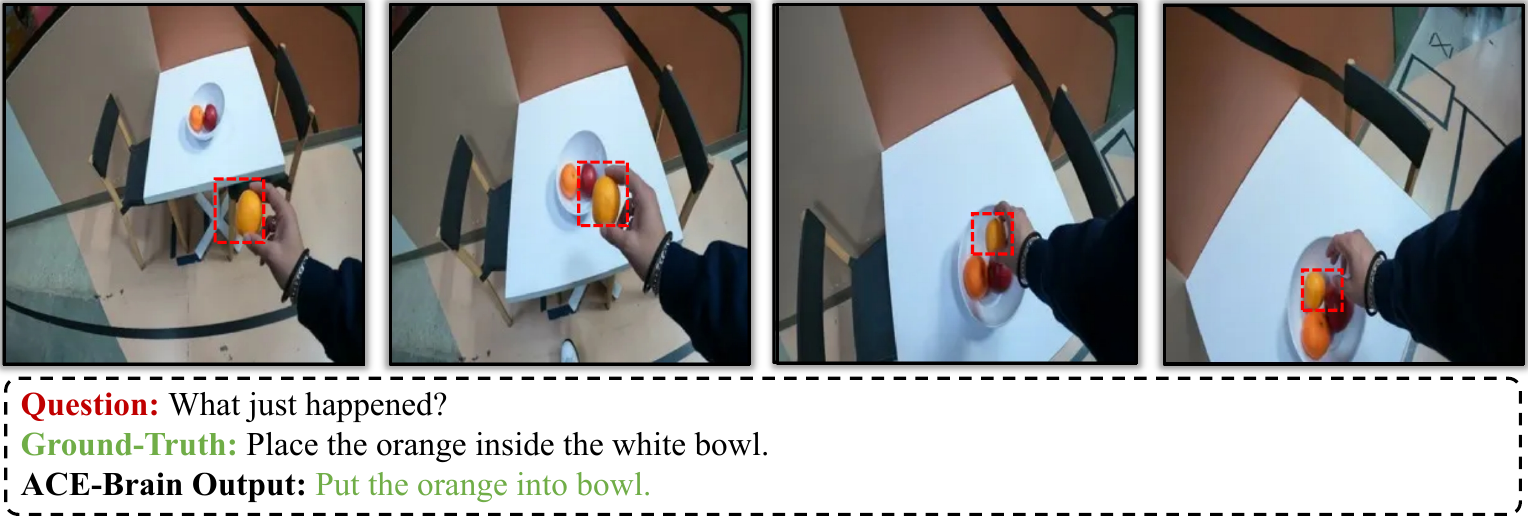}
    \caption{\textbf{Example 1 of RoboVQA Benchmark.}}
    \label{robovqa1}
\end{figure*}
\begin{figure*}
     \centering
    \includegraphics[width=1\textwidth]{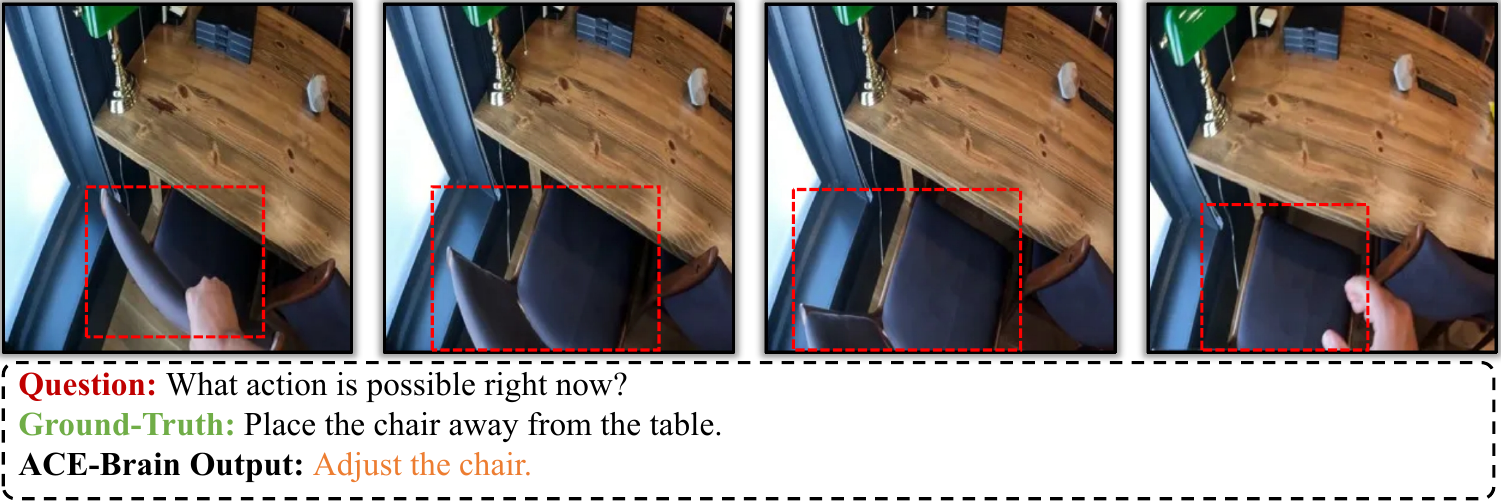}
    \caption{\textbf{Example 2 of RoboVQA Benchmark.}}
    \label{robovqa2}
\end{figure*}
\begin{figure*}
     \centering
    \includegraphics[width=1\textwidth]{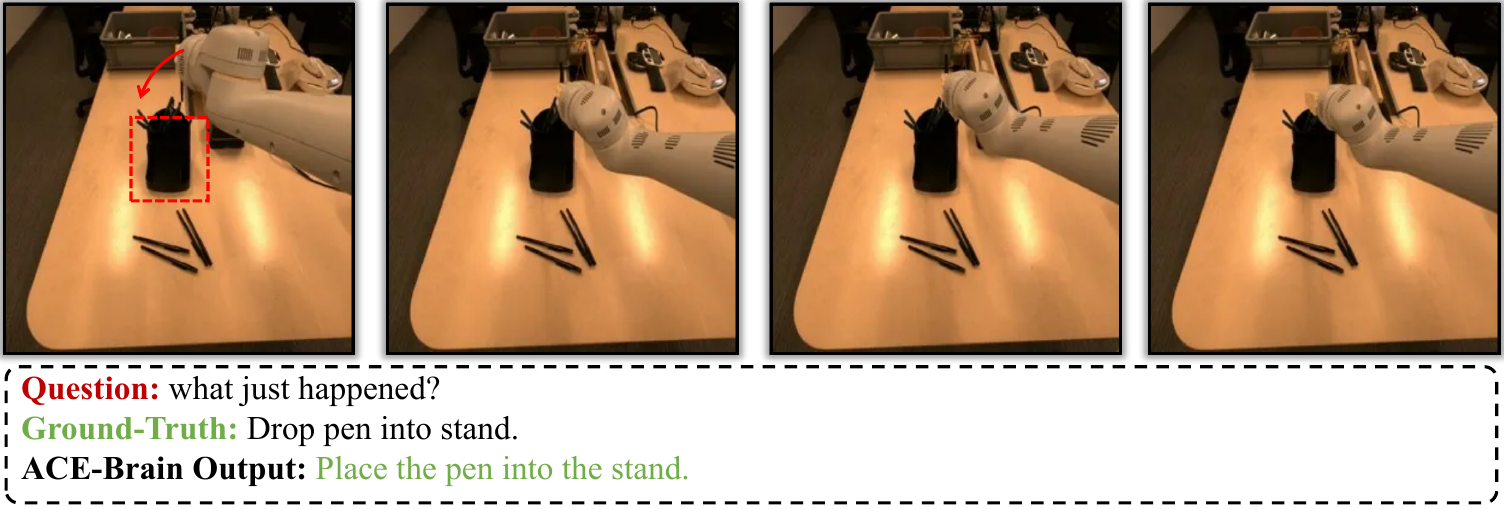}
    \caption{\textbf{Example 2 of RoboVQA Benchmark.}}
    \label{robovqa3}
\end{figure*}
\begin{figure*}
     \centering
    \includegraphics[width=1\textwidth]{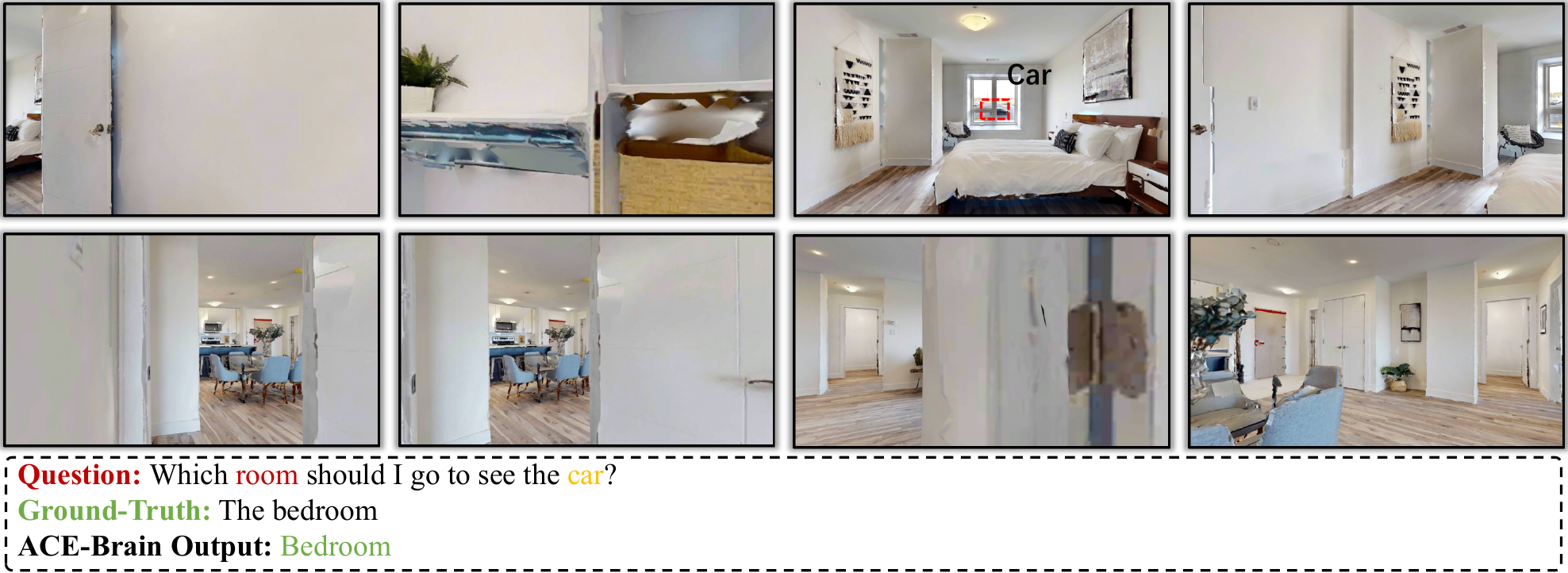}
    \caption{\textbf{Example 1 of OpenEQA Benchmark.}}
    \label{openeqa1}
\end{figure*}
\begin{figure*}
     \centering
    \includegraphics[width=1\textwidth]{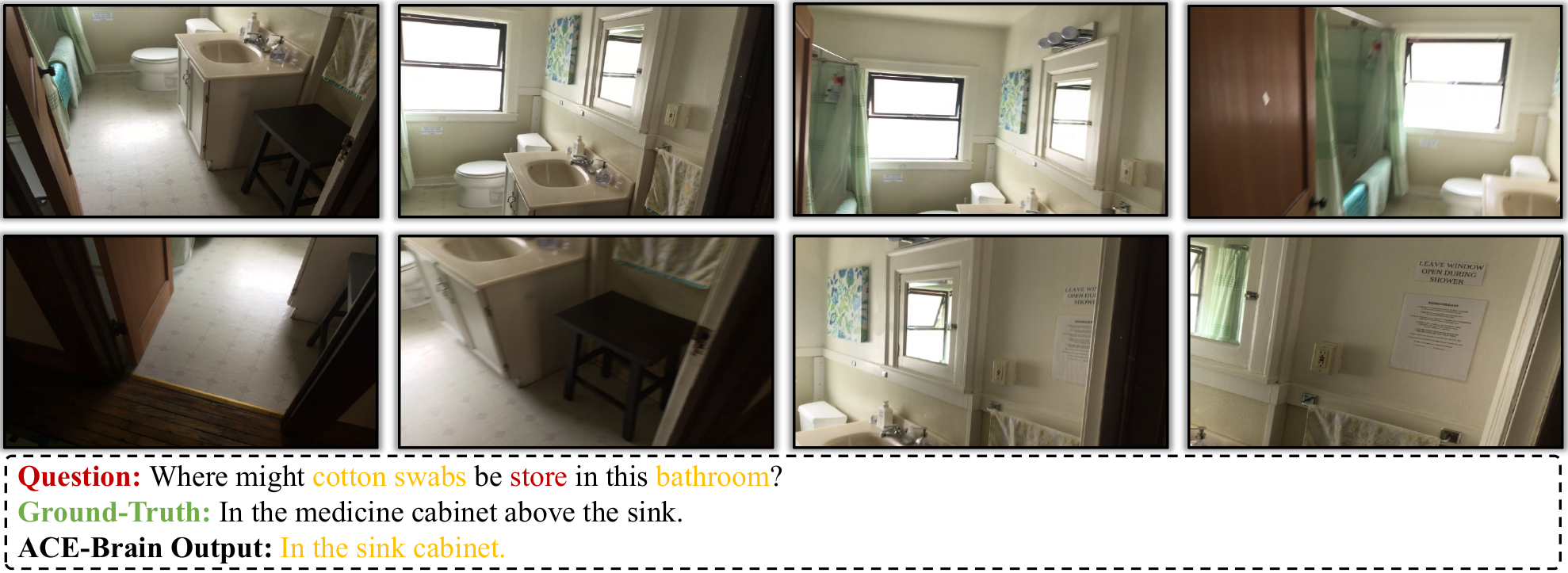}
    \caption{\textbf{Example 2 of OpenEQA Benchmark.}}
    \label{openeqa2}
\end{figure*}
\begin{figure*}
     \centering
    \includegraphics[width=1\textwidth]{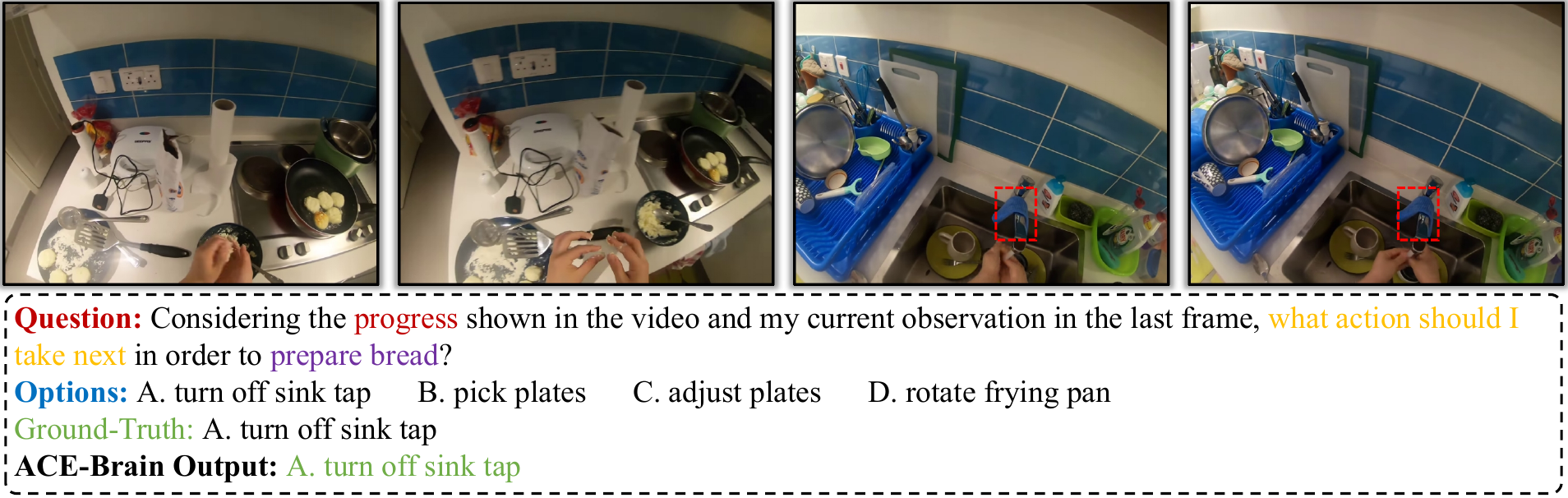}
    \caption{\textbf{Example 1 of EgoPlan-Bench2 Benchmark.}}
    \label{egoplan1}
\end{figure*}
\begin{figure*}
     \centering
    \includegraphics[width=1\textwidth]{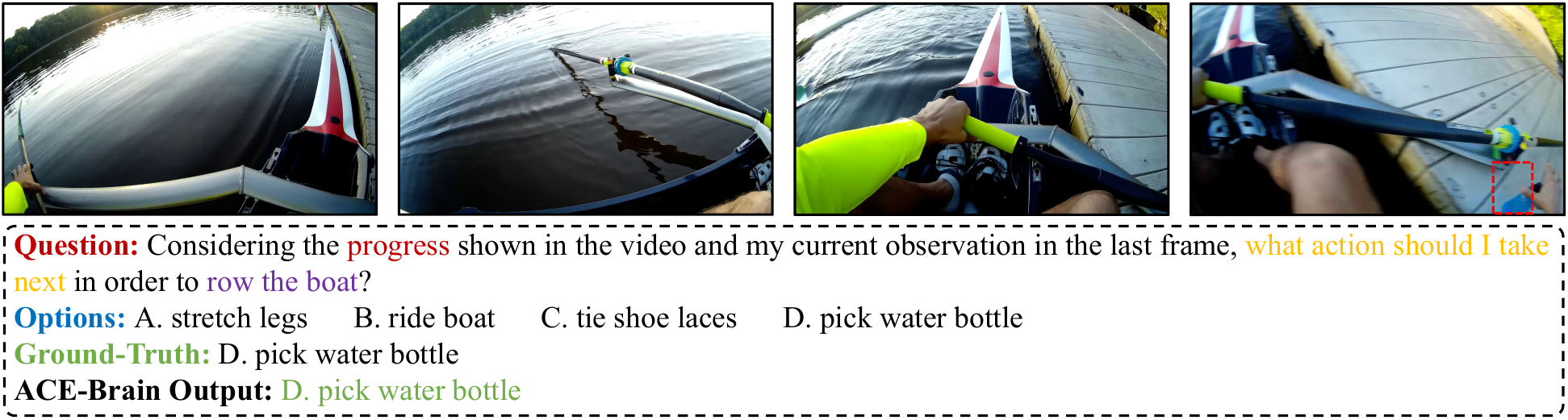}
    \caption{\textbf{Example 2 of EgoPlan-Bench2 Benchmark.}
}
    \label{egoplan2}
\end{figure*}

\end{document}